\documentclass{article}

\usepackage[preprint]{tmlr}

\usepackage{amsmath, amssymb}
\usepackage{amsthm}
\usepackage{graphicx}
\usepackage{booktabs}
\usepackage{float}
\usepackage{placeins}
\usepackage{xcolor}
\usepackage{tikz}
\usetikzlibrary{positioning,arrows.meta,calc}
\usepackage[colorlinks=true,linkcolor=blue!60!black,citecolor=blue!60!black,urlcolor=blue!60!black]{hyperref}
\usepackage{url}

\newtheorem{lemma}{Lemma}
\newtheorem{theorem}{Theorem}
\newtheorem{corollary}{Corollary}
\newtheorem{proposition}{Proposition}

\title{Measuring Learned Monotone Temporal Aggregation at Matched Admissibility}
\author{\name Yew Lee Tan\thanks{This work was conducted in a personal capacity; the views expressed are the author's own.} \email yewlee001@e.ntu.edu.sg}

\begin{document}
\maketitle

\begin{abstract}
Regulated risk scoring imposes directional constraints on how scores respond to risk drivers; this paper adopts their strict per-input form---the score monotone non-decreasing in every exposure input---as a normative commitment. The monotone core of the deployed pipeline already satisfies it: hand-crafted monotone temporal aggregates feeding sign-constrained gradient boosting are monotone end-to-end by composition. The constrained-versus-unconstrained comparisons standard in the monotone-network literature therefore price a guarantee this incumbent has for free. The right evaluation holds admissibility fixed on both sides of every comparison and measures what \emph{learning} the temporal aggregation is worth; we construct that evaluation and report its results. Its central finding is a subsumption boundary, measured by direct functional regression: a learned monotone recurrent channel reproduces the geometrically weighted separable family of hand-crafted statistics, one channel per member, to Spearman $\rho \ge 0.996$ ($1.000$ for every decay, on- and off-grid), approximates trailing-window and windowed-extremal statistics with measurable ceilings (the trailing-window tier priced at ${\sim}0.008$ AUC), and fails at consecutivity ($\rho = 0.924$) and time localization ($0.628$)---both structural---and at the exposure floor ($0.829$), a learnability boundary. Each failure is repaired by one explicit admissible basis with a one-line monotonicity proof, restoring rank correlation $1.000$ and returning the matching adversarial regime to its Bayes ceiling. Learned and engineered monotone aggregation are therefore substitutes on separable structure, and on mechanisms in or near that family the learned channel is never statistically behind at full sample size and specified capacity. Each of its advantages holds against a specific incumbent: a committed grid pays up to $0.019$ AUC in decay regions it leaves uncovered (no more than $0.004$ separates the learned channel from the strongest engineered consumer at any swept point); the highest-dimensional comparator degrades fastest as training data becomes scarce; and under stress beyond the training support, grid-fed tree-ensemble scores go flat while a strictly increasing head keeps ranking. No single incumbent is dominated on all three axes. The instrument that makes these measurements possible is a recurrent architecture whose state consists of classical risk statistics---an exponentially weighted moving average and a high-water mark with learned monotone transforms---monotone by construction in every input and per MC-dropout sample.
\end{abstract}

\section{Introduction}
\label{sec:intro}

Exposure time series arise in many economic and operational settings (e.g., utilization, capacity burden, or limit usage over time), where larger exposures are expected to correspond to greater risk. In regulated risk scoring settings, this relationship is enforced through directional constraints---in deployed practice, typically at the level of scorecard features rather than of every raw input. This paper studies the strict per-input form: increasing any exposure input $x_t$ should never decrease the resulting score. Supervisory guidance and validation practice require risk models to be conceptually sound and explainable \citep{fed2011_sr11_7,occ2011_12}, and a monotone exposure--score relationship is a standard way to meet that expectation; the per-input form is a normative commitment this paper adopts and defends (Section~\ref{sec:setup}), not a regulatory mandate. Once adopted, the constraint is normative rather than statistical---a model that violates it is inadmissible regardless of its predictive accuracy---so it must hold exactly, for every input, not merely on average over a validation set.

Deployed practice already meets this constraint on its monotone core. Credit pipelines collapse exposure histories into windowed aggregates---trailing means and maxima, breach counts, last values---of the kind packaged as bureau trended-data attributes \citep{experian_trended3d,vantagescore4_userguide}, and feed them to a scorecard model, increasingly a gradient-boosted ensemble with per-feature monotone constraints \citep{chen2016xgboost,ke2017lightgbm,lightgbm_monotone}. Each such aggregate is monotone in every exposure input, and the constrained ensemble is monotone in every feature, so the pipeline is monotone end-to-end by composition. The guarantee, by itself, is available to engineered features today.

This fact has a consequence for evaluation that the monotone-network literature has not absorbed. The standard experimental design compares a constrained model against an unconstrained one and reports the accuracy cost of the constraint. In a setting where the incumbent can be monotone end-to-end at no cost, that comparison measures the wrong quantity: it prices a guarantee the deployment gets for free either way. The comparison that matters holds admissibility fixed on both sides---every model in it provably monotone---and measures the one thing that actually differs between the learned and the engineered pipeline: whether the temporal aggregation is \emph{learned} or fixed ex ante by an analyst. Every window length, decay, and threshold in an engineered grid is hand-chosen; a learned monotone channel can adapt them to the data. We call this design \emph{matched admissibility}, and this paper builds and runs it.

Run that way, the evaluation returns a sharper answer than a scoreboard. Its central finding is a \emph{subsumption boundary}, established constructively by direct functional regression (experiment E11): a single learned monotone recurrent channel spans the geometrically weighted, time-separable family of hand-crafted statistics---EWMAs at every decay, on-grid and off-grid, cumulative sums, breach counts, cumulative excesses, last values---to rank equivalence. Learned monotone features are therefore substitutes for the individual separable statistics, not complements (one channel per member; E8 supplies the set-level counterpart on the core regimes). Trailing-window and windowed-extremal statistics outside the family are approximated with a measurable ceiling, and the spanning fails at three places: consecutivity (the longest run of threshold exceedances) and time localization (an early-window statistic), both structural, and the exposure floor (the running minimum), a learnability boundary. Each failure is repaired by one explicit admissible basis with a one-line monotonicity proof, demonstrated by rerunning the emulation and the adversarial-regime ladder with the matching channel added. The predictive experiments locate where the learned channel's adaptivity pays: decay regions the chosen grid leaves uncovered, scarce training data against the highest-dimensional comparator, and stress beyond the training support, where piecewise-constant tree ensembles over few engineered features go flat while a strictly increasing head keeps ranking.

The instrument that makes these measurements possible is an architecture, and we present it as an instrument rather than as the primary contribution: the contribution is the evaluation and the boundary; the architecture is claimed only in the role that makes them measurable. Guaranteeing monotonicity for a sequence-to-scalar mapping is nontrivial: recurrence introduces path dependence, local non-monotone effects can accumulate across time, and generic recurrent architectures may violate monotonicity even when trained on data generated from monotone mechanisms \citep{elman1990finding,hochreiter1997lstm}. Soft remedies---monotonicity penalties in the loss, or post-hoc adjustments---reduce violations empirically but provide no guarantee, and the residual violations are input-dependent and difficult to audit. Modern deployments additionally require uncertainty estimates, commonly via MC dropout or deep ensembles \citep{gal2016dropout,lakshminarayanan2017deep_ensembles}; if monotonicity is enforced only on average or only for a single fitted instance, sampled models may violate it, making uncertainty summaries semantically inconsistent with monotone scoring. A useful guarantee must survive stochastic inference: every sampled model must itself be monotone.

Following the theory-guided design principle of embedding known invariants in the hypothesis class \citep{daw2019pga_uq}, our instrument maintains an economically interpretable two-dimensional state: a recency-weighted exposure level, through an exponentially weighted moving average (EWMA), and a historical extreme, through a high-water mark (HWM). Both are classical risk statistics with long operational histories \citep{hyndman2002ets,goetzmann2003hwm}, and both are monotone in every input. The input transforms and the decay are learned within an admissible monotone class, and the final score is a monotone function of the state, so monotonicity holds by construction---for every input, under every admissible configuration, and per MC-dropout sample (Corollary~\ref{cor:mcdropout}). The state is not a pair of fixed, handcrafted features: no monotone model over static aggregates matches the recurrent state on recency-driven mechanisms (E2), and the learned transforms and decay recover the generating mechanism rather than merely consuming it (E6).

\paragraph{Claims.}
The empirical study is organized around eight explicit claims, each with one primary experiment (Section~\ref{sec:design}; the robustness claim's two halves, C4a and C4b, each have their own):
\begin{itemize}
  \item \textbf{C1 (by construction).} The proposed architecture is monotone by construction: zero violations under deterministic inference \emph{and} per stochastic forward pass under MC dropout, whereas unconstrained and penalty-trained baselines violate.
  \item \textbf{C2 (temporal state is necessary).} Monotonicity alone is not enough: a monotone head over static aggregates (order-invariant summaries plus the last value) cannot represent the recency-decay state and is predictively insufficient on recency-driven mechanisms.
  \item \textbf{C3 (little or no predictive cost).} Across the core label regimes, the constrained model matches unconstrained sequence baselines, with any regime-specific deficits quantified rather than averaged away.
  \item \textbf{C4a (robust to distribution shift)} and \textbf{C4b (robust to missingness).} The guarantee persists exactly under ordering-preserving covariate shift and under missing data with standard imputations, with predictive degradation quantified per model class rather than averaged away.
  \item \textbf{C5 (interpretable).} With learned transforms, the model recovers the generating decay $\alpha$ and the peak transform $g_2$; we characterize the extent to which the recency transform $g_1$ is identifiable.
  \item \textbf{C6 (design-insensitive guarantee).} Zero violations hold across every admissible configuration of $g_1$, $g_2$, and $\alpha$; accuracy varies by at most $0.0351$, concentrated in the recency channel's known optima, with no catastrophic configuration.
  \item \textbf{C7 (basis sufficiency).} The EWMA+HWM pair is near-sufficient on the core regimes; adversarial regimes built around each remaining admissible basis map the sufficiency boundary, which lies at non-separable and time-localized statistics rather than at any separable accumulation the learned recurrence can emulate.
  \item \textbf{C8 (learned vs.\ engineered aggregation).} Against monotone consumers of hand-crafted temporal feature grids---including monotone-constrained gradient boosting over practice-style and densified grids and over raw lags---the learned channel matches in-distribution at full sample size on the core regimes, and earns its advantage where a committed grid cannot adapt: under decay misspecification of that grid, under scarce training data against the highest-dimensional comparator, and under ordering-preserving stress beyond the training support, where grid-fed tree-ensemble scores saturate, while replacing a multi-column feature library with one auditable channel.
\end{itemize}
The claims are tested on a controlled synthetic suite with known monotone mechanisms, because misspecification distance, Bayes ceilings, and per-basis adversaries are observable only in that setting.

\paragraph{Contributions.}
\begin{itemize}
  \item \textbf{The matched-admissibility evaluation.} A claims-driven experimental design (C1--C8, experiments E1--E10 and the emulation study E11) in which every comparison against the engineered incumbent runs with both sides provably monotone, so that any measured difference is the value of learning the aggregation rather than the value of the constraint. The design includes known monotone mechanisms, adversarial per-basis regimes, severity-swept shifts, missingness grids, and misspecification sweeps whose distance from the engineered grid is exactly known.
  \item \textbf{The subsumption boundary} (E8, E11): a functional-emulation methodology that shows by direct regression which hand-crafted statistics a learned monotone channel spans (the separable family, member by member, to rank equivalence) and where it fails, converting an expressiveness worry into a concrete audit question with admissible repairs.
  \item \textbf{The instrument}: a theory-guided recurrent architecture with an interpretable classical-statistics state (EWMA and HWM with learned monotone transforms and decay), monotone by construction and per MC-dropout sample, together with a library of ten admissible bases with one-line proofs.
\end{itemize}

\paragraph{Paper organization.}
Section~\ref{sec:related-work} reviews related work. Section~\ref{sec:method} formalizes the problem and presents the architecture and its monotonicity guarantees. Section~\ref{sec:data} describes the synthetic data-generating processes: label regimes and data protocols. Section~\ref{sec:design} specifies the experimental design---models, test operators, and the mapping of each claim to one experiment. Section~\ref{sec:results} presents the results, and Section~\ref{sec:discussion} discusses limitations and future work. The paper's length is spent on auditability rather than breadth: each claim is small and carries its full evidence, and a reading guide opens Section~\ref{sec:results}.

\section{Related Work}
\label{sec:related-work}

\subsection{Monotone neural networks}
\label{sec:rw-monotone-nn}

Enforcement of monotonicity in neural networks has a long history, on two tracks: soft penalties---\emph{monotonicity hints}---that discourage violations during training \citep{sill1997hints}, and architectural constraints that exclude them. Early architectural work constrained feed-forward networks to non-negative weights with monotone activations, establishing that the constrained class remains a usable function approximator for classification problems with known input--output orderings \citep{archer1993monotonic}. Min--max networks compute a maximum over minima of linear functions with sign-constrained weights and are universal approximators of multivariate monotone functions \citep{sill1997monotonic}; later work extended this line to partially monotone settings, where only a subset of inputs carries the constraint \citep{daniels2010monotone}, and a smooth reformulation recently repaired the original construction's trainability \citep{igel2024smooth}. Lattice-based methods take a different route: monotonic calibrated interpolated look-up tables \citep{gupta2016lattice} and their deep extension, deep lattice networks \citep{you2017deep}, impose monotonicity through linear inequality constraints on interpolation parameters and have been deployed in production ranking systems. More recent approaches expand expressiveness while retaining the guarantee. Unconstrained monotonic neural networks parameterize a function through a strictly positive integrand and obtain monotonicity by integration \citep{wehenkel2019umnn}; certified monotonic networks train unconstrained networks and verify monotonicity post hoc via mixed-integer programming \citep{liu2020certified_monotonic}; counterexample-guided learning alternates training with SMT-based verification and repair \citep{sivaraman2020comet}. Combining saturated and non-saturated activations has been shown to remove the convexity bias that plagues naive non-negative-weight networks \citep{runje2023cmnn}, and the same line has since reached universal approximation with unbounded convex activations through alternating-saturation constructions \citep{sartor2025universal}; Lipschitz-constrained residual architectures obtain expressive monotone models through a different weight constraint \citep{kitouni2023expressive}, and scalable partially connected architectures extend constrained monotone learning to higher-dimensional inputs \citep{kim2024smnn}. On the theoretical side, the size and depth required for monotone networks with non-negative parameters to interpolate and approximate monotone data have been characterized, showing that depth restrictions carry real expressivity costs \citep{mikulincer2022size}.

All of these academic methods address \emph{static} inputs: a fixed-dimensional feature vector mapped through a constrained network. Industrial practice has gone one step further: weight-constrained monotonic recurrent networks for credit-risk indicators are patented by a credit-bureau assignee \citep{boardman2024mrnn}---the construction our baseline set evaluates as the Elman form of its generic monotone RNN (Section~\ref{sec:models}), and direct evidence that the governance setting demands the recurrent guarantee. What that line supplies is the constraint mechanism, not the rest: no interpretable classical-statistics state, no per-sample guarantee under stochastic inference, and no public evaluation against the engineered incumbent. Within the monotone-network literature, no architecture provides a recurrent state with interpretable temporal semantics for how the effect of an input at time $t$ persists, decays, or saturates over the remaining horizon. Our architecture fills this gap by restricting the recurrence itself to order-preserving updates, so that the guarantee composes across time as well as across layers. (A separate line of work uses ``monotone'' in the operator-theoretic sense: monotone operator equilibrium networks \citep{winston2020monotone} constrain fixed-point layers for well-posedness and stability, not for input--output order, and are orthogonal to the coordinatewise monotonicity required here.)

\subsection{Shape constraints and isotonic methods}
\label{sec:rw-shape}

A complementary statistical tradition estimates functions under order restrictions rather than constraining a network architecture: isotonic regression and its order-restricted inference framework \citep{barlow1972order}, its machine-learning use for mapping classifier scores to calibrated probabilities \citep{zadrozny2002isotonic}, and shape-constrained additive models with monotone or convex smooth components \citep{pya2015scam}. These methods differ from ours in where the constraint lives: shape-constrained \emph{estimation} restricts the fitted function after choosing a flexible family, whereas constrained \emph{architectures} restrict the hypothesis class so that every reachable parameter setting is admissible---what regulated deployment requires, since the guarantee must hold under retraining, fine-tuning, and stochastic inference, not only for one audited fit.

\subsection{Exponentially weighted and extreme statistics in sequential risk}
\label{sec:rw-ewma}

The two components of our recurrent state are classical sequential statistics with long operational histories. The EWMA originates in the exponential-smoothing forecasting literature \citep{brown1963smoothing,holt2004ewma,winters1960seasonal}, later unified in a state-space framework in which the smoothed level is an explicit latent state \citep{hyndman2002ets,hyndman2021fpp3}; the same recursion underlies the EWMA control chart \citep{roberts1959ewma} and cumulative-sum schemes \citep{page1954cusum} in process control, the RiskMetrics volatility methodology---which institutionalized a single decay parameter as a transparent, auditable summary of recency weighting \citep{riskmetrics1996}---and the exponentially discounted event history in the intensity of self-exciting point processes \citep{hawkes1971spectra}. Running extremes are equally established: the high-water mark is a contractual object in hedge fund management, with its economic consequences analyzed in closed form \citep{goetzmann2003hwm}. The closest learned-recurrence relative of our design is the hybrid exponential-smoothing recurrent model that won the M4 forecasting competition \citep{smyl2020hybrid}, which fits smoothing parameters jointly with a recurrent network; it pursues a forecasting objective and carries no monotonicity guarantee---precisely the two axes our setting fixes.

This lineage is load-bearing: the states $e_t$ and $m_t$ of Section~\ref{sec:architecture} \emph{are} these classical statistics, imported as the entire hidden state, so the architecture inherits their interpretability and monotonicity, and learning reduces to choosing $g_1$, $g_2$, $\alpha$, and a monotone combination $F$ rather than rediscovering recency weighting and peak memory from data.

\subsection{Theory- and physics-guided architectures}
\label{sec:rw-theory-guided}

Physics-informed and theory-guided machine learning embeds known scientific structure into learning systems, either as soft penalties on the loss or as hard constraints on the architecture. Physics-informed neural networks impose differential-equation residuals as training objectives \citep{raissi2019pinn}, and the broader programme of physics-informed machine learning surveys the spectrum from observational biases through loss penalties to architectural inductive biases \citep{karniadakis2021piml}. Most relevant to us is the architectural end of this spectrum: the physics-guided architecture of \citet{daw2019pga_uq} hard-codes a monotonic density--depth relationship into a lake temperature model and shows that, because the constraint is structural, it continues to hold under Monte Carlo dropout, whereas loss-based penalties do not survive parameter perturbation. Our work transfers this design principle from physical conservation-type laws to an \emph{order} constraint: monotonicity of an exposure score is enforced in the hypothesis class itself, and the same argument---structural constraints are invariant to dropout masks---yields our per-sample guarantee under stochastic inference (Corollary~\ref{cor:mcdropout}).

\subsection{Interpretable additive and tabular models}
\label{sec:rw-interpretable}

Generalized additive models decompose a prediction into per-feature shape functions \citep{hastie1986gam}, and their modern tree- and boosting-based instantiations achieve near state-of-the-art tabular accuracy while remaining inspectable \citep{lou2012intelligible,nori2019interpretml}. Neural additive models replace the spline or tree components with per-feature neural subnetworks \citep{agarwal2021nam}. In production tabular learning, gradient-boosted trees \citep{chen2016xgboost,ke2017lightgbm} support per-feature monotone constraints enforced during tree construction \citep{lightgbm_monotone}, which has made monotone boosting an increasingly standard component of credit modeling pipelines (alongside the binned logistic scorecards that still dominate much of regulated origination). Our learned transforms $g_1$ and $g_2$ play a role analogous to additive shape functions---each is a one-dimensional monotone function that can be plotted and audited, and experiment E6 tests whether they recover the generating shapes---but the analogy stops at the temporal dimension: additive and boosted models consume a static feature vector and are order-invariant in time, whereas our shape functions feed a recurrent state whose decay and peak semantics are part of the model's interpretation. One consequence of this pipeline structure shapes our evaluation: a constrained boosted ensemble over \emph{monotone temporal aggregates} is monotone end-to-end (Theorem~\ref{thm:score}), so the differentiator is the aggregation, which experiment E10 measures directly (Section~\ref{sec:models}).

\subsection{Missingness, distribution shift, and uncertainty in sequence models}
\label{sec:rw-robustness}

Deployed sequence models must remain trustworthy under incomplete data, shifted inputs, and stochastic inference, and each concern has its own literature. The formal treatment of missing data mechanisms (MCAR, MAR, MNAR) originates in the missing-data framework of \citet{rubin1976missing}; for sequences, GRU-D incorporates masking and time-decay terms directly into a recurrent cell to exploit informative missingness \citep{che2018grud}. Dataset shift has been taxonomized into covariate shift, label shift, and related forms \citep{quinonero2009shift}; our experiment E4 studies the specific subclass of ordering-preserving covariate shifts, under which a monotone architecture's guarantee is unaffected by construction. For uncertainty, MC dropout interprets dropout at test time as approximate Bayesian inference \citep{gal2016dropout}, and deep ensembles aggregate independently trained models \citep{lakshminarayanan2017deep_ensembles}; the quality of the resulting predictive probabilities is assessed with proper scores and calibration diagnostics \citep{brier1950verification,naeini2015bayesian_binning,guo2017calibration}. Interpretability in clinical sequence models has been pursued through attention, most prominently in RETAIN \citep{choi2016retain}, which explains predictions via learned per-visit and per-feature attention weights; such explanations are post hoc descriptions of an unconstrained model, whereas our interpretability is structural---the state variables themselves have fixed economic meanings that hold for every input. A common thread distinguishes our requirement from this literature: properties enforced in expectation or for a single fitted model do not automatically hold for each stochastic forward pass. Per-sample monotonicity under MC dropout, which our architecture provides, is strictly stronger than monotonicity of the averaged predictor.

\paragraph{Governance context.}
The evaluation design is motivated by credit-scoring governance, where monotone relationships between risk drivers and scores are an entrenched expectation \citep{thomas2017credit_scoring,crook2007consumer_credit}: a scorecard in which higher utilization lowers estimated risk would be rejected by validation teams regardless of its fit statistics, and the model risk management frameworks of SR 11-7 and OCC 2011-12 \citep{fed2011_sr11_7,occ2011_12} elevate that expectation to a requirement that must be demonstrated, documented, and preserved across model updates. Temporal information enters these pipelines as hand-crafted trended attributes \citep{experian_trended3d,vantagescore4_userguide}; notably, many trended attributes (utilization slopes, months since last delinquency, balance changes) are \emph{not} monotone in exposures and are therefore inadmissible under the constraint of Section~\ref{sec:setup}. Recent modeling work targets the same requirement with monotonic neural additive models for credit scoring \citep{chen2022mnam} and frames monotonicity as a first-class model-risk-management property \citep{chen2023monotonicity_mrm}, but operates on static feature vectors rather than exposure sequences. Learned temporal aggregation for credit itself is not new---recurrent and convolutional models over transaction and payment sequences outperform feature-engineered baselines in application scoring \citep{babaev2019etrnn} and mortgage-default prediction \citep{kvamme2018mortgage}---but that line is unconstrained: none of its models carries the monotonicity guarantee the governance setting requires, which is exactly the admissibility mismatch our evaluation removes. This is the setting in which admissibility must be matched before any comparison is informative.

\paragraph{Positioning.}
The literature provides monotone architectures for static inputs, patented industrial monotone recurrences without a public evaluation, shape-constrained estimators fitted post hoc, classical sequential statistics with monotone semantics but no learned scoring, and theory-guided architectures for physical invariants. What it does not provide is the evaluation this setting needs (Section~\ref{sec:intro}): to our knowledge, no prior work measures a monotone sequence model against the pipeline it would actually replace, at matched admissibility, with the misspecification distance between the learned and the engineered aggregation under experimental control. That evaluation---together with the recurrent instrument of Section~\ref{sec:architecture} that makes it runnable---is the contribution of this paper.

\section{Monotone Exposure Scoring: Problem and Architecture}
\label{sec:method}

\subsection{Problem setup and monotonicity constraint}
\label{sec:setup}

We observe a non-negative exposure time series $x_{1:T} = (x_1, \dots, x_T)$ with $x_t \ge 0$ and compute a scalar exposure score $s=f(x_{1:T})$ intended for use as an upstream feature in a downstream risk model.

Monotonicity is enforced as a structural constraint (motivated by regulated risk scoring): for the coordinatewise partial order
\begin{equation}
  x_{1:T} \preceq x'_{1:T} \iff x_t \le x'_t \; \forall t,
\end{equation}
we require
\begin{equation}
  x_{1:T} \preceq x'_{1:T} \Rightarrow f(x_{1:T}) \le f(x'_{1:T}).
\end{equation}

\paragraph{Governance admissibility.}
We treat the monotonicity requirement as a normative commitment rather than an empirical hypothesis: a component may enter the score only if it is provably monotone non-decreasing in every exposure input. We call such components \emph{admissible}. The per-input commitment is stricter than deployed practice, which commonly ships trended attributes that are monotone only at the scorecard-feature level, not in the underlying exposures (Section~\ref{sec:related-work}); the engineered comparators of this study are accordingly the incumbent library's admissible core under that commitment, not a census of what production systems ship: the excluded trended attributes carry predictive value this comparison forgoes, so matched admissibility measures learning against a deliberately restricted incumbent. We adopt the per-input form on recourse-coherence grounds: under an exact per-input guarantee, increasing any exposure can never improve the score, so adverse-action directions are unambiguous---a property feature-level monotonicity, holding only after aggregation, does not deliver. Statistics that are informative but not monotone (e.g., variance, trend, or time since last breach) may be used for diagnostics and oracle comparisons, but are excluded from the score itself. This separation lets the model remain auditable---each score component carries a one-line monotonicity proof---without discarding non-monotone information from the analysis.

\subsection{Theory-guided monotone recurrent architecture}
\label{sec:architecture}

The model uses a low-dimensional, economically interpretable recurrent state $h_t = [e_t, m_t]^\top$, where $e_t$ is a recency-weighted exposure (EWMA-like) state and $m_t$ is a high-water mark (HWM) state:
\begin{align}
  e_t &= \alpha e_{t-1} + g_1(x_t), \qquad 0<\alpha<1, \\
  m_t &= \max\big(m_{t-1},\; g_2(x_t)\big).
\end{align}
The final score is computed as $s = F(e_T,m_T)$. The two states have direct economic readings: $e_t$ captures the level, persistence, and recency of exposure through exponential decay, while $m_t$ records the worst exposure observed so far, so that severe events permanently affect the score. Figure~\ref{fig:architecture} summarizes the data flow; every arrow is a monotone non-decreasing map, so the score is monotone in each $x_t$ by composition (Theorem~\ref{thm:score}).

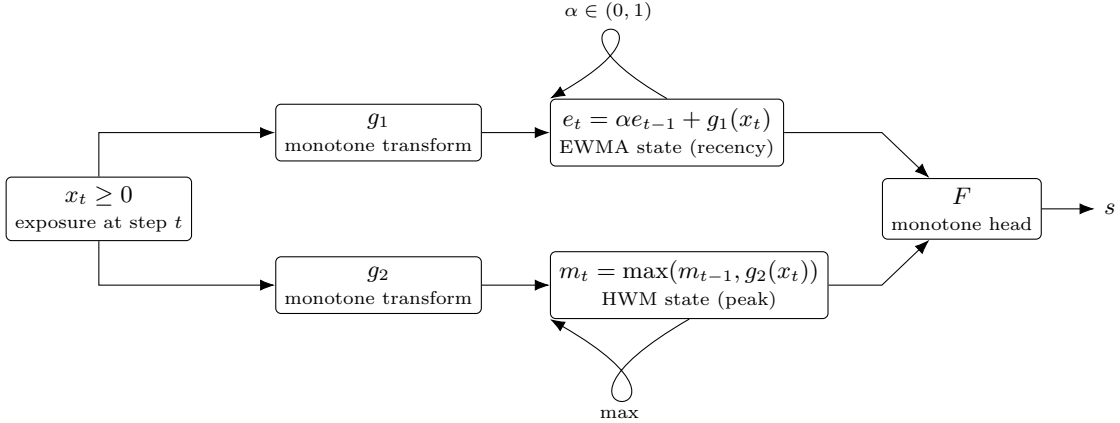
\begin{figure}[t]
\centering
\begin{tikzpicture}[
  font=\small,
  node distance=6mm and 9mm,
  box/.style={draw, rounded corners=2pt, minimum height=7.5mm, minimum width=17mm, align=center},
  mono/.style={-{Latex[length=2mm]}},
  lbl/.style={font=\scriptsize, midway}
]
\node[box] (x) {$x_t \ge 0$\\[-1pt]\scriptsize exposure at step $t$};
\node[box, above right=2mm and 11mm of x] (g1) {$g_1$\\[-1pt]\scriptsize monotone transform};
\node[box, below right=2mm and 11mm of x] (g2) {$g_2$\\[-1pt]\scriptsize monotone transform};
\node[box, right=of g1] (e) {$e_t = \alpha e_{t-1} + g_1(x_t)$\\[-1pt]\scriptsize EWMA state (recency)};
\node[box, right=of g2] (m) {$m_t = \max(m_{t-1}, g_2(x_t))$\\[-1pt]\scriptsize HWM state (peak)};
\node[box, right=10mm of $(e.east)!0.5!(m.east)$] (F) {$F$\\[-1pt]\scriptsize monotone head};
\node[right=7mm of F] (s) {$s$};
\draw[mono] (x) |- (g1);
\draw[mono] (x) |- (g2);
\draw[mono] (g1) -- (e);
\draw[mono] (g2) -- (m);
\draw[mono] (e.east) -- (F.west |- e.east) -- (F);
\draw[mono] (m.east) -- (F.west |- m.east) -- (F);
\draw[mono] (F) -- (s);
\draw[mono] (e.north) to[out=150,in=30, looseness=4] node[lbl, above] {$\alpha \in (0,1)$} (e.north west);
\draw[mono] (m.south) to[out=-150,in=-30, looseness=4] node[lbl, below] {$\max$} (m.south west);
\end{tikzpicture}
\caption{The monotone exposure RNN. Each exposure $x_t$ enters through learned monotone transforms $g_1, g_2$ (non-negative-weight MLPs with monotone activations, or fixed admissible forms); the two-dimensional state carries a recency-weighted level $e_t$ and a high-water mark $m_t$; a monotone head $F$ maps the final state to the score. Every arrow is a monotone non-decreasing map, so $s$ is monotone in every $x_t$ by construction---for every parameter setting and every MC-dropout sample.}
\label{fig:architecture}
\end{figure}

\paragraph{Admissible transforms.}
We assume $g_1,g_2:\mathbb{R}_{\ge 0}\to\mathbb{R}$ are monotone non-decreasing (and typically non-negative for interpretability).
Admissible parametric forms include the identity $g(x)=x$, the logarithm $g(x)=\log(1+x)$ (diminishing sensitivity), the hinge $g(x)=(x-\tau)^+$ with $\tau\ge 0$ (stress above a threshold), and the power $g(x)=x^p$ with $p\ge 1$ (tail emphasis). A learned alternative is a \emph{monotone MLP}: a feed-forward network whose weights are constrained non-negative through a parameterization such as $W=\mathrm{softplus}(\Theta)$ and whose activations are monotone non-decreasing (tanh in our implementation); biases are unconstrained. Because non-negativity is enforced by the parameterization, no clipping or projection step is needed during training.

\paragraph{Activation curvature and expressivity.}
Monotonicity requires only that the activations be non-decreasing, but their curvature determines which shapes the constrained network can express. With non-negative weights, an affine layer is a non-negative combination of its inputs plus a constant; if the activation is convex and non-decreasing (e.g., ReLU or softplus), then by induction every layer output is a convex, non-decreasing function of the network input, since non-negative combinations of convex functions are convex and composing a convex non-decreasing function with a convex function preserves convexity. The function class is therefore \emph{convex-only}, and a concave target transform such as $\log(1+x)$ is unrepresentable at any width or depth---the closest achievable fit is a near-linear convex approximation. Non-convex monotone activations such as tanh remove this restriction while leaving the monotonicity guarantee untouched, consistent with the convexity bias of non-negative-weight networks observed in the constrained-monotone literature \citep{runje2023cmnn}. (The convex-only limit is a property of the plain non-negative-weight stack ablated here; universality can be recovered within the constrained class by alternating-saturation constructions \citep{sartor2025universal}, a different architecture from the one this ablation prices.) This motivates the tanh default, and experiment E6 ablates the activation to test exactly this restriction.

\paragraph{Score head.}
The head $F:\mathbb{R}^2\to\mathbb{R}$ is monotone non-decreasing in each argument: either a linear map with non-negative coefficients or a monotone MLP as above. No non-monotone interaction between $e_T$ and $m_T$ is permitted.

\paragraph{Decay parameterization.}
The decay $\alpha$ may be fixed in $(0,1)$ or learned via $\alpha=\sigma(\alpha_{\mathrm{raw}})$, which keeps $\alpha\in(0,1)$ for any unconstrained $\alpha_{\mathrm{raw}}$.

\paragraph{Uncertainty-quantification variant.}
For uncertainty estimation we use a variant of the model with dropout applied inside the monotone MLP components, evaluated with MC dropout at inference time. Corollary~\ref{cor:mcdropout} below shows that every dropout-instantiated model remains monotone, so the per-sample guarantee is preserved.

\subsection{Monotonicity guarantee}
\label{sec:proof}

Define $f(x_{1:T}) := F(e_T(x_{1:T}), m_T(x_{1:T}))$.
We show that $x_{1:T}\preceq x'_{1:T}$ implies $f(x_{1:T})\le f(x'_{1:T})$.

\begin{lemma}[EWMA state is monotone]
\label{lem:ewma}
Assume $g_1$ is monotone non-decreasing and $0<\alpha<1$. Then for all $t$,
\begin{equation}
  x_{1:t} \preceq x'_{1:t} \Longrightarrow e_t(x_{1:t}) \le e_t(x'_{1:t}).
\end{equation}
\end{lemma}
\begin{proof}
By induction on $t$. The base case holds since $e_0$ is fixed. For the inductive step, if $e_{t-1}(x)\le e_{t-1}(x')$ and $x_t\le x'_t$, then
$e_t(x)=\alpha e_{t-1}(x)+g_1(x_t)\le \alpha e_{t-1}(x')+g_1(x'_t)=e_t(x')$.
\end{proof}

\begin{lemma}[HWM state is monotone]
\label{lem:hwm}
Assume $g_2$ is monotone non-decreasing. Then for all $t$,
\begin{equation}
  x_{1:t} \preceq x'_{1:t} \Longrightarrow m_t(x_{1:t}) \le m_t(x'_{1:t}).
\end{equation}
\end{lemma}
\begin{proof}
By induction on $t$. The base case holds since $m_0$ is fixed. For the inductive step, if $m_{t-1}(x)\le m_{t-1}(x')$ and $x_t\le x'_t$, then using that $\max(\cdot,\cdot)$ is non-decreasing in each argument,
$m_t(x)=\max(m_{t-1}(x),g_2(x_t))\le \max(m_{t-1}(x'),g_2(x'_t))=m_t(x')$.
\end{proof}

\begin{theorem}[Full score is monotone]
\label{thm:score}
Assume $g_1$ and $g_2$ are monotone non-decreasing and $F$ is monotone non-decreasing in each argument. Then
\begin{equation}
  x_{1:T} \preceq x'_{1:T} \Longrightarrow f(x_{1:T}) \le f(x'_{1:T}).
\end{equation}
\end{theorem}
\begin{proof}
By Lemmas \ref{lem:ewma}--\ref{lem:hwm}, $(e_T(x),m_T(x))\preceq(e_T(x'),m_T(x'))$. Since $F$ is monotone in both inputs, applying $F$ preserves order.
\end{proof}

When $g_1$, $g_2$, or $F$ is a monotone MLP, the assumptions of Theorem~\ref{thm:score} are met because an affine map with elementwise non-negative weights preserves the coordinatewise order, elementwise monotone activations preserve it as well, and compositions of order-preserving maps are order-preserving.

\paragraph{Composition with engineered features.}
The same composition argument extends beyond this architecture: any monotone non-decreasing consumer of any fixed monotone featurization---the bases of Proposition~\ref{prop:bases}, windowed means and maxima, or a gradient-boosted tree ensemble with per-feature non-decreasing constraints \citep{lightgbm_monotone}---yields an end-to-end monotone score. The by-construction guarantee is therefore \emph{not} unique to the learned channel; it is equally available to the deployed pipeline of hand-crafted aggregates feeding a constrained scorecard model. What the architecture adds over that pipeline is (i) \emph{learned} aggregation: the decay $\alpha$, the transforms $g_1, g_2$, and their composition with $F$ are fitted jointly rather than fixed ex ante, and (ii) a closed-form readable score: a trained instance collapses to a monotone function of one EWMA and one HWM with plottable one-dimensional transforms, rather than an ensemble over a multi-column feature library. Experiment E10 quantifies what (i) is worth and where.

\begin{corollary}[Monotonicity per MC-dropout sample]
\label{cor:mcdropout}
Let $f_\omega$ denote the model instantiated with a dropout mask $\omega$ that multiplies a subset of weights by zero or by a positive rescaling constant. If the undropped model satisfies the assumptions of Theorem~\ref{thm:score}, then $f_\omega$ is monotone for every $\omega$. In particular, every MC-dropout sample is monotone, and any per-sample or aggregated uncertainty summary is computed over monotone scoring functions.
\end{corollary}
\begin{proof}
Dropout multiplies non-negative weights by non-negative scalars, so all weights of $f_\omega$ remain non-negative; the recurrence structure, activations, and the admissible transforms are unchanged. Hence $f_\omega$ satisfies the assumptions of Theorem~\ref{thm:score}.
\end{proof}

\paragraph{Stability by construction.}
The EWMA component is bounded for $\alpha\in(0,1)$ (for non-negative $g_1$ and $e_0=0$, the geometric sum gives $e_T \le \max_t g_1(x_t)/(1-\alpha)$), while the HWM component is bounded by the maximum transformed exposure. This mitigates numerical instability that can occur in naive monotone RNNs without explicit decay.

\subsection{Additional admissible bases}
\label{sec:bases}

Beyond the EWMA and HWM states, several other sequence statistics satisfy the admissibility requirement of Section~\ref{sec:setup} and serve as building blocks for the basis-sufficiency analysis of experiment E8.

\begin{proposition}[Admissible basis library]
\label{prop:bases}
Each of the following statistics is monotone non-decreasing in every exposure input $x_t$: the last exposure $L_T(x)=x_T$; the cumulative sum $C_T(x)=\sum_{t=1}^T x_t$; the breach count $B_T(x)=\sum_{t=1}^T \mathbf{1}[x_t>\tau]$ for fixed $\tau\ge 0$; the cumulative excess $E_T(x)=\sum_{t=1}^T (x_t-\tau)^+$; the longest run above a threshold $R_T(x)$; saturating sums $S_T(x)=\sum_{t=1}^T \phi(x_t)$ for any non-decreasing $\phi$; the low-water mark $\ell_T(x)=\min_t g(x_t)$ for non-decreasing $g$---the monotone extremal mirror of the HWM, recording the floor of the exposure path; and the early-window EWMA $e_{t_0}(x)=\sum_{t \le t_0} \alpha^{t_0-t} g(x_t)$ for fixed $t_0 < T$ and $\alpha \in (0,1)$, non-decreasing in each input it observes and constant in the rest.
\end{proposition}
\begin{proof}
See Appendix~\ref{app:bases}.
\end{proof}

Together with EWMA and HWM, this yields a library of \emph{ten} monotone-admissible bases. Any score formed by a monotone head over any subset of these bases inherits the guarantee of Theorem~\ref{thm:score} by the same composition argument. The sequence mean is deliberately absent from the library: at fixed length $T$ it equals $C_T/T$, an affine rescaling of the cumulative sum, so it is collinear with an existing basis and rank-equivalent under any monotone head; it participates in the evaluation only as a feature of the static aggregate baselines and the hand-crafted grids.

\section{Synthetic Data: Label Regimes and Protocols}
\label{sec:data}

All experiments use synthetic data to enable controlled, ground-truth monotone mechanisms and precise monotonicity diagnostics. The data design has three orthogonal axes: a \emph{label regime} (the ground-truth monotone risk mechanism), a \emph{data protocol} (a transformation of a base regime's dataset, e.g., missingness or distribution shift), and a \emph{test operator} (an evaluation-time perturbation probing a trained model; operators are specified with the experimental design in Section~\ref{sec:operators}).

\paragraph{Exposure process.}
We generate non-negative exposure sequences $x_{1:T}$ by modeling log-exposure $y_t:=\log(1+x_t)$ with (i) a latent two-state stress process and (ii) state-dependent AR(1) dynamics with rare spikes.
A latent state $z_t\in\{0,1\}$ (normal vs stress) evolves as a Markov chain with per-sequence transition probabilities $p_{01}$ and $p_{10}$, producing occasional but persistent stress periods.
Conditional on $z_t$, log-exposure follows
\begin{equation}
  y_t = \mu_{z_t} + \rho_{\mathrm{AR}}\, (y_{t-1}-\mu_{z_{t-1}}) + \varepsilon_t,
  \qquad \varepsilon_t\sim\mathcal{N}(0,\sigma_{z_t}^2),
\end{equation}
initialized at $y_1 \sim \mathcal{N}(\mu_{z_1}, \sigma_{z_1}^2)$, with $\mu_1>\mu_0$ and $\sigma_1>\sigma_0$ so that the stress state is both higher-level and more volatile (the persistence is subscripted because unadorned $\rho$ denotes Spearman correlation throughout).
We recover exposures via $x_t=\max(\exp(y_t)-1, 0)$ and add a spike process: at each $t$, an indicator $b_t \sim \mathrm{Bernoulli}(p_{\text{spike}})$ selects rare steps, and a log-normal shock $u_t$ is added on those steps, $x_t\leftarrow x_t + b_t u_t$.
This construction yields non-negativity, heavy tails, temporal dependence, regime switching, and rare extremes. Every numeric constant of the process and of each label regime is listed in Appendix~\ref{app:protocols} (Tables~\ref{tab:dgp-exposure} and~\ref{tab:dgp-regimes}), so the suite is reproducible from the paper alone.

\paragraph{Ground-truth monotone risk functionals.}
For each dataset we fix a true decay $\alpha^*$ ($0.9$ in every default regime; E2 and E10 sweep it over a ten-point grid, and E6's recovery sweep retrains at off-reference values) and monotone transforms $g_1^*,g_2^*$ (e.g., identity, $\log(1+x)$, or hinge $\max(0,x-\tau^*)$ with $\tau^*$ a fixed constant).
We define true monotone summaries
\begin{align}
  e_t^* &= \alpha^* e_{t-1}^* + g_1^*(x_t), \quad e_0^*=0, \\
  m_t^* &= \max(m_{t-1}^*,\, g_2^*(x_t)), \quad m_0^*=0.
\end{align}
Labels are generated as $Y\sim\mathrm{Bernoulli}(\sigma(\eta))$, with the logit $\eta$ specified per regime below and intercepts calibrated so class prevalence is approximately $25\%$ per regime (Appendix~\ref{app:protocols} tabulates the exact constants).

\subsection{Label regimes}
\label{sec:regimes}

The fifteen label regimes fall into two groups. The five \emph{core} regimes (A--E) are EWMA/HWM-centric: their logits are functions of the true summaries $e_T^*$ and $m_T^*$, so they test whether the inductive bias of Section~\ref{sec:architecture} recovers the correct mechanism when the mechanism lies in, or near, the model family. The ten \emph{basis-adversarial} regimes (F--O) are each engineered around one admissible basis from the ten-basis library of Section~\ref{sec:bases}, placing the label mechanism beyond the EWMA+HWM pair (F--M), entirely on the pair's HWM member (regime N), or entirely on the pair's extremal mirror, the low-water mark (regime O); they are used by experiment E8 to map the sufficiency boundary of the two-state summary. Table~\ref{tab:regimes} lists the mechanisms; all slope coefficients are fixed positive constants, and intercepts are calibrated to the target prevalence. (A further diagnostic-only regime P---a pure trailing-window label---is introduced by experiment E11 to price the window tier and is not a member of the fifteen-regime suite; Appendix~\ref{app:protocols} defines it.)

\begin{table}[t]
\centering
\caption{Label regimes. Top: core regimes (EWMA/HWM-centric). Bottom: basis-adversarial regimes, each engineered around one admissible basis from the library of Section~\ref{sec:bases} (used by E8). Here $E_T^{\mathrm{ex}} = \sum_t (x_t - \tau^*)^+$, $B_T = \sum_t \mathbf{1}[x_t > \tau^*]$ (the breach count of Proposition~\ref{prop:bases}), $r_T^*$ is the longest run of consecutive indices with $x_t > \tau^*$; the shared exceedance threshold $\tau^*$ of the basis regimes and regime D's hinge threshold are distinct fixed constants ($3.0$ and $10.17$; Table~\ref{tab:dgp-regimes}). The mixture indicator $z \sim \mathrm{Bernoulli}(0.5)$ (distinct from the latent stress state $z_t$ of Section~\ref{sec:data}) is drawn per sequence, $t_0 = \lfloor T/3 \rfloor$, $C$ is a fixed cap (a high population quantile of the EWMA distribution), and $k$ is a short recent-window length. Regime N is the steep variant of regime D: its threshold $\tau_N$ is a fixed constant placed high in the HWM distribution (the population $q_{80}$) with a steep slope $\lambda_N$, so the label depends almost entirely on the peak and makes the HWM maximally prominent relative to recency and accumulation summaries. Regime O is N's extremal mirror: its label depends only on the exposure floor $\min_t x_t$, the low-water mark of Proposition~\ref{prop:bases}.}
\label{tab:regimes}
\small
\setlength{\tabcolsep}{4pt}
\begin{tabular}{@{}lll@{}}
\toprule
Regime & Logit mechanism $\eta$ & Targets \\
\midrule
A (recency) & $a\,e_T^* + b$ & $g_1^* = \log(1+x)$ \\
B (peak) & $c\,m_T^* + d$ & $g_2^* = x$ \\
C (mixed) & $a\,e_T^* + c\,m_T^* + b$ & EWMA and HWM jointly \\
D (threshold stress) & $\lambda\,m_T^* + b$ & $g_2^* = (x-\tau^*)^+$ \\
E (diminishing returns) & $a\log(1+e_T^*) + c\log(1+m_T^*) + b$ & saturating EWMA/HWM \\
\midrule
F (cumulative burden) & $\gamma \sum_t g(x_t) + b$ & cumulative sum \\
G (persistence) & $\beta\,r_T^* + b$ & longest run \\
H (severity $\times$ duration) & $\lambda\,E_T^{\mathrm{ex}} \log(1+B_T) + b$ & excess $\times$ breach count \\
I (substitution) & $a\,e_T^* + b$ if $z{=}1$, else $c\,m_T^* + d$ & joint EWMA{+}HWM \\
J (breach count) & $\kappa \sum_t \mathbf{1}[x_t > \tau^*] + b$ & breach count \\
K (early warning) & $a\,e_{t_0}^* + b$ & early-window EWMA \\
L (saturation) & $a\,\min(e_T^*, C) + b$ & capped / saturating \\
M (multi-scale) & $a\,e_T^* + c\,m_T^* + d \sum_{t=T-k+1}^{T} x_t + b$ & combined bases \\
N (peak prominence) & $\lambda_N\,(\max_t x_t - \tau_N)^+ + b$ & HWM, standalone \\
O (exposure floor) & $a\,\min_t x_t + b$ & LWM, standalone \\
\bottomrule
\end{tabular}
\end{table}

Every regime is monotone non-decreasing in every exposure input $x_t$: each mechanism is a non-decreasing composition of statistics from the admissible library (Proposition~\ref{prop:bases}), plus two one-line extensions the library does not list---regime M's trailing-window sum $\mathrm{RS}_k$, monotone as a partial sum, and regime H's product, monotone because both factors are non-negative and non-decreasing (for $f, g \ge 0$ non-decreasing, $fg$ is non-decreasing). For the breach-count regime J this follows from Lemma~\ref{lem:breach}---increasing any $x_t$ can only switch an indicator from $0$ to $1$---so J is admissible and is the unique regime isolating the breach-count basis. The threshold $\tau^*$ is always a fixed constant rather than a data-dependent quantile, which is what keeps the threshold-based summaries monotone in $x$.

\subsection{Data protocols}
\label{sec:protocols}

A data protocol transforms a base regime's dataset---inputs, labels, or observation mask---while reusing that regime's label mechanism: it moves the inputs (and, in one case, the labels), never the underlying risk function. Each protocol therefore takes a base regime as a parameter.

\paragraph{Clean.}
The identity transformation; the default protocol for most experiments.

\paragraph{Missingness.}
Applies masks after generation: MCAR at rates $\{5, 10, 15, 20, 30\}\%$ and block missingness with contiguous block lengths $\{2, 4, 6, 8\}$. Labels are computed from the unmasked sequence, and the masked inputs are completed with zero, forward-fill, or mean imputation.

\paragraph{Distribution shift.}
Applies test-time covariate shift to the inputs: scale inflation, heavier tails, increased spike frequency, and mixture contamination. Labels are frozen at their pre-shift values, so only ordering-preserving shifts ($x' \succeq x$) are admissible---a shift that reordered exposures would invalidate the frozen labels rather than test robustness. Severity grids are given in Appendix~\ref{app:protocols}.

\paragraph{Mechanism-parameter sweeps.}
Two sweeps recalibrate a regime's label coefficients as a generating parameter varies, holding class prevalence and the slope on the standardized mechanism feature fixed so that only the swept property changes: the E2/E10 decay sweep over regime A's $\alpha^*$, and the E10 threshold sweep over regime D's $\tau^*$. Both recalibrations reproduce the regime's published coefficients at the reference parameter (Appendix~\ref{app:protocols}).

\paragraph{Training-size subsampling.}
For the E10 sample-size sweep, the training split is truncated to its first $n$ sequences after generation while the validation and test splits remain at full size and identical across $n$, so within-seed AUC-versus-$n$ curves are exactly comparable; because generation is i.i.d., the prefix is an i.i.d.\ subsample.

\paragraph{Dataset sizes and seeds.}
For each regime we generate 50{,}000 training, 10{,}000 validation, and 10{,}000 test sequences of length $T = 24$ (the released suite also supports $T \in \{12, 36\}$; no reported experiment uses them). All experiments use five random seeds (42--46) governing both data generation and training, and compared models always share identical splits, so that any difference between models is attributable to the model class rather than the sample.

\section{Experimental Design}
\label{sec:design}

The evaluation maps each claim of Section~\ref{sec:intro} to exactly one experiment, so that a reviewer can audit any claim from one primary result; E9 re-tests C1 and C3 across capacity, and E11 is the constructive counterpart of C7 and C8. Each experiment is a point in the three-axis space of Section~\ref{sec:data}---a label regime, a data protocol, and a set of test operators---together with a model set, and the axis values are chosen to maximize the contrast that supports the claim rather than by historical accident. Trained models are shared across experiments where the measurements are independent (for example, E1 reuses the regime-B and regime-D models trained for E3), but no experiment is the sole evidence for two claims. Table~\ref{tab:claims-map} summarizes the mapping.

\subsection{Models and baselines}
\label{sec:models}

\paragraph{Proposed model.}
The monotone exposure RNN of Section~\ref{sec:architecture} (\texttt{monotone\_exposure\_rnn}), with the two-state EWMA/HWM recurrence and a monotone head.

\paragraph{Generic monotone RNN.}
A monotone recurrent baseline whose hidden state is a generic learned vector rather than the theory-guided two-state structure, evaluated in two recurrence forms. The shipped form is a bank of fixed-decay linear accumulators over a learned non-negative-weight monotone input map, $h_t = 0.9\, h_{t-1} + T_{\mathrm{in}}(x_t)$---monotone by construction and the simplest stable generic monotone state, but itself time-separable. Because that form cannot test whether a \emph{nonlinear} monotone recurrence escapes the separable family, the boundary evaluation adds a weight-constrained Elman form, $h_t = \tanh(W_{hh} h_{t-1} + T_{\mathrm{in}}(x_t))$ with $W_{hh} = \mathrm{softplus}(\Theta)$ entrywise non-negative and therefore order-preserving---the weight-constrained recurrence of the industrial monotone-RNN patent's core claims (Section~\ref{sec:rw-monotone-nn}) \citep{boardman2024mrnn}.\footnote{One departure from that patent's preferred embodiment is deliberate: its gated (LSTM) variant requires non-negative-range activations, which excludes tanh---the stronger activation on these regimes, curvature-wise (Section~\ref{sec:architecture})---so the ungated arm evaluates the patented mechanism at its most favorable activation. The gated embodiment is not evaluated: gating would add a second mechanism (input-conditioned decay) to the control, and a gated monotone recurrence is still time-invariant, so the time-localization argument of Section~\ref{sec:e10-results} covers it unchanged; its behavior on the other boundary regimes is untested here.} Both forms control for the possibility that the weight constraint alone, rather than the EWMA/HWM structure, drives the results; Section~\ref{sec:e10-results} evaluates both forms on the boundary regimes.

\paragraph{Static aggregate baselines.}
A monotone head, and an unconstrained counterpart, over fixed static aggregates of the sequence: sum, mean, maximum, minimum, and last value. The first four are order-invariant (the fixed-pooling case of the permutation-invariant class formalized by \citet{zaheer2017deepsets}); the last value is order-dependent and is retained because deployed scorecards carry it. These control for whether temporal state is needed at all: the monotone variant is itself monotone by construction, so any gap against the recurrent model is purely representational.

\paragraph{Calibrated monotone lattice comparator.}
A modern monotone architecture in the lattice tradition \citep{gupta2016lattice,you2017deep} over the same fixed monotone featurizations, included so that no result can be attributed to the monotone-MLP head class: per-feature monotone calibrators (convex combinations of non-negative-slope sigmoids of $\log(1+x)$, mapping each feature into $(0,1)$) feed a multilinear lattice on the unit hypercube whose M\"obius coefficients are constrained non-negative through a softplus parameterization---a sufficient condition for coordinatewise monotonicity, so the guarantee holds by construction with no projection step. The lattice consumes the five static aggregates in E3 (\texttt{aggregate\_lattice}) and each hand-crafted grid in E10; it is a comparator class only, never a component of the proposed model. The lattice is chosen over the other modern monotone classes of Section~\ref{sec:related-work} deliberately: it is exactly monotone by parameterization, whereas, e.g., UMNN's quadrature-based construction is monotone only up to numerical integration error in the multi-input setting and certified monotonic networks verify a trained network post hoc---neither meets this paper's requirement that every monotone model pass the zero-violation battery by construction.

\paragraph{Unconstrained sequence baselines.}
A standard RNN and an LSTM with no monotonicity constraint, representing what generic sequence learning achieves on the same data. We refer to the proposed model, the generic monotone RNN, the two static aggregate baselines, and the unconstrained RNN and LSTM collectively as the \emph{core} model set. We deliberately include no attention or transformer baseline: the claims concern monotone temporal aggregation against engineered feature grids, not state-of-the-art sequence modeling, the unconstrained LSTM is already scaled to roughly $200{,}000$ parameters in the capacity study (E9) without overturning parity, and an unconstrained transformer could only add to the violation counts that C1 already establishes for unconstrained sequence models.

\paragraph{Penalty-guided learning baselines (PGL).}
The unconstrained RNN and LSTM trained with a soft monotonicity penalty in the loss,
\begin{equation}
\label{eq:pgl}
\mathcal{L} \;=\; \mathcal{L}_{\mathrm{pred}} \;+\; \lambda_{\mathrm{mono}}\, \mathbb{E}\!\left[\max\big(0,\; f(x) - f(x') + \varepsilon\big)\right],
\end{equation}
where the expectation runs over monotone pairs $x \preceq x'$ sampled during training (Appendix~\ref{app:protocols})---the classic \emph{monotonicity hints} approach \citep{sill1997hints} instantiated for sequence models. These baselines address the central skeptical question behind C1: whether a soft penalty could achieve the same guarantee as the architecture.

\paragraph{Dropout-equipped variant.}
A copy of the proposed model with dropout inside its monotone MLP components, carried solely for the MC-dropout check in E1. Isolating dropout in a dedicated variant ensures that dropout never affects the predictive results reported elsewhere, while Corollary~\ref{cor:mcdropout} applies to every sampled mask of the variant.

\paragraph{Ablation grid (E7).}
The proposed model instantiated over $g_1 \in \{\text{identity},\ \log,\ \text{monotone MLP}\}$ $\times$ $g_2 \in \{\text{identity},\ \text{hinge},\ \text{power},\ \text{monotone MLP}\}$ $\times$ $\alpha \in \{0.70,\ 0.85,\ 0.95,\ \text{learned}\}$, giving $3 \times 4 \times 4 = 48$ configurations.

\paragraph{Hand-crafted grid baselines (E10).}
Two fixed monotone featurizations represent the engineered alternative to the learned channel. The \emph{practice grid} mirrors deployed trended-data attributes \citep{experian_trended3d}: trailing-window mean and maximum over windows $\{3, 6, 12, T\}$, breach counts at fixed thresholds $\{2, 5, 10\}$, and the last value (12 features). The \emph{steelman grid} is a densified hand-crafted approximation of the learnable $(g_1, \alpha, g_2)$ family that no practitioner ships but a skeptical reviewer would build: EWMA at fixed decays $\{0.5, 0.7, 0.85, 0.95\}$, the HWM, and cumulative hinge excesses plus breach counts at thresholds $\{2, 5, 10\}$ (11 features).

Each grid is consumed by (i) a monotone MLP head, monotone by Theorem~\ref{thm:score}'s composition argument, (ii) LightGBM with a non-decreasing monotone constraint on every feature \citep{ke2017lightgbm,lightgbm_monotone}---the industry's scorecard learner in its governance configuration---and (iii) the calibrated monotone lattice above. A further comparator feeds the \emph{raw lags} $x_1, \dots, x_T$ to constrained LightGBM as $T$ columns: order-aware, monotone by construction, and requiring no feature engineering at all. Every comparator is therefore monotone end-to-end, so E10 isolates the value of \emph{learned} aggregation rather than of the guarantee. The classical monotone-binned logistic scorecard is not run as a separate consumer: it is one more monotone head family over the same columns, and E10's head-class insensitivity (Section~\ref{sec:e10-results}) leaves little room for a fourth head to change any verdict; its piecewise-constant bins would saturate under the stress phase exactly as the tree ensembles do. One operator caveat applies: the adversarial monotonicity search is gradient-based and inapplicable to tree ensembles, so tree models are verified by the random, coordinatewise, and structured-probe operators and their adversarial cell is reported as not applicable rather than as zero.

\paragraph{Basis models M1--M9, in two arms (E8).}
Monotone exposure RNNs whose state tracks a subset of the admissible basis library of Section~\ref{sec:bases} and whose monotone head consumes all tracked statistics: M1 (EWMA only), M2 (HWM only), M3 (EWMA+HWM), M4 (M3 plus breach count), M5 (M4 plus longest run), M6 (M3 plus cumulative sum and breach count), M7 (M3 plus cumulative sum, breach count, longest run, and cumulative excess), M8 (M3 plus the low-water mark), and M9 (M3 plus the early-window EWMA). Each model is instantiated in two arms. The \emph{learned} arm keeps the proposed model's learnable transforms and decay on the EWMA and HWM channels. The \emph{frozen} arm (M1f--M9f) fixes both transforms to the identity and the decay to $\alpha = 0.9$, so the channels compute the literal fixed EWMA and HWM and only the head learns.

The frozen arm is required because the learned recurrence is itself a basis learner: with a learnable monotone $g_1$ and learned $\alpha$, the e-channel (the model's $e_t$ coordinate; likewise the m-channel for $m_t$) realizes geometrically weighted, time-separable monotone statistics $\sum_t \alpha^{T-t} g_1(x_t)$ (up to the reach of the bounded monotone MLP within that family)---a family that contains the cumulative sum, breach count, and cumulative excess as limiting cases, and whose rank closure contains the HWM itself (and, symmetrically, the LWM; Appendix~\ref{app:rank-closure} formalizes the closure). Basis-level necessity is therefore only measurable against fixed channels; the learned arm then quantifies how much of each gap the learnable channel recovers. Identity transforms and fixed decay are members of the admissible class, so both arms inherit the guarantee of Theorem~\ref{thm:score} by the composition argument of Section~\ref{sec:bases}.

\subsection{Test operators}
\label{sec:operators}

A test operator is an evaluation-time procedure applied to a trained model; operators generate no training data.

\paragraph{Random, adversarial, and coordinatewise perturbations.}
The random operator draws monotone perturbations $x \preceq x'$ at random and counts violations; the adversarial operator searches for worst-case monotone perturbations by projected gradient descent on the score with random restarts (full procedure in Appendix~\ref{app:protocols}); the coordinatewise operator increases a single time index at a time.

\paragraph{Structured probes P1--P5.}
P1 increases a single time step by a fixed amount; P2 increases a contiguous block placed either early or late in the sequence; P3 increases the index holding the current maximum exposure; P4 applies small increments to low-exposure indices one at a time; P5 perturbs exposures around large observed jumps, a detected proxy for latent stress-state transitions. Full definitions, including the paired construction for P4, are given in Appendix~\ref{app:probes}.

\paragraph{MC dropout.}
We draw $K = 20$ pinned dropout masks per (regime, seed) run---each a realized dropout sub-network---and verify $f^{(k)}(x) \le f^{(k)}(x')$ within each mask $k$ over a dedicated budget of $1{,}000$ random monotone pairs ($20$ base sequences $\times$ $50$ upward perturbations, base and perturbed inputs sharing the mask), reporting the per-mask violation rate and the maximum violation magnitude across masks.

\paragraph{Ensemble.}
A deep ensemble of independently trained instances, compared against MC dropout on predictive and uncertainty quality; the comparison is reported in Appendix~\ref{app:results-tables} (Table~\ref{tab:s1c}) and read in Section~\ref{sec:discussion}.

\subsection{Evaluation protocol}
\label{sec:eval-protocol}

Predictive quality is reported as the area under the ROC curve (AUC), log loss, Brier score, and expected calibration error (ECE; $10$ equal-width probability bins), each as mean $\pm$ standard deviation over the five seeds, with identical splits across compared models. Confidence intervals on paired deltas come from a two-level paired bootstrap that resamples seeds and then test examples on shared splits; with five seeds the outer level is necessarily coarse, so the intervals should be read as approximate rather than exact coverage statements: throughout the paper, \emph{statistically} ahead or behind means the joint criterion below rather than an exact-coverage claim. The suite reports many pairwise contrasts across regime--model cells with no multiplicity adjustment; individual interval endpoints should therefore not be over-read, and every verdict the paper leans on is required to hold jointly---the interval excluding zero, the per-seed delta carrying the same sign across all five seeds, and, where multiple comparators exist, the same direction against each of them. Deltas below ${\sim}0.002$ AUC are described as leans, not wins, even when they meet the joint criterion. A monotonicity violation is recorded when a monotone increase $x \preceq x'$ yields $f(x') < f(x) - \varepsilon$. All violation checks are computed in float64 with an explicit numerical tolerance $\varepsilon$, so that reported violations reflect genuine order inversions rather than floating-point noise. Tables print four AUC decimals where paired deltas at the $10^{-3}$ scale carry the verdict (E2, E3, E5, the E10 parity and delta tables) and three decimals where seed-level variation dominates (E8, E11, the E10 boundary table); prose quotes match the tables they draw from. Derived quantities (deltas, gaps, spreads) are computed from unrounded per-seed values and can differ from the difference of two rounded table operands in the last digit.

\subsection{Experiments}
\label{sec:experiments}

\paragraph{E1: monotonicity by construction (tests C1).}
Regimes B and D under the clean protocol, with the full operator suite (random, adversarial, coordinatewise, probes P1--P5, and MC dropout) applied to the core model set, the PGL baselines, and the dropout-equipped variant; we record violation counts and rates by operator, model, and seed. Regimes B and D are chosen because peak- and threshold-driven mechanisms are where unconstrained baselines violate most frequently, so the contrast between zero and non-zero violation counts is unambiguous, whereas a regime where baselines barely violate would be the weakest possible showcase. The clean protocol isolates the by-construction property: distribution shift is the job of E4, and bundling stressors would blur attribution. The random, adversarial, and coordinatewise operators establish worst-case zero violations; the structured probes localize \emph{where} the baselines fail, separating distributed micro-increment failures from single-step, block, peak-aligned, and stress-transition ones. The MC-dropout operator, run on the dropout-equipped variant, tests the per-sample guarantee of Corollary~\ref{cor:mcdropout}---the property that distinguishes this architecture from static monotone networks---while the contrast against the penalty-trained baselines is carried by the deterministic operators.

\paragraph{E2: necessity of temporal state (tests C2).}
Regime A swept over its data-generating decay $\alpha^*$, with regime C as a cross-mechanism control, under the clean protocol with coordinatewise and probe operators, comparing the proposed model against the static monotone aggregate baseline. For each $\alpha^*$ in the sweep, the label coefficients are recalibrated so that class prevalence (approximately $25\%$) and the logit slope on the standardized EWMA feature are held fixed, so only recency varies along the sweep (Appendix~\ref{app:protocols}). The expected signature is an inverted-U AUC gap as a function of $\alpha^*$: closing toward both sweep ends, where the generating EWMA degenerates to a static summary, and peaking at intermediate $\alpha^*$, where $e_T$ is a genuinely order-dependent decaying sum; the endpoint mechanics are analyzed with the results. Regimes B and D are deliberately avoided here: their labels depend only on the maximum, which the aggregate's $\max$ feature already captures, so they cannot exhibit the gap. Regime C, whose label the aggregate's order-invariant features largely recover, runs once as a control whose gap should shrink markedly relative to the within-A peak, ruling out the reading that the aggregate is simply a weak baseline everywhere. Both models are monotone by construction and are confirmed at zero violations throughout, so the gap is purely predictive: monotonicity is necessary but not sufficient.

\paragraph{E3: predictive cost of the constraint (tests C3).}
The core model set, plus the calibrated-lattice comparator over the same static aggregates, on regimes A--E under the clean protocol, with predictive metrics only (AUC, log loss, Brier score, ECE; mean $\pm$ std over seeds). The regime set is the full core spread precisely so that parity cannot be read as cherry-picking: it includes regimes where the inductive bias is exactly in family (A--C), misspecified in the transform (D), and misspecified in the head (E). We report per-regime deltas between the proposed model and the unconstrained RNN/LSTM baselines rather than a single average, so that any regime-specific deficit---in particular on the recency-driven regime A---is quantified rather than averaged away; each delta carries a two-level paired bootstrap $95\%$ confidence interval (Section~\ref{sec:eval-protocol}) judged against the $-0.01$ non-inferiority margin shared with E9. The regime-B and regime-D runs are shared with E1, which reuses the same trained models for its violation analysis.

\paragraph{E4: robustness to distribution shift (tests C4a).}
Regime C under the distribution-shift protocol run as severity sweeps, with predictive scoring plus the random and coordinatewise operators, on the proposed model, the static monotone aggregate, and the unconstrained RNN and LSTM. Regime C is chosen for headroom: its baseline discrimination leaves room to observe degradation, and both state channels are active so the shift genuinely stresses the mechanism, whereas regime A is near ceiling and regime E near floor, either of which would blunt the contrast. Each shift family---scale, upper-tail amplification, spike frequency, and mixture contamination---is swept over severity (Appendix~\ref{app:protocols}), so the result is an AUC-versus-severity curve rather than a single, cherry-pickable operating point. Because labels are frozen at pre-shift values, the battery is restricted to ordering-preserving covariate shifts. The verdict splits into two metrics: R1, violation counts, where the monotone models must remain at zero at every severity while unconstrained violations grow; and R2, AUC across severity, where the static aggregate's ranking quality should stay flat under broad ordering-preserving shift and the constrained recurrent model should degrade most gracefully among the sequence models. Calibration metrics are reported but degrade by design under exposure-inflating shift with frozen labels; they are not the C4a verdict.

\paragraph{E5: robustness to missingness (tests C4b).}
Mirrors E4 on regime C with the missingness protocol: MCAR masking at rates $\{5,10,15,20,30\}\%$ and block missingness with block lengths $\{2,4,6,8\}$, each followed by zero, forward-fill, or mean imputation, with the same model set and operators as E4. The block lengths extend to a third of the sequence length, so the grid covers outages long enough to stress the recency state. Retaining regime C keeps the two robustness experiments directly comparable, so shift and missingness are assessed against the same predictive baseline. Labels are computed from the unmasked sequences, so the protocol degrades the observations while leaving the ground-truth mechanism intact. We report AUC versus missingness rate per imputation scheme, and violation counts for every mechanism--imputation combination. The claim is supported if the monotone models preserve zero violations under every combination and degrade gracefully in AUC relative to the unconstrained and static-aggregate baselines.

\paragraph{E6: interpretability via shape recovery (tests C5).}
The proposed model with learned $g_1$, $g_2$, and $\alpha$, trained under the clean protocol on regimes A, B, and D, and evaluated with the shape-recovery operator. These regimes are chosen because their generating transforms span three qualitatively different admissible shapes---logarithmic (A), identity (B), and hinge (D)---so successful recovery cannot be attributed to all targets being log-like. We additionally ablate the activation of the learned monotone transforms (tanh, softplus, ReLU, sigmoid) to test the expressivity restriction of Section~\ref{sec:architecture}: convex activations should fail on the concave log target (regime A) and remain adequate on the linear and convex targets (B, D). We compare recovered and generating shapes through overlays scored by both RMSE on a common evaluation grid and correlation, and compare the recovered $\alpha$ to $\alpha^*$ in regime A, the one regime whose label depends on the recency channel (in the peak-driven regimes B and D the decay is not identified, since the label does not constrain it); a decay-recovery sweep additionally retrains the tanh arm at $\alpha^* \in \{0.5, 0.7, 0.95\}$ under the E2 recalibration, so recovery is tested across the decay range rather than at the reference point alone. The expected outcome is faithful recovery of $g_2$ and $\alpha$, with the identifiability of $g_1$ characterized per activation rather than asserted, since distinct $(g_1, \alpha)$ pairs can induce similar distributions of $e_T$.

\paragraph{E7: insensitivity of the guarantee to admissible design (tests C6).}
The 48-configuration ablation grid of Section~\ref{sec:models}, trained under the clean protocol on regimes A--D, with the monotonicity suite run on every configuration. The regimes are chosen so that each design knob is load-bearing in at least one of them: A is EWMA-only, exercising $g_1$ and $\alpha$ while $g_2$ is correctly inert; B (pure peak) and D (threshold, with generating $g_2^*$ a hinge) make the peak channel the dominant driver; and C is the mixed case. We report the violation count for every configuration--regime pair and the AUC distribution across the grid. The claim is supported if every configuration records zero violations---establishing that the guarantee is design-independent rather than an artifact of one tuned configuration---and accuracy stays within a tight band with no catastrophic configuration. Secondary observations, such as whether the shape of $g_2$ is immaterial even in the peak-driven regimes and whether $g_1 = \log$ and learned $\alpha$ emerge as mild optima, are reported as observations rather than claims.

\paragraph{E8: basis-function sufficiency (tests C7).}
The core regimes A--E plus the ten basis-adversarial regimes F--O, under the clean protocol, with the basis models M1--M9 in both the learned and frozen arms. The core regimes establish where EWMA+HWM should be near-oracle; each adversarial regime is a minimal adversary engineered around one admissible basis (or, for regimes H, I, and M, a small combination), so together they trace the sufficiency boundary rather than a single anecdote. The experiment runs in four phases:
\begin{enumerate}
  \item mutual-information rankings of individual bases against the label, including the incremental information of each basis over the pair $\{$EWMA, HWM$\}$---the headline incremental MI is permutation-null debiased, because the stratified plug-in estimator inflates conditional MI for discrete bases in small conditioning strata, and the raw estimate is recorded alongside (Appendix~\ref{app:protocols});
  \item a Pareto subset search (greedy forward selection over subset sizes) reporting the smallest admissible basis set reaching at least $95\%$ of the oracle AUC on the training regime C;
  \item the structured basis-model ablations in both arms, which form a ladder per adversarial regime---frozen subset without the target basis, frozen subset with it, and the learned subset without it---so that every performance gap is attributed either to the basis level or to the learnable channel;
  \item cross-regime generalization, training on regime C and evaluating on the held-out adversarial regimes.
\end{enumerate}
Non-monotone statistics (variance, trend, time since last breach) participate only in the oracle pool for comparison and are excluded from any admissible score, in keeping with Section~\ref{sec:setup}. All admissible subsets must record zero violations.

\paragraph{E9: capacity robustness (tests C1 and C3 across model sizes).}
The core model set re-instantiated at four capacity tiers---tiny, small, medium, and large---on regimes A and C under the clean protocol, reporting AUC, violation counts, and parameter counts per tier; the unconstrained LSTM spans roughly $360$ to $200{,}000$ parameters across the tiers. Regime C is the mixed default; regime A is the recency-only regime that most stresses the constrained class, so capacity is swept where the monotone model has the least slack rather than only where it is comfortable. This is a robustness check rather than a new claim: it tests that zero violations hold at every capacity, ruling out the reading that the guarantee is a small-model artifact, and that the predictive parity of E3 is not overturned when the unconstrained baselines are scaled up. Parity is tested formally rather than by inspection: per (size, regime) cell we report $\Delta\mathrm{AUC}$ between the monotone model and a fixed comparator (the unconstrained LSTM) with a two-level paired bootstrap 95\% confidence interval (Section~\ref{sec:eval-protocol}), declaring non-inferiority when the interval's lower bound is at least $-0.01$; the per-cell best unconstrained baseline serves as a conservative sensitivity comparator. A training-adequacy audit accompanies the sweep: runs that hit the epoch cap while validation loss was still improving within the early-stopping patience window are flagged, and parity verdicts in cells containing flagged baseline runs carry an explicit caveat, so parity cannot be attributed to undertrained large baselines. E9 also justifies the lean default capacity (the medium tier) used by E1--E8.

\paragraph{E10: learned vs.\ hand-crafted monotone aggregation (tests C8).}
The grid-baseline model set---the proposed model, the static monotone aggregate as the naive rung, the practice and steelman grids under monotone MLP heads and under calibrated monotone lattices, and the three constrained-LightGBM comparators (practice, steelman, raw lags)---compared in four phases. \emph{Parity ladder:} regimes A--E under the clean protocol, with paired two-level bootstrap deltas against the constrained gradient-boosted-machine (GBM) comparators (Section~\ref{sec:eval-protocol}) and a parsimony account (parameter and feature counts) per model. At full sample size on in-family mechanisms, parity is the expected outcome and is not the verdict: at or near a grid point (the default $\alpha^* = 0.9$ sits $0.05$ from the steelman grid's nearest decay), the grid is essentially as in-family as the learned channel. \emph{Misspecification curves:} the E2 decay sweep and an analogous threshold sweep on regime D (Appendix~\ref{app:protocols}), reporting each baseline's paired AUC gap against the learned channel as a function of the distance from the generating $\alpha^*$ or $\tau^*$ to the nearest grid value. This is the designed contrast: the grids' decays and thresholds are fixed ex ante, so their misspecification cost is the price of hand-crafting, while the learned channel adapts. \emph{Sample-size sweep:} training-set sizes $\{1000, \dots, 50000\}$ on regimes A and C with fixed full validation/test splits, testing whether the compact learned channel degrades more slowly than the multi-column grids and the tree ensembles as labeled data becomes scarce. Because the suite calibrates prevalence to ${\sim}25\%$---far above rare-event default rates---a low-prevalence sensitivity accompanies the sweep: regime C recalibrated to $5\%$ prevalence (intercept only; slopes and decay unchanged), rerun at full $n$ and $n = 1000$ against the practice-grid and raw-lags comparators, so the small-sample contrast is checked under class rarity and not only label scarcity. \emph{Stress extrapolation:} the input-displacing E4 severity ladders (scale and upper tail) on regime C, augmented with score-degeneracy diagnostics (tie rate and the fraction of scores saturated at the maximum), which make the predicted structural contrast---piecewise-constant ensembles going flat beyond the training support---measurable. All monotone models must record zero violations on every applicable operator throughout.

\paragraph{E11: functional emulation of hand-crafted statistics (corollary of C7 and C8).}
E8 and E10 measure subsumption \emph{predictively}: the learned channel closes accuracy gaps without being given the target basis or grid feature. E11 makes the same statement \emph{constructively}, by regression. A single-basis learned channel---the e-channel (learned $g_1$, learned $\alpha$), the m-channel (learned $g_2$ feeding $\max$), the explicit LWM channel, the explicit early-window channel (the e-channel recurrence with increments accepted only for $t \le t_0$), or the explicit run-length channel (the current-run/max-run state above a fixed threshold)---is trained by direct mean-squared-error regression to reproduce a fixed target statistic of the exposure path, standardized by train-split mean and scale (a monotone affine map), and scored by held-out $R^2$ and Spearman rank correlation. The eighteen targets span four structural families: \emph{separable accumulations} (EWMA at the steelman grid's decays plus off-grid probes, cumulative sum, breach count, cumulative excess, last value), which the $(g_1, \alpha)$ family provably spans (Proposition~\ref{prop:rank-closure}; for the breach count, up to the tie caveat discharged in Appendix~\ref{app:rank-closure}); \emph{trailing-window aggregates} (window means at $w \in \{3, 12\}$), whose step-shaped weight profile no member of the geometric family realizes (Appendix~\ref{app:rank-closure}; the rank-closure status of this one tier is open); \emph{extremal statistics} (HWM and windowed maxima); and \emph{negative controls} (longest run, an early-window EWMA, and the LWM), engineered to lie outside the separable family---the run and the early window provably outside even its rank closure, the floor inside the closure but outside the family (Proposition~\ref{prop:rank-closure} and the remark following it). The fitted regressors are channel diagnostics and never enter a score, so admissibility is untouched. Expected signature: rank correlation ${\approx}1$ on the separable family---including decays off every grid---high but imperfect on windows, and clearly reduced on the controls, mirroring the E8 boundary functionally. Because the window tier is the one family with a measured ceiling and no designed adversary, a companion diagnostic prices it predictively: regime P (Appendix~\ref{app:protocols}), whose label is a pure trailing-window sum at a practice-grid window ($k = 6$), trains the two-state channel, the learned repair rung with the recent-sum basis, and the practice-grid and raw-lags comparators under the clean protocol, with the per-seed Bayes index as the ceiling.

\begin{table}[t]
\centering
\caption{Claims-to-experiments map. ``Core'' denotes the six-model core set of Section~\ref{sec:models}; ``lattice'' the calibrated-lattice comparator; ``robustness set'' the proposed model, the static monotone aggregate, and the unconstrained RNN and LSTM; ``grid set'' the grid-baseline set of Section~\ref{sec:models} including the grid lattices. E10's ``$n$'' is the training-size subsampling phase (regimes A and C), which carries the $5\%$-prevalence sensitivity; E11's regime P is its predictive window-price diagnostic (two-state channel, recent-sum rung, and grid comparators against the Bayes ceiling).}
\label{tab:claims-map}
\scriptsize
\setlength{\tabcolsep}{1pt}
\begin{tabular}{@{}llllll@{}}
\toprule
Exp & Claim & Regimes & Protocol & Operators & Models \\
\midrule
E1 & C1 & B, D & clean & rand., adv., coord., P1--P5, MC drop. & core + PGL + dropout var. \\
E2 & C2 & A (swept), C & clean & coordinate, probes & proposed, mono.\ aggregate \\
E3 & C3 & A--E & clean & predictive only & core + lattice \\
E4 & C4a & C & dist.\ shift & predictive; rand., coord. & robustness set \\
E5 & C4b & C & missingness & predictive; rand., coord. & robustness set \\
E6 & C5 & A, B, D & clean & shape recovery & proposed (learned $g_1, g_2, \alpha$) \\
E7 & C6 & A--D & clean & monotonicity suite & ablation grid (48) \\
E8 & C7 & A--O & clean & deb.\ MI, Pareto, abl., cross-reg. & basis M1--M9, 2 arms \\
E9 & C1, C3 & A, C & clean & predictive; monotonicity; non-inf. & core $\times$ 4 sizes \\
E10 & C8 & A--E; A, D (swept); $n$; C & clean; subsamp.; shift & paired deltas; misspec.; degen. & grid set \\
E11 & C7/C8 (constr.) & 18 targets; P & clean & MSE regr.; $R^2$, Spearman & single-basis channels \\
\bottomrule
\end{tabular}
\end{table}

\section{Results}
\label{sec:results}

This section reports experiments E1--E10 and the functional-emulation study E11 against their claims; all statistics are aggregated over five seeds (mean $\pm$ standard deviation unless stated otherwise). E1--E7 and E9 certify the expected properties of the instrument---the guarantee, its cost, and its robustness---on the core regimes A--E, whose labels are drawn from the same EWMA/HWM family the model parameterizes; parity and shape recovery there are in-family consistency checks of a well-specified model rather than evidence about misspecified settings, and are reported compactly. E8, E10, and E11 carry the paper's substantive findings---the out-of-family boundary and the misspecification cost---and are reported in full. Table~\ref{tab:verdicts} summarizes the verdicts so that each can be audited from its primary subsection.

\begin{table}[t]
\centering
\caption{Claim-by-claim verdicts, in the wording of the primary results subsections.}
\label{tab:verdicts}
\small
\begin{tabular}{llp{9.2cm}}
\toprule
Claim & Experiments & Verdict \\
\midrule
C1 & E1, E9 & Supported in full: zero violations on every operator and per MC-dropout sample, identically at every capacity tier. \\
C2 & E2 & Supported: inverted-U decay-sweep gap over the static aggregate, peak $+0.044$ ($+0.034$ under a tanh aggregate head), closing toward both sweep ends. \\
C3 & E3, E9 & Supported: worst paired CI lower bound $-0.0017$ across the ten regime--baseline comparisons; capacity-robust from the small tier upward, with the four-parameter tiny tier's one miss an expressiveness floor. \\
C4a & E4 & Holds as stated: guarantee exact under every ordering-preserving shift; the static aggregate is the more shift-robust scorer, and the learned channel degrades most gracefully among sequence models. \\
C4b & E5 & Supported, with its boundary quantified rather than asserted. \\
C5 & E6 & Supported in asymmetric form: $g_2$ and $\alpha$ identifiable regardless of parameterization, $g_1$ given a non-convex monotone activation. \\
C6 & E7 & Supported: zero violations at every configuration; the accuracy band is regime-dependent (${\le}0.003$ on the peak regimes, $0.035$ on the recency regime). \\
C7 & E8; E11 & Supported in sharpened form: the separable family is subsumed; the run and the early window fail structurally, the floor as a learnability boundary, each repaired by one admissible basis. \\
C8 & E10; E11 & Supported as a division of labor: parity at full sample size on mechanisms in or near the learned family; advantages on uncovered decays, scarce data, and stress; symmetric boundaries at extreme severity (grid-fed neural heads) and on the boundary regimes (raw lags). \\
\bottomrule
\end{tabular}
\end{table}

\subsection{E1: Monotonicity by construction (C1)}
We evaluate the full operator suite (random, adversarial, coordinatewise, structured probes P1--P5, and per-sample MC dropout) on regimes B and D. Table~\ref{tab:e1-violations} reports violation rates pooled over five seeds. Every monotone-by-construction model---the proposed exposure RNN, its MC-dropout variant, the generic monotone RNN, and the static monotone aggregate---records \emph{zero} violations on all $128{,}724$ ordered test pairs per model (the random, adversarial, and coordinatewise checks verified in double precision; the structured probes in inference precision, Appendix~\ref{app:repro}). The dropout-equipped variant is violation-free for every one of its $200$ sampled dropout sub-networks ($20$ pinned masks per regime--seed run, each verified over its dedicated $1{,}000$-pair budget), with maximum violation magnitude exactly zero, instantiating Corollary~\ref{cor:mcdropout} empirically. The unconstrained baselines violate on $49{,}601$--$51{,}958$ pairs ($38.5$--$40.4\%$ of their test budgets), with the highest rates concentrated on the adversarial operator and the structured probes; the totals are not an artifact of adversarial search, since under natural random perturbations alone the same models violate on $8.6$--$13.7\%$ of pairs on regime B and $31.9$--$40.5\%$ on regime D (Table~\ref{tab:e1-violations}). Penalty training reduces the totals only to $44{,}428$--$45{,}837$ ($34.5$--$35.6\%$) and still leaves $25$--$49\%$ violation rates on the adversarial and probe operators; the rates are quoted at the disclosed penalty weight ($\lambda = 1$, Appendix~\ref{app:protocols}), and while heavier weights would trade accuracy for lower rates, the property at issue is qualitative---a training-time penalty shrinks violations near the training distribution but provides no guarantee away from it. C1 is supported in full.

\begin{table}[t]
\centering
\caption{E1 monotonicity violation rates by operator family (pooled over five seeds; per-seed budgets: $2{,}500$ random, $150$ adversarial, $7{,}200$ coordinatewise, and ${\sim}3{,}000$ structured-probe pairs per regime, the probe counts being data-dependent by construction, Appendix~\ref{app:probes}). Monotone-by-construction models (top block) record zero violations on every operator; penalty-trained (PGL) baselines reduce but do not eliminate violations; unconstrained models violate on roughly a third to nearly two-thirds of non-random tests. The MC-dropout variant is additionally violation-free for each of its $200$ sampled dropout sub-networks.}
\label{tab:e1-violations}
\small
\setlength{\tabcolsep}{5pt}
\begin{tabular}{lcccccccc}
\toprule
 & \multicolumn{4}{c}{Regime B} & \multicolumn{4}{c}{Regime D} \\
\cmidrule(lr){2-5}\cmidrule(lr){6-9}
Model & Rand. & Adv. & Coord. & Probes & Rand. & Adv. & Coord. & Probes \\
\midrule
Monotone exposure RNN & $0$ & $0$ & $0$ & $0$ & $0$ & $0$ & $0$ & $0$ \\
Monotone exposure RNN (MC) & $0$ & $0$ & $0$ & $0$ & $0$ & $0$ & $0$ & $0$ \\
Generic monotone RNN & $0$ & $0$ & $0$ & $0$ & $0$ & $0$ & $0$ & $0$ \\
Monotone aggregate & $0$ & $0$ & $0$ & $0$ & $0$ & $0$ & $0$ & $0$ \\
\midrule
Unconstrained RNN + PGL & $6.3\%$ & $30.7\%$ & $34.3\%$ & $35.7\%$ & $28.4\%$ & $48.3\%$ & $43.8\%$ & $49.0\%$ \\
Unconstrained LSTM + PGL & $6.7\%$ & $25.2\%$ & $36.0\%$ & $36.1\%$ & $28.8\%$ & $43.7\%$ & $39.8\%$ & $44.6\%$ \\
\midrule
Unconstrained aggregate & $8.6\%$ & $44.0\%$ & $37.5\%$ & $39.2\%$ & $33.5\%$ & $58.8\%$ & $51.8\%$ & $52.0\%$ \\
Unconstrained RNN & $9.5\%$ & $40.7\%$ & $39.3\%$ & $40.0\%$ & $40.5\%$ & $62.7\%$ & $42.7\%$ & $46.6\%$ \\
Unconstrained LSTM & $13.7\%$ & $31.3\%$ & $36.7\%$ & $37.8\%$ & $31.9\%$ & $45.3\%$ & $46.7\%$ & $52.8\%$ \\
\bottomrule
\end{tabular}

\end{table}

\subsection{E2: Necessity of temporal state (C2)}
Figure~\ref{fig:e2-gap} and Table~\ref{tab:e2-gap} report the decay sweep on regime A. The recurrent model holds AUC between $0.837$ and $0.848$ across the entire grid, while the static aggregate dips and recovers, producing the predicted inverted-U gap curve: the paired AUC gap is $+0.0018 \pm 0.0014$ at $\alpha^*{=}0.1$, climbs to a peak of $+0.0439 \pm 0.0030$ at $\alpha^*{=}0.8$, and closes again to $+0.0071 \pm 0.0015$ at $\alpha^*{=}0.99$. Both tails behave as the design's degeneracy prediction (Section~\ref{sec:experiments}) anticipates, with one instructive asymmetry. At $\alpha^* \to 0$ the generating summary degenerates to $\log(1 + x_T)$, a monotone reparameterization of the aggregate's last-value feature, and the gap nearly closes ($+0.0018$); at $\alpha^* \to 1$ it degenerates to the order-invariant $\sum_t \log(1 + x_t)$, so the temporal advantage vanishes but a transform mismatch against the aggregate's raw-sum feature remains, visible in the larger right-endpoint residual ($+0.0071$). Intermediate decay creates genuinely temporal structure that no static summary captures. On the regime-C control, where the aggregate's $\max$ feature recovers the label, the gap shrinks to $+0.0103 \pm 0.0018$, less than a quarter of the within-A peak. One sensitivity is measured rather than assumed: the shipped aggregate head uses a ReLU activation, which under non-negative weights is convex-only (Section~\ref{sec:architecture})---the same restriction E6 diagnoses for the proposed model's transforms---while the proposed model's monotone components use tanh. Retraining the aggregate with a tanh head on the full sweep (same seeds and splits) shifts its AUC by at most $+0.0096 \pm 0.0006$ (at $\alpha^*{=}0.8$; $+0.0009$/$+0.0004$ at the endpoints, $+0.0006$ on the regime-C control), so the peak gap reads $+0.0343$ against the strongest aggregate head: the inverted-U, its location, and both endpoint degeneracies are unchanged, and roughly a fifth of the peak magnitude is attributable to the baseline's head class rather than to missing temporal state. Both models record zero violations at every grid point, isolating the gap as purely predictive. C2 is supported.

\begin{figure}[t]
\centering
\includegraphics[width=\linewidth]{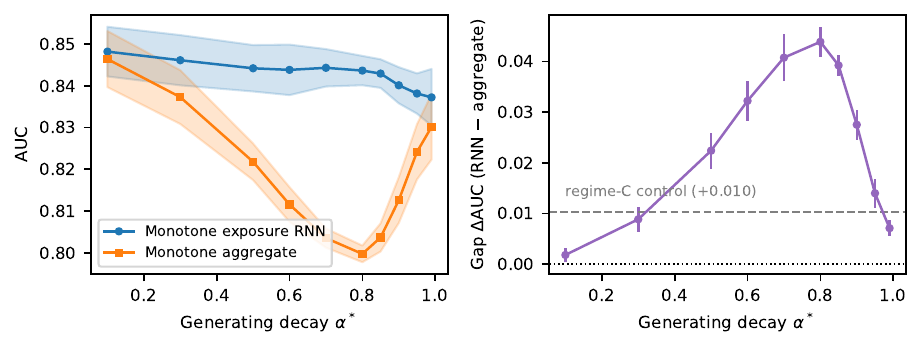}
\caption{E2 decay sweep on regime A (five seeds; bands and bars are $\pm 1$ std). \emph{Left:} test AUC of the recurrent monotone model and the static monotone aggregate as the generating decay $\alpha^*$ varies. \emph{Right:} the paired AUC gap forms the predicted inverted-U, peaking at $+0.044$ at $\alpha^*{=}0.8$ and closing toward both endpoints where the generating EWMA degenerates to static summaries (the last value at $\alpha^*{\to}0$, a saturating sum at $\alpha^*{\to}1$); the dashed line marks the regime-C control gap.}
\label{fig:e2-gap}
\end{figure}

\begin{table}[t]
\centering
\caption{E2 AUC and paired gap by generating decay $\alpha^*$ on regime A, with the regime-C control (mean $\pm$ std over five seeds; the gap is the per-seed paired difference). Both models record zero monotonicity violations at every grid point.}
\label{tab:e2-gap}
\small
\begin{tabular}{lccc}
\toprule
Decay $\alpha^*$ & Monotone exposure RNN & Monotone aggregate & Gap $\Delta$AUC \\
\midrule
0.10 & $0.8482 \pm 0.0060$ & $0.8464 \pm 0.0067$ & $0.0018 \pm 0.0014$ \\
0.30 & $0.8461 \pm 0.0060$ & $0.8373 \pm 0.0064$ & $0.0088 \pm 0.0024$ \\
0.50 & $0.8442 \pm 0.0056$ & $0.8218 \pm 0.0044$ & $0.0224 \pm 0.0036$ \\
0.60 & $0.8438 \pm 0.0061$ & $0.8116 \pm 0.0041$ & $0.0322 \pm 0.0039$ \\
0.70 & $0.8443 \pm 0.0045$ & $0.8035 \pm 0.0023$ & $0.0408 \pm 0.0046$ \\
0.80 & $0.8436 \pm 0.0035$ & $0.7997 \pm 0.0020$ & $0.0439 \pm 0.0030$ \\
0.85 & $0.8429 \pm 0.0034$ & $0.8037 \pm 0.0034$ & $0.0393 \pm 0.0021$ \\
0.90 & $0.8401 \pm 0.0043$ & $0.8126 \pm 0.0055$ & $0.0275 \pm 0.0030$ \\
0.95 & $0.8381 \pm 0.0048$ & $0.8242 \pm 0.0066$ & $0.0140 \pm 0.0028$ \\
0.99 & $0.8372 \pm 0.0068$ & $0.8301 \pm 0.0078$ & $0.0071 \pm 0.0015$ \\
\midrule
Regime C (control) & $0.8535 \pm 0.0045$ & $0.8432 \pm 0.0050$ & $0.0103 \pm 0.0018$ \\
\bottomrule
\end{tabular}

\end{table}

\subsection{E3: Predictive cost of the constraint (C3)}
Table~\ref{tab:e3-auc} reports AUC for the core model set and the calibrated-lattice comparator on regimes A--E; log loss, Brier score, and ECE appear in Appendix~\ref{app:results-tables}. Per-regime deltas, not a single average, carry the verdict. Against the stronger of the two unconstrained sequence baselines, the constrained model's AUC delta is $-0.0004$ (A), $+0.0023$ (B), $+0.0003$ (C), $+0.0005$ (D), and $+0.0001$ (E), and the constrained model posts the best mean AUC in four of the five regimes. Table~\ref{tab:e3-delta} reports fully paired two-level bootstrap $95\%$ CIs on the shared test splits (the Section~\ref{sec:eval-protocol} machinery, shared with the E5 and E9 verdicts). Across all ten regime--baseline comparisons the worst lower bound is $-0.0017$, well inside the $-0.01$ non-inferiority margin shared with E9 (a design constant, not an externally registered threshold); the verdict also holds at margin $0.005$. The one negative point estimate, regime A against the LSTM ($-0.0004$, CI $[-0.0009, +0.0000]$), bounds any residual cost in the pure-recency regime below a tenth of an AUC point; the verdict there is non-inferiority, not superiority. The training budget is asymmetric only in the conservative direction: the constrained model trains to the 100-epoch cap in 23 of 25 runs while the unconstrained sequence baselines early-stop well inside it (28--96 epochs), so a larger budget could only improve the constrained side. The only model with a visible deficit is the static monotone aggregate in the recency-driven regime A ($0.8120$ vs.\ $0.8397$), which is the E2 phenomenon, not a cost of the monotonicity constraint. Neither a modern monotone head class nor a non-convex activation repairs a static featurization: the calibrated-lattice comparator over the same five aggregates narrows the deficit only marginally ($0.8185$), retraining the aggregate's own ReLU head with tanh reaches $0.8177$ (paired head effect $+0.0057 \pm 0.0008$ on regime A, within $\pm 0.0012$ on B--E), and both tie the aggregate baselines elsewhere ($\pm 0.002$ on regimes B--E), with zero violations throughout---the regime-A gap is representational, not an artifact of the head family. Calibration tells the same story: the constrained model's ECE sits at $0.011$--$0.013$ across regimes, indistinguishable from the unconstrained sequence baselines ($0.008$--$0.016$). On mechanisms within the model's representational class, the constraint costs nothing measurable; C3 is supported.

\begin{table}[t]
\centering
\caption{E3 test AUC (mean $\pm$ std over five seeds) for the core model set plus the calibrated-lattice comparator on regimes A--E; best per regime in bold. Column abbreviations: monotone exposure RNN, generic monotone RNN, monotone aggregate, calibrated lattice over the same aggregates, unconstrained aggregate, unconstrained RNN, unconstrained LSTM. The constrained exposure RNN matches or exceeds the unconstrained sequence baselines in every regime; log loss, Brier, and ECE are reported in Appendix~\ref{app:results-tables}.}
\label{tab:e3-auc}
\scriptsize
\setlength{\tabcolsep}{2.5pt}
\begin{tabular}{lccccccc}
\toprule
Regime & Mono.\ exp.\ RNN & Gen.\ mono.\ RNN & Mono.\ agg. & Lattice agg. & Unc.\ agg. & Unc.\ RNN & Unc.\ LSTM \\
\midrule
A & $0.8397 \pm 0.0043$ & $0.8308 \pm 0.0042$ & $0.8120 \pm 0.0056$ & $0.8185 \pm 0.0047$ & $0.8212 \pm 0.0054$ & $0.8396 \pm 0.0040$ & $\mathbf{0.8402 \pm 0.0041}$ \\
B & $\mathbf{0.7540 \pm 0.0087}$ & $0.7302 \pm 0.0056$ & $0.7538 \pm 0.0089$ & $0.7534 \pm 0.0086$ & $0.7533 \pm 0.0084$ & $0.7501 \pm 0.0083$ & $0.7516 \pm 0.0087$ \\
C & $\mathbf{0.8535 \pm 0.0045}$ & $0.8461 \pm 0.0038$ & $0.8432 \pm 0.0050$ & $0.8450 \pm 0.0046$ & $0.8450 \pm 0.0051$ & $0.8519 \pm 0.0046$ & $0.8531 \pm 0.0048$ \\
D & $\mathbf{0.7142 \pm 0.0109}$ & $0.7061 \pm 0.0070$ & $0.7129 \pm 0.0089$ & $0.7139 \pm 0.0110$ & $0.7139 \pm 0.0091$ & $0.7105 \pm 0.0123$ & $0.7136 \pm 0.0102$ \\
E & $\mathbf{0.8757 \pm 0.0033}$ & $0.8635 \pm 0.0031$ & $0.8709 \pm 0.0045$ & $0.8709 \pm 0.0040$ & $0.8710 \pm 0.0045$ & $0.8746 \pm 0.0038$ & $0.8756 \pm 0.0035$ \\
\bottomrule
\end{tabular}

\end{table}

\begin{table}[t]
\centering
\caption{E3 paired AUC deltas: monotone exposure RNN minus each unconstrained sequence baseline, per regime. Deltas are computed within (regime, seed) cells on the shared test split; intervals are two-level paired bootstrap $95\%$ CIs over the stored per-example predictions (resampling seeds, then test examples on shared indices). Every lower bound clears the $-0.01$ non-inferiority margin prespecified for the capacity study (E9).}
\label{tab:e3-delta}
\footnotesize
\begin{tabular}{lcccc}
\toprule
 & \multicolumn{2}{c}{vs.\ Unc.\ RNN} & \multicolumn{2}{c}{vs.\ Unc.\ LSTM} \\
\cmidrule(lr){2-3}\cmidrule(lr){4-5}
Regime & $\Delta$AUC & 95\% CI & $\Delta$AUC & 95\% CI \\
\midrule
A & $+0.0001$ & $[-0.0006, +0.0009]$ & $-0.0004$ & $[-0.0009, +0.0000]$ \\
B & $+0.0038$ & $[+0.0022, +0.0053]$ & $+0.0023$ & $[+0.0013, +0.0033]$ \\
C & $+0.0016$ & $[+0.0008, +0.0026]$ & $+0.0003$ & $[-0.0002, +0.0009]$ \\
D & $+0.0037$ & $[-0.0003, +0.0078]$ & $+0.0005$ & $[-0.0017, +0.0025]$ \\
E & $+0.0010$ & $[+0.0000, +0.0022]$ & $+0.0001$ & $[-0.0006, +0.0008]$ \\
\bottomrule
\end{tabular}

\end{table}

\subsection{E4: Robustness to distribution shift (C4a)}
On regime C, models trained on the base distribution are evaluated under severity sweeps of four covariate-shift families (Appendix~\ref{app:protocols}). (R1) \emph{Violations:} both monotone models remain at zero across all 15 shifted severities (16 evaluation conditions including the unshifted baseline), while the unconstrained RNN and LSTM accumulate $51{,}596$ and $52{,}429$ violations respectively, concentrated in the scale and upper-tail families (Table~\ref{tab:e4-violations}). (R2) \emph{Ranking quality:} Figure~\ref{fig:e4-shift} shows AUC across severity. Under input-displacing shifts the constrained recurrent model degrades most gracefully among the sequence models: at scale $\times 8$ its AUC falls from $0.8535$ to $0.8147$ ($-0.039$), versus $-0.054$ for the unconstrained LSTM ($0.7996$) and $-0.116$ for the unconstrained RNN ($0.7363$); at upper-tail $\gamma{=}10$ the ordering is the same ($0.8225$ vs.\ $0.8064$ and $0.7806$). The static monotone aggregate is essentially flat (exactly $0.843$ at every scale severity, and within $0.003$ under upper-tail amplification, where the kinked amplification map does not commute with the aggregate's sums and means, so flatness is only approximate), the expected behavior when a shift displaces every sequence coherently: each aggregate feature is transformed monotonically across the whole test set, leaving its ranks essentially unchanged. Spike-frequency shifts---the designed null case---move all models by at most $0.022$. One family is an exception: mixture contamination rescales a random $(1-\pi)$ fraction of test rows by $\times 3$ while leaving the rest at base (pathwise $x' \succeq x$ row by row, so the frozen labels remain valid), and it degrades all models together (at contamination $1-\pi=0.5$: $0.7751$ for the constrained model vs.\ $0.7672$--$0.7737$ for the baselines). The shared degradation is mechanical. Selective per-row displacement destroys cross-sequence score comparability---a contaminated negative leapfrogs base positives---and no scorer consuming the displaced inputs can undo it: provably for any monotone scorer, whose score on a pathwise-dominating row cannot decrease, and empirically here for the unconstrained baselines as well. This is also why the otherwise-flat aggregate degrades under contamination but not under uniform scaling. These behaviors are specific to shifts that preserve the input ordering, which is how C4a is stated; the flatness belongs to the static aggregate, not to the learned channel, so on this axis the engineered aggregate is strictly the more shift-robust scorer. Calibration metrics degrade across all families by design, since labels are frozen at pre-shift exposures. C4a holds in exactly this form.

\begin{figure}[tp]
\centering
\includegraphics[width=\linewidth]{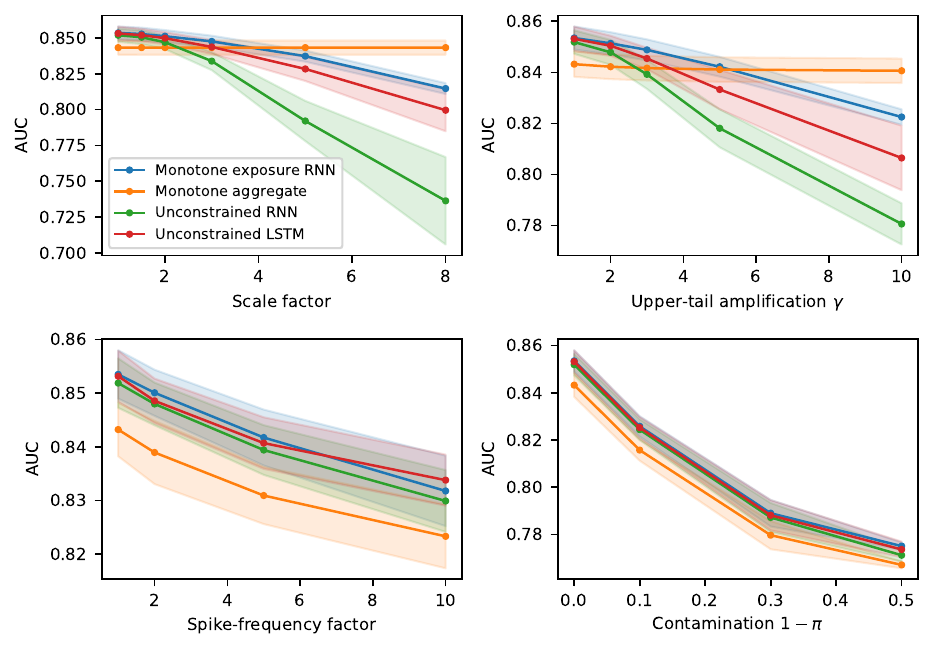}
\caption{E4 test AUC under severity sweeps of the four shift families on regime C (five seeds; bands are $\pm 1$ std). The static monotone aggregate is flat under input-displacing shifts (exactly under scale, within $0.003$ under upper tail); the constrained recurrent model degrades most gracefully among sequence models; mixture contamination (bottom right) degrades all models together: its selective per-row displacement destroys cross-sequence score comparability (see text).}
\label{fig:e4-shift}
\end{figure}

\begin{table}[t]
\centering
\caption{E4 monotonicity violations by shift family (counts summed over severities and five seeds, identical test budgets across models). Monotone models remain at zero under every shift; unconstrained violations concentrate where inputs are displaced furthest from the training support.}
\label{tab:e4-violations}
\small
\begin{tabular}{lrrrrrr}
\toprule
Model & Baseline & Scale & Upper tail & Spike freq. & Mixture & Total \\
\midrule
Monotone exposure RNN & 0 & 0 & 0 & 0 & 0 & 0 \\
Monotone aggregate & 0 & 0 & 0 & 0 & 0 & 0 \\
Unconstrained RNN & 1,302 & 23,758 & 15,960 & 4,083 & 6,493 & 51,596 \\
Unconstrained LSTM & 1,182 & 24,015 & 17,152 & 4,408 & 5,672 & 52,429 \\
\bottomrule
\end{tabular}

\end{table}

\subsection{E5: Robustness to missingness (C4b)}
Figure~\ref{fig:e5-missingness} reports AUC across the missingness grid on regime C. From the clean anchor ($0.8535$), the constrained model degrades gracefully under MCAR: $-0.012$ at $30\%$ missingness under forward-fill and $-0.021$ under zero imputation. Contiguous blocks are harder, as the recency structure predicts: $-0.046$ at block length $8$ with zero imputation, recovered to $-0.040$ by forward-fill. The constrained model posts the best or tied-best AUC among the four models in every cell except at the longest block (zero imputation: the static aggregate edges ahead, $0.8087$ vs.\ $0.8079$ at length $8$; mean: a within-noise $0.0005$ LSTM edge). Both monotone models record zero violations under every mechanism--imputation combination, while the unconstrained RNN and LSTM accumulate $2{,}373$ and $2{,}425$ violations (Table~\ref{tab:baseline-violations}): imputation composed with the architecture preserves the guarantee because all three imputation operators are monotone. Table~\ref{tab:e5-noninf} reports the headline paired non-inferiority test on degradation, $\Delta\Delta\text{AUC} = \text{degradation(monotone)} - \text{degradation(baseline)}$, with two-level paired bootstrap $95\%$ CIs. Against the unconstrained RNN and LSTM, all $54$ contrasts pass at margin $0.005$. Against the static monotone aggregate, $25$ of $27$ contrasts pass; the two failures are zero-imputed contiguous blocks of length $6$ and $8$---the structural caveat of recency weighting analyzed in Section~\ref{sec:discussion}. Forward-fill removes the effect ($0.8132$ vs.\ $0.8035$ at block length $8$). C4b is supported, with its boundary quantified rather than asserted.

\begin{figure}[tp]
\centering
\includegraphics[width=\linewidth]{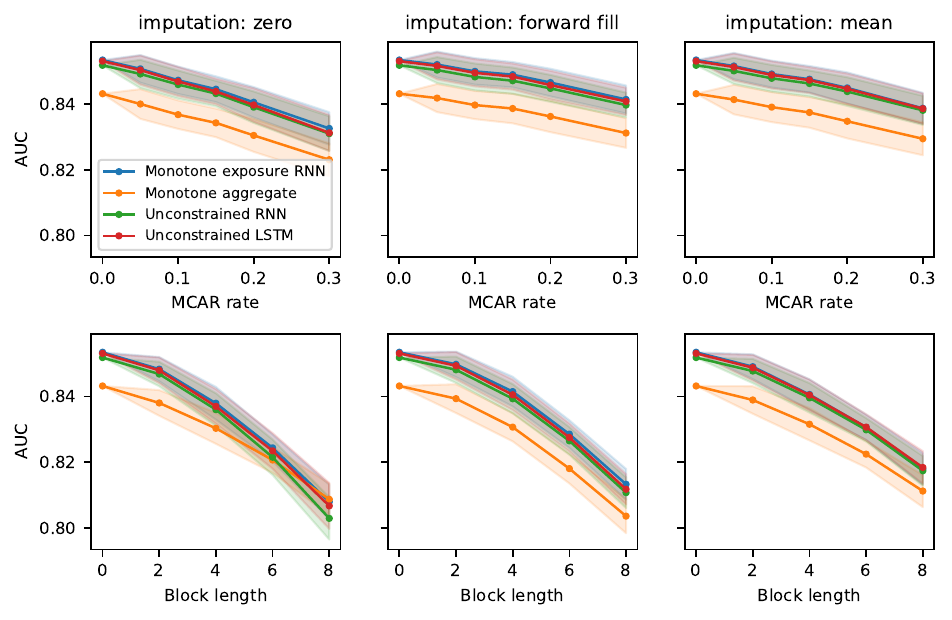}
\caption{E5 test AUC under MCAR (top) and contiguous-block (bottom) missingness on regime C, by imputation (five seeds and three mask draws; bands are $\pm 1$ std). The constrained model degrades gracefully and keeps the best or tied-best AUC across the grid; under zero imputation, lengthening blocks erode its advantage over the static aggregate---the structural caveat of Section~\ref{sec:discussion}---while forward-fill preserves it.}
\label{fig:e5-missingness}
\end{figure}

\begin{table}[t]
\centering
\caption{E5 paired non-inferiority summary: $\Delta\Delta$AUC $=$ degradation(monotone) $-$ degradation(baseline) with two-level paired bootstrap $95\%$ CIs (seeds $\times$ examples, mask draws averaged); a condition passes when the CI upper bound is at most $0.005$. Cells show the worst (largest) mean and CI upper bound across missingness levels and baselines in each group. All $54$ contrasts against the unconstrained baselines pass; the only failures are zero-imputed blocks of length $6$ and $8$ against the static monotone aggregate.}
\label{tab:e5-noninf}
\footnotesize
\setlength{\tabcolsep}{4pt}
\begin{tabular}{llcccccc}
\toprule
 & & \multicolumn{3}{c}{vs.\ unconstrained (RNN, LSTM)} & \multicolumn{3}{c}{vs.\ monotone aggregate} \\
\cmidrule(lr){3-5}\cmidrule(lr){6-8}
Mechanism & Imputation & worst $\Delta\Delta$AUC & worst CI$_{95}$ hi & pass & worst $\Delta\Delta$AUC & worst CI$_{95}$ hi & pass \\
\midrule
MCAR & zero & $+0.0001$ & $+0.0015$ & 10/10 & $+0.0007$ & $+0.0017$ & 5/5 \\
MCAR & forward fill & $-0.0001$ & $+0.0006$ & 10/10 & $+0.0001$ & $+0.0006$ & 5/5 \\
MCAR & mean & $+0.0010$ & $+0.0017$ & 10/10 & $+0.0010$ & $+0.0017$ & 5/5 \\
\midrule
Block & zero & $-0.0009$ & $+0.0011$ & 8/8 & $+0.0111$ & $+0.0121$ & 2/4 \\
Block & forward fill & $-0.0004$ & $+0.0003$ & 8/8 & $+0.0006$ & $+0.0014$ & 4/4 \\
Block & mean & $+0.0011$ & $+0.0017$ & 8/8 & $+0.0035$ & $+0.0044$ & 4/4 \\
\bottomrule
\end{tabular}

\end{table}

\subsection{E6: Interpretability via shape recovery (C5)}
Table~\ref{tab:e6-recovery} and Figure~\ref{fig:e6-shapes} report shape recovery on regimes A ($g^*=\log$), B ($g^*=$ identity), and D ($g^*=$ hinge), with the activation of the learned transforms ablated (tanh, sigmoid, softplus, ReLU). The peak channel and the decay are recovered by every activation: $g_2$ correlations are $0.985$--$0.998$ (RMSE $0.021$--$0.115$ on the normalized grid) against the identity and hinge targets, and the learned decay lands within $0.006$ of $\alpha^*{=}0.90$ for all four activations ($0.899 \pm 0.003$ with tanh). Decay recovery is not specific to the reference point: retraining the tanh arm on regime A recalibrated at off-reference decays (the E2 recalibration, Appendix~\ref{app:protocols}) recovers $\hat{\alpha} = 0.499 \pm 0.008$ at $\alpha^* = 0.5$, $0.698 \pm 0.002$ at $0.7$, and $0.949 \pm 0.002$ at $0.95$---the mean recovered decay lies within $0.002$ of truth at every swept point (per-seed spread at most $\pm 0.008$). The recency transform separates the activations exactly as the expressivity analysis of Section~\ref{sec:architecture} predicts. Non-convex activations recover the concave $\log$: tanh reaches correlation $0.997$ with RMSE $0.036$, sigmoid $0.989$ with $0.070$. Convex activations cannot: ReLU and softplus are pinned at correlation $0.940$ and $0.939$ with RMSE $0.244$ and $0.247$---seven times the tanh error---because the best convex approximation to a concave transform is near-linear. The contrast also shows why both shape metrics are reported: correlation alone flatters the convex fits ($0.94$ reads as high), while RMSE on the common evaluation grid exposes them, and panel (a) of Figure~\ref{fig:e6-shapes} makes the failure visible. C5 is supported in asymmetric form (mechanism analyzed in Section~\ref{sec:discussion}): $g_2$ and $\alpha$ are identifiable regardless of parameterization, and $g_1$ is identifiable given a non-convex monotone activation.

\begin{figure}[tp]
\centering
\includegraphics[width=\linewidth]{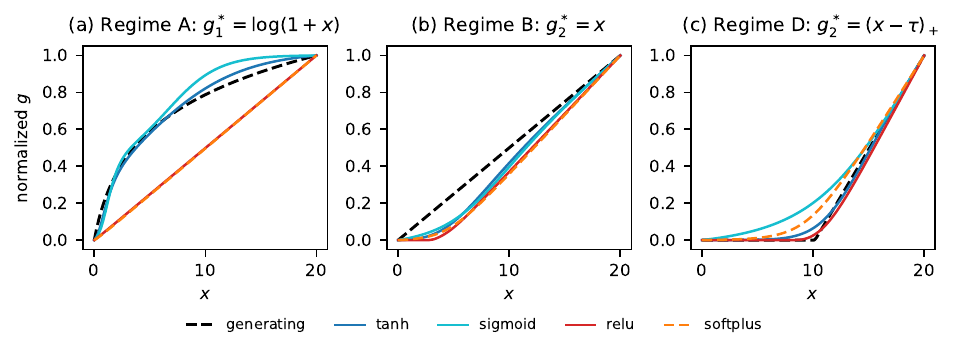}
\caption{E6 recovered transforms (seed-averaged) against the generating shapes; curves are min--max normalized to $[0,1]$ on the evaluation grid because the monotone head absorbs affine scale, so only the shape is identified. \emph{(a)} The concave $g_1^*=\log(1+x)$ of regime A is recovered by the non-convex activations (tanh, sigmoid) while the convex activations (softplus, ReLU) collapse onto a near-linear fit. \emph{(b, c)} The peak transforms of regimes B and D are recovered by all activations on the upper input range; in (b) the recovered curves flatten at low $x$ because the $\max$ channel constrains $g_2$ only where per-sequence maxima concentrate, leaving the low-$x$ region unidentified---the same affine-scale caveat quantified by the RMSE column of Table~\ref{tab:e6-recovery}.}
\label{fig:e6-shapes}
\end{figure}

\begin{table}[t]
\centering
\caption{E6 shape-recovery metrics (mean $\pm$ std over five seeds): correlation and RMSE between the learned and generating transforms on a common normalized grid, and the learned decay on regime A. All activations recover $g_2$ and $\alpha$; only the non-convex activations recover the concave $g_1$, and correlation alone would mask the failure that RMSE exposes.}
\label{tab:e6-recovery}
\footnotesize
\setlength{\tabcolsep}{3pt}
\begin{tabular}{lccccccc}
\toprule
 & \multicolumn{2}{c}{$g_1$, regime A ($\log$)} & \multicolumn{2}{c}{$g_2$, regime B (identity)} & \multicolumn{2}{c}{$g_2$, regime D (hinge)} & $\hat\alpha$, regime A \\
\cmidrule(lr){2-3}\cmidrule(lr){4-5}\cmidrule(lr){6-7}\cmidrule(lr){8-8}
Activation & corr.\ & RMSE & corr.\ & RMSE & corr.\ & RMSE & ($\alpha^*=0.90$) \\
\midrule
tanh & $0.997 \pm 0.003$ & $0.036 \pm 0.018$ & $0.993 \pm 0.003$ & $0.086 \pm 0.016$ & $0.998 \pm 0.000$ & $0.021 \pm 0.002$ & $0.899 \pm 0.003$ \\
sigmoid & $0.989 \pm 0.005$ & $0.070 \pm 0.017$ & $0.992 \pm 0.002$ & $0.089 \pm 0.007$ & $0.985 \pm 0.002$ & $0.088 \pm 0.006$ & $0.897 \pm 0.003$ \\
softplus & $0.939 \pm 0.000$ & $0.247 \pm 0.000$ & $0.989 \pm 0.002$ & $0.115 \pm 0.008$ & $0.993 \pm 0.001$ & $0.055 \pm 0.005$ & $0.894 \pm 0.003$ \\
relu & $0.940 \pm 0.000$ & $0.244 \pm 0.000$ & $0.990 \pm 0.002$ & $0.115 \pm 0.009$ & $0.998 \pm 0.003$ & $0.025 \pm 0.015$ & $0.894 \pm 0.003$ \\
\bottomrule
\end{tabular}

\end{table}

\subsection{E7: Insensitivity of the guarantee to admissible design (C6)}
Table~\ref{tab:e7-grid} reports the pooled violation counts and the AUC distribution across the 48-configuration ablation grid on regimes A--D; Table~\ref{tab:e7-marginals} reports the marginal AUC of each knob setting. Every one of the $192$ configuration--regime pairs records zero violations across its $49{,}250$ ordered test pairs ($9{,}456{,}000$ pooled over the grid), verified in double precision: monotonicity holds identically for every admissible choice of $g_1$, $g_2$, and $\alpha$. Accuracy stays within a quantified band with no catastrophic configuration: the spread of seed-averaged AUC across the $48$ configurations is at most $0.0351$ (regime A), comparable to the E2 headline gap, and on the peak-driven regimes B and D it is $0.0027$ and $0.0007$---below the within-configuration seed noise, so there the choice of admissible design matters less than the random seed.

The secondary observations land as anticipated. The shape of $g_2$ is immaterial even where the peak channel carries the label: its marginal AUCs span $0.0009$ on regime B and $0.0002$ on regime D. The grid's spread is instead concentrated in the recency channel on the recency-driven regimes, where $g_1=\log$ and learned $\alpha$ emerge as mild optima: on regime A the learned-$\alpha$ marginal exceeds the worst fixed setting ($\alpha=0.70$) by $0.024$ and $g_1=\log$ exceeds the identity by $0.008$---the same function-class mechanism that E6 and E9 isolate---with regime C showing the same ordering attenuated. Even the worst cell in the grid (identity $g_1$ with $\alpha=0.70$ on regime A, seed-averaged AUC $0.805$ against the grid best of $0.840$, and below the static aggregate's $0.812$) degrades gracefully rather than failing. C6 is supported in full: the guarantee and the bulk of the accuracy are properties of the admissible class, not of a tuned instantiation within it.

\begin{table}[t]
\centering
\caption{E7 ablation grid: pooled monotonicity violations and the distribution of seed-averaged AUC across the 48 admissible configurations, per regime. Every configuration--regime pair is violation-free. The spread column, read against the mean within-configuration seed std, shows that on the peak-driven regimes B and D the design choice matters less than the random seed.}
\label{tab:e7-grid}
\footnotesize
\setlength{\tabcolsep}{4pt}
\begin{tabular}{lrccccc}
\toprule
 & & \multicolumn{4}{c}{AUC over the 48 configurations} & \\
\cmidrule(lr){3-6}
Regime & Viol.\ / tests & min & median & max & spread & seed std \\
\midrule
A & $0$ / $2{,}364{,}000$ & $0.8052$ & $0.8316$ & $0.8404$ & $0.0351$ & $0.0042$ \\
B & $0$ / $2{,}364{,}000$ & $0.7513$ & $0.7538$ & $0.7540$ & $0.0027$ & $0.0087$ \\
C & $0$ / $2{,}364{,}000$ & $0.8381$ & $0.8503$ & $0.8537$ & $0.0156$ & $0.0043$ \\
D & $0$ / $2{,}364{,}000$ & $0.7137$ & $0.7141$ & $0.7144$ & $0.0007$ & $0.0109$ \\
\bottomrule
\end{tabular}

\end{table}

\begin{table}[t]
\centering
\caption{E7 marginal AUC per knob setting, averaged over the rest of the grid and five seeds. The shape of $g_2$ is immaterial even on the peak-driven regimes B and D; the only material knobs are the recency channel's $g_1$ and $\alpha$ on the recency-driven regimes A and C, where $\log$ and the learned decay emerge as mild optima.}
\label{tab:e7-marginals}
\footnotesize
\setlength{\tabcolsep}{4pt}
\begin{tabular}{llcccc}
\toprule
Knob & Setting & A & B & C & D \\
\midrule
$g_1$ & identity & $0.8231$ & $0.7536$ & $0.8471$ & $0.7141$ \\
 & $\log$ & $0.8312$ & $0.7534$ & $0.8495$ & $0.7141$ \\
 & monotone MLP & $0.8310$ & $0.7540$ & $0.8494$ & $0.7141$ \\
\midrule
$g_2$ & identity & $0.8288$ & $0.7539$ & $0.8491$ & $0.7141$ \\
 & hinge & $0.8288$ & $0.7539$ & $0.8491$ & $0.7141$ \\
 & power & $0.8273$ & $0.7530$ & $0.8477$ & $0.7140$ \\
 & monotone MLP & $0.8287$ & $0.7539$ & $0.8487$ & $0.7142$ \\
\midrule
$\alpha$ & $0.70$ & $0.8131$ & $0.7538$ & $0.8414$ & $0.7142$ \\
 & $0.85$ & $0.8340$ & $0.7537$ & $0.8505$ & $0.7141$ \\
 & $0.95$ & $0.8296$ & $0.7535$ & $0.8501$ & $0.7140$ \\
 & learned & $0.8370$ & $0.7536$ & $0.8526$ & $0.7141$ \\
\bottomrule
\end{tabular}

\end{table}

\subsection{E8: Basis-function sufficiency (C7)}
\label{sec:e8-results}
Across all $2{,}180$ basis-model runs---both arms, fifteen regimes, the Phase III ablations and the Phase IV transfer evaluations (including the held-out-regime rows of the repair rungs), five seeds---every admissible subset records zero monotonicity violations, so the guarantee is independent of basis configuration and the remainder of this section concerns sufficiency only.

Near-sufficiency on the core regimes exceeds what C7 requires. The Phase II subset search on the training regime C selects \emph{EWMA alone} as the smallest admissible set reaching $95\%$ of the oracle AUC in every seed (a permissive bar on a scale where $0.5$ is chance; on the quoted test values the subset also clears the stricter analogue on the Gini scale $2\,\mathrm{AUC}-1$, $0.834$; selection discipline in Appendix~\ref{app:protocols}) at $0.846 \pm 0.004$ against an oracle of $0.852 \pm 0.004$ that may use non-monotone statistics, a governance gap below $0.008$ AUC. Across all core regimes A--E, the EWMA+HWM pair (M3) matches the full six-basis model (M7) within $\pm 0.001$ AUC.

The adversarial regimes resolve into a two-level finding that the frozen arm makes precise (Table~\ref{tab:e8-necessity}). At the \emph{basis} level the engineered signals are real. Against the literal fixed EWMA+HWM channels, adding the smallest frozen subset containing the target basis is worth $+0.027$ AUC on the cumulative-burden regime F, $+0.064$ on the persistence regime G, $+0.051$ on the breach-count regime J, $+0.071$ on the peak-prominence regime N (HWM over EWMA alone), $+0.094$ on the early-warning regime K (the early-window EWMA), and $+0.143$ on the exposure-floor regime O---the largest basis-level gap in the study. For J, N, K, and O the subset adds exactly the target basis; the F and G ladders pass through the breach count, whose own increment is $+0.020$ (F) and $+0.036$ (G), so the \emph{basis-specific} increments are $+0.007$ for the cumulative sum and $+0.028$ for the run length---the latter matching the learned arm's open residual of $-0.027$ on G almost exactly. The single-basis gaps on J, N, K, and O sit an order of magnitude above seed noise (the cumulative sum's $+0.007$ is at seed-noise scale), corroborated by the debiased incremental information ($0.116 \pm 0.020$ nats for longest run beyond $\{$EWMA, HWM$\}$ on G, $0.122 \pm 0.027$ for breach count on J, $0.119 \pm 0.005$ for the LWM on O). At the \emph{channel} level, however, the learned arm closes the F, J, and N gaps without ever being given the target basis: the residual column of Table~\ref{tab:e8-necessity} shows learned EWMA+HWM tying the frozen-plus-target models on F and J, and learned EWMA alone tying frozen EWMA+HWM on N. This is the subsumption anticipated in Section~\ref{sec:models}: cumulative sums, breach counts, and steep peak hinges are reachable without explicit bases. The strongest case is regime N, engineered to make the HWM maximally prominent: the learned e-channel alone matches the explicit max channel (paired residual $-0.000$, Table~\ref{tab:e8-necessity}), consistent with the HWM lying in the rank closure of the learnable family. The explicit max channel still pays its way as a \emph{learnability dividend}---it supplies the $p \to \infty$ limit exactly through a parameter-free recurrence, where the bounded monotone MLP must approximate extreme convexity---worth $+0.0053 \pm 0.0009$ (B), $+0.0033 \pm 0.0007$ (C), and $+0.0054 \pm 0.0006$ (M) AUC for the learned arm, the paired per-seed delta positive in every seed on each regime.

Three boundary witnesses survive the learned arm's subsumption, and they characterize the sufficiency boundary structurally; each is then closed by its matching explicit basis.

\emph{Consecutivity (regime G).} The longest-run statistic is not time-separable: it depends on the \emph{adjacency pattern} of threshold exceedances, not on any per-step transform, and it lies outside even the separable family's rank closure (Proposition~\ref{prop:rank-closure}, part (vi)). The learned channel closes the ladder's breach-count share but almost none of the run-specific increment: $0.027$ of its $+0.028$ stays open, with learned M3 at $0.806$ against $0.833 \pm 0.005$ for the explicit run-length rung. That rung sits at the regime's data-level ceiling (fresh-pool Bayes AUC $0.831$, protocol in Appendix~\ref{app:protocols}; paired same-split delta $-0.0005 \pm 0.0006$, Table~\ref{tab:e8-necessity}), so the one-basis repair returns G to the best any model consuming the same sequences can achieve.

\emph{Time localization (regime K).} The early-window label requires a temporal weight profile that vanishes after $t_0$, which no geometric decay realizes---the statistic lies outside the rank closure (Proposition~\ref{prop:rank-closure}, part (v))---and no \emph{separable accumulator} in the pool supplies: without the matching basis, every model in the basis pool sits roughly $0.07$ AUC below the data-level ceiling (best $0.689 \pm 0.007$ against a fresh-pool Bayes AUC of $0.764$), and the full frozen pool of separable bases improves on the EWMA+HWM pair by only $+0.022$. The engineered grids share the outage (the grid GBMs also sit at $0.689$; Table~\ref{tab:e10-boundary}), while a constrained tree over raw lags---no aggregation committed---reaches the ceiling, confirming that the boundary belongs to separable aggregation rather than to the data. The repair is one explicit admissible basis: adding the early-window EWMA---the prefix-restricted recurrence $e_{t_0} = \sum_{t \le t_0} \alpha^{t_0-t} x_t$, monotone by the same lemma as the EWMA itself---lifts the frozen arm to $0.761 \pm 0.008$, at the Bayes ceiling (paired same-split delta $-0.0003 \pm 0.0003$; the nominal gap to the fresh-pool $0.764$ is split noise), a $+0.072$ recovery over the best repair-free pool (Table~\ref{tab:e8-necessity}); the learned arm over the same bases reaches the identical value.

\emph{The exposure floor (regime O).} The third witness is \emph{partial}, and it separates basis-level necessity from ranking recoverability. The running minimum is not a member of the separable family, but unlike the run and the early window it does lie in the family's \emph{rank closure}: saturating reciprocal-power transforms $g(x) = c^{-p} - (x+c)^{-p}$ are admissible, and ranking by their discounted sum converges to ranking by the floor as $p \to \infty$---the exact mirror of the power limit that reaches the maximum (Proposition~\ref{prop:rank-closure}, part (iv)). The boundary at the floor is therefore one of \emph{learnability} rather than representability, and the measurements bear that out: the learned channel recovers most of the O gap (learned M3 reaches $0.774$ against $0.783 \pm 0.008$ with the explicit LWM, itself exactly at O's data-level ceiling, the fresh-pool Bayes AUC $0.783$---paired rung-minus-Bayes delta $-0.0010 \pm 0.0044$; paired residual $-0.0091 \pm 0.0019$, negative in every seed---small but consistent), because a bounded monotone MLP realizes a finite-steepness proxy of that limit: the minimum exceeds $c$ exactly when no step falls below $c$, so a steep saturating $g_1$ turns the discounted sum into a count of low steps. The proxy is also an optimization accident of the two-channel configuration, not a stable property of the learned arm: when the separable bases are added to the head (learned M6 and M7), the trained e-channel abandons the steep saturating solution and collapses to the frozen-arm level ($0.647$ for learned M6 and M7---near the frozen base's $0.640$, against learned M3's $0.774$---in all five seeds), so the recovery does not survive richer configurations. The explicit LWM channel supplies the exact limit through a parameter-free recurrence, closing the last of the largest basis-level gap in the ladder---the low-water mirror of the HWM's learnability dividend.

The remaining adversaries behave as robustness regimes rather than boundary witnesses: regime L's cap is an outer non-decreasing transform of $e_T$ and cannot create a ranking gap (the ties it introduces are label-exchangeable, leaving expected AUC unchanged), and regimes H and M are well proxied by their separable factors in the learned arm (gaps ${\le}0.001$; at the frozen basis level H retains a $+0.011$ interaction gap that no additive subset closes).

Cross-regime transfer is flat where the design originally predicted an ordering (Table~\ref{tab:e8-transfer}). Training on regime C and evaluating on the held-out adversarial regimes (averaged over the nine evaluated for every subset), the full subset M7 is nominally ahead of EWMA-alone in every seed but only by $+0.0003$ to $+0.0024$ AUC (mean $+0.001$)---at most a lean under the Section~\ref{sec:eval-protocol} convention, an order of magnitude below the necessity gaps---while HWM-alone transfers clearly worst ($0.755 \pm 0.007$). The frozen multi-basis models sit at or marginally above the learned models (differences ${\sim}0.001$), so extra fixed bases neither help nor hurt transfer materially.

C7 is therefore supported in a sharpened form. The two-state summary is near-sufficient on core mechanisms not because the other admissible bases are uninformative but because the learned recurrence spans the separable ones: the e-channel is a learned separable accumulator, and the m-channel is that family's hardest-to-learn limit point kept explicit. The demonstrated boundary is consecutivity (G), time localization (K), and---at the basis level---the exposure floor (O), exactly the statistics that fall outside the geometrically weighted separable family; their structural-versus-learnability placements are proved in Proposition~\ref{prop:rank-closure}, and each is restored by its matching explicit basis at the regime's Bayes ceiling (Table~\ref{tab:e8-necessity}). Experiment E11 (Section~\ref{sec:e14}) confirms this boundary functionally, by regression rather than prediction.

\begin{table}[t]
\centering
\caption{E8 necessity ladder on the boundary-probing regimes (AUC, mean $\pm$ std over five seeds). The frozen arm fixes the EWMA/HWM channels to their literal fixed statistics, so adding the engineered target basis measures basis-level necessity (Gap, paired per seed); the base subset is frozen EWMA+HWM (M3f) for F, G, J, O, and K, and frozen EWMA alone (M1f) for N, and the augmented subset is the smallest frozen subset containing the target: F$\to$M6f, G$\to$M5f, J$\to$M4f, N$\to$M3f, O$\to$M8f, and K$\to$M9f (the early-window EWMA rung; the basis is available to the ablation arms but kept out of the Phase I/II pool). Residual is the learned arm over the base subset minus the frozen-plus-target rung: ${\approx}0$ means the learnable $(g_1, \alpha)$ channel recovers the basis without being given it (F, J, N), negative means the gap stays open in the learned arm without the explicit basis (G; K, where the learned channel cannot realize the vanishing weight profile; and marginally O). The F and G Gap columns include the breach count's pass-through increment, and each explicit rung sits at its regime's data-level ceiling; the decomposition and the paired rung-minus-Bayes deltas are in the text. The derived Gap and Residual columns are computed from unrounded per-seed means, so they can differ from the difference of the rounded operands by one unit in the final digit.}
\label{tab:e8-necessity}
\footnotesize
\setlength{\tabcolsep}{4pt}
\begin{tabular}{llccccc}
\toprule
Regime & Target basis & Frozen base & Frozen $+$ target & Gap & Learned base & Residual \\
\midrule
F (cumulative burden) & cumulative sum & $0.805 \pm 0.007$ & $0.832 \pm 0.008$ & $+0.027$ & $0.838 \pm 0.008$ & $+0.007$ \\
G (persistence) & longest run & $0.769 \pm 0.003$ & $0.833 \pm 0.005$ & $+0.064$ & $0.806 \pm 0.005$ & $-0.027$ \\
J (breach count) & breach count & $0.799 \pm 0.005$ & $0.849 \pm 0.006$ & $+0.051$ & $0.849 \pm 0.006$ & $+0.000$ \\
N (peak prominence) & HWM & $0.787 \pm 0.007$ & $0.858 \pm 0.009$ & $+0.071$ & $0.858 \pm 0.009$ & $-0.000$ \\
O (exposure floor) & LWM & $0.640 \pm 0.009$ & $0.783 \pm 0.008$ & $+0.143$ & $0.774 \pm 0.009$ & $-0.009$ \\
K (early warning) & early-window EWMA & $0.667 \pm 0.007$ & $0.761 \pm 0.008$ & $+0.094$ & $0.689 \pm 0.007$ & $-0.072$ \\
\bottomrule
\end{tabular}

\end{table}

\begin{table}[t]
\centering
\caption{E8 Phase IV cross-regime generalization: AUC after training on regime C, averaged over the held-out adversarial regimes evaluated for every subset (nine; mean $\pm$ std over five seeds). The full subset holds at most a lean over EWMA-alone (per-seed paired difference $+0.0003$ to $+0.0024$, positive in every seed), HWM-alone transfers worst, and the frozen arm matches the learned arm throughout; the floor subset M8 transfers identically to M3 and the early-window subset M9 within $0.002$ of it, as expected when the training label carries no floor or early-window signal.}
\label{tab:e8-transfer}
\footnotesize
\setlength{\tabcolsep}{4pt}
\begin{tabular}{lcc}
\toprule
Basis subset & Learned arm & Frozen arm \\
\midrule
M1 (EWMA) & $0.777 \pm 0.007$ & $0.769 \pm 0.006$ \\
M2 (HWM) & $0.755 \pm 0.007$ & $0.755 \pm 0.007$ \\
M3 (EWMA+HWM) & $0.776 \pm 0.007$ & $0.776 \pm 0.007$ \\
M4 (M3 + breach count) & $0.777 \pm 0.007$ & $0.779 \pm 0.007$ \\
M5 (M4 + longest run) & $0.777 \pm 0.007$ & $0.779 \pm 0.007$ \\
M6 (M3 + cumulative + breach count) & $0.779 \pm 0.006$ & $0.779 \pm 0.007$ \\
M7 (M3 + cum.\ + breach + run + excess) & $0.779 \pm 0.006$ & $0.779 \pm 0.007$ \\
M8 (M3 + LWM) & $0.776 \pm 0.007$ & $0.776 \pm 0.007$ \\
M9 (M3 + early-window EWMA) & $0.777 \pm 0.007$ & $0.778 \pm 0.007$ \\
\bottomrule
\end{tabular}

\end{table}

\subsection{E9: Capacity robustness (C1 and C3 across sizes)}
We retrain the core model set at four capacities on regimes A and C; the unconstrained LSTM spans $361$ to $199{,}297$ parameters across the tiers, the monotone exposure RNN $4$ to $6{,}788$. Table~\ref{tab:e9-capacity} reports parameter counts, AUC, violation totals, and the per-cell non-inferiority verdict against the prespecified LSTM comparator; in every cell the LSTM is also the per-cell best unconstrained baseline.

The guarantee is capacity-independent. All three monotone-by-construction models record zero violations at every size on both regimes (40 evaluation cells per model), while the unconstrained baselines accumulate $4{,}502$ (LSTM), $7{,}135$ (RNN), and $7{,}832$ (aggregate) violations pooled over the same sweep (per-tier breakdown in Table~\ref{tab:baseline-violations}; the static aggregates have no capacity knob, so their tier rows are the same fitted model evaluated once per tier and the pooled count restates one model four times). C1 therefore holds identically for a four-parameter instantiation and a $6{,}788$-parameter one. Parity survives scaling up the baselines: seven of the eight (size, regime) cells pass the prespecified test, and from the small tier upward the paired $\Delta\mathrm{AUC}$ never leaves $\pm 0.001$ even as the LSTM grows to $199{,}297$ parameters against the monotone model's $6{,}788$---which also justifies the lean medium tier used as the default in E1--E8.

The single failing cell---tiny on regime A ($\Delta\mathrm{AUC}=-0.0093$, CI $[-0.0105,-0.0078]$)---is an expressiveness floor, not a training artifact. The tiny tier fixes both transforms to the identity and the head to a monotone linear map (four parameters), while regime A's generating recency transform is the concave $g_1^*=\log(1+x)$, which the identity channel cannot represent---the same function-class mechanism that E6 isolates for convex activations, here in its sharpest form. The training-adequacy audit flags all five tiny regime-A monotone runs as cap-binding, so we tested the undertraining explanation directly: retraining the cell (a single-seed diagnostic rerun at the first seed) with early stopping disabled out to a 300-epoch ceiling leaves test AUC unchanged between epochs 100 and 300 ($0.8345$ at both---the small offset from the five-seed mean of $0.8309 \pm 0.0042$ in Table~\ref{tab:e9-capacity} is the single-seed difference), moves validation loss by less than $10^{-4}$ past epoch 100, and the production early-stopping rule would in fact have fired at epoch 107. The deficit survives an unlimited training budget. Parity is restored as soon as the transforms are learnable: the small tier passes with $84$ parameters against a $1{,}233$-parameter LSTM. The miss thus sharpens rather than weakens C3: the constraint costs nothing wherever the admissible class can express the mechanism, and the capacity threshold for expressing it is small.

\begin{table}[t]
\centering
\caption{E9 capacity sweep (mean $\pm$ std over five seeds): parameter count, AUC, and pooled violation totals for the monotone exposure RNN and the unconstrained-LSTM comparator, with the paired $\Delta$AUC (monotone minus LSTM), its two-level bootstrap 95\% confidence interval, and the non-inferiority verdict (pass iff the lower bound is at least $-0.01$). The only failure is the four-parameter tiny tier on the concave regime A, an expressiveness floor that a cap-raising diagnostic shows is not an artifact of the training budget.}
\label{tab:e9-capacity}
\footnotesize
\setlength{\tabcolsep}{2pt}
\begin{tabular}{llrcrrcrccc}
\toprule
 & & \multicolumn{3}{c}{Monotone exposure RNN} & \multicolumn{3}{c}{Unconstrained LSTM} & \multicolumn{3}{c}{paired non-inferiority} \\
\cmidrule(lr){3-5}\cmidrule(lr){6-8}\cmidrule(lr){9-11}
Size & Regime & params & AUC & viol.\ & params & AUC & viol.\ & $\Delta$AUC & CI$_{95}$ & verdict \\
\midrule
tiny & A & 4 & $0.8309 \pm 0.0042$ & 0 & 361 & $0.8402 \pm 0.0042$ & 170 & $-0.0093$ & $[-0.0105, -0.0078]$ & \textbf{fail} \\
tiny & C & 4 & $0.8513 \pm 0.0045$ & 0 & 361 & $0.8525 \pm 0.0044$ & 1182 & $-0.0012$ & $[-0.0020, -0.0004]$ & pass \\
small & A & 84 & $0.8403 \pm 0.0044$ & 0 & 1{,}233 & $0.8401 \pm 0.0044$ & 53 & $+0.0002$ & $[-0.0002, +0.0005]$ & pass \\
small & C & 84 & $0.8536 \pm 0.0045$ & 0 & 1{,}233 & $0.8528 \pm 0.0047$ & 1115 & $+0.0008$ & $[+0.0001, +0.0014]$ & pass \\
medium & A & 548 & $0.8397 \pm 0.0043$ & 0 & 12{,}961 & $0.8402 \pm 0.0041$ & 6 & $-0.0004$ & $[-0.0009, +0.0000]$ & pass \\
medium & C & 548 & $0.8535 \pm 0.0045$ & 0 & 12{,}961 & $0.8531 \pm 0.0048$ & 1182 & $+0.0003$ & $[-0.0002, +0.0009]$ & pass \\
large & A & 6{,}788 & $0.8396 \pm 0.0043$ & 0 & 199{,}297 & $0.8402 \pm 0.0041$ & 31 & $-0.0005$ & $[-0.0011, +0.0001]$ & pass \\
large & C & 6{,}788 & $0.8534 \pm 0.0045$ & 0 & 199{,}297 & $0.8528 \pm 0.0049$ & 763 & $+0.0006$ & $[+0.0000, +0.0012]$ & pass \\
\bottomrule
\end{tabular}

\end{table}

\subsection{E10: Learned vs.\ hand-crafted monotone aggregation (C8)}
\label{sec:e10-results}
Every comparator in the grid-baseline set is monotone end-to-end, and all $1{,}980$ E10 runs---every model, phase, severity, and seed---record zero violations on every applicable operator, so the guarantee is shared by design and the remainder of this section concerns what learning the aggregation is worth.

On the parity ladder the learned channel (the two-state exposure RNN) is never statistically worse and is the strict winner where recency structure matters. Table~\ref{tab:e10-parity} reports per-regime AUC with the parsimony account; Table~\ref{tab:e10-delta} reports paired two-level bootstrap deltas against the three constrained-LightGBM comparators on the shared test splits. The exposure RNN posts the best mean AUC on regimes A and C, where the recency channel carries signal, and on the peak regime B, with paired deltas whose confidence intervals exclude zero against all three GBMs on all three regimes and whose per-seed sign is positive for every seed (from $+0.0010$ $[+0.0004, +0.0017]$ against the practice GBM on B up to $+0.0065$ $[+0.0040, +0.0092]$ against the raw-lags GBM on the same regime); under the evaluation protocol's sub-$0.002$ rule the grid-GBM margins on B ($+0.0010/{+}0.0011$) and C ($+0.0017/{+}0.0018$) are leans rather than substantive advantages, and so is the regime-A margin against the raw-lags GBM ($+0.0013$), so the strict wins are regime A against the two grid GBMs ($+0.0029/{+}0.0037$) and regimes B and C against the raw-lags GBM; on the threshold and saturating regimes D and E the deltas against the grid GBMs are statistical ties, the expected outcome near a grid point. The strict wins are not a tuning artifact of the disclosed GBM defaults: retraining both grid GBMs over an $18$-point regularization grid (leaves $\{15, 31, 63\}$ $\times$ learning rate $\{0.02, 0.05, 0.1\}$ $\times$ min-data-in-leaf $\{20, 100\}$), selecting per seed by validation AUC, moves the paired deltas by at most $0.0001$ on regimes A and C ($+0.0028$ $[+0.0019, +0.0038]$ and $+0.0037$ $[+0.0029, +0.0045]$ tuned on A), so the regime-A wins and the C leans survive per-model tuning unchanged. The head class of the grid consumer is immaterial at parity: the calibrated lattices track the monotone-MLP heads within $\pm 0.002$ AUC on every regime, so none of the parity conclusions is an artifact of one monotone head family. The parsimony account favors the learned channel against every sequence-aware comparator's \emph{feature pipeline}: $548$ parameters and no engineered features against $2{,}159$--$4{,}659$ tree leaves and $2{,}312$--$4{,}384$ lattice parameters over $11$--$24$ fixed input columns (hand-crafted for the grids, raw lags for the third GBM; the grid-MLP heads themselves are smaller, at $337$--$353$ parameters, but still consume the engineered columns).

\emph{Boundary regimes.} The parity ladder runs on regimes A--E, where the generating mechanism lies inside the learned channel's separable family; that is the scope of every parity statement in this paper, and Table~\ref{tab:e10-boundary} shows why the scoping matters, rerunning every comparator class on the E8 boundary regimes (same splits and seeds; the generic-recurrence arms at the default tier). The grid GBMs inherit the separable-aggregation boundary essentially unchanged---their windowed means, maxima, and breach counts carry no adjacency, early-window, or floor information---while the raw-lags GBM, which commits to no aggregation, crosses it on all three regimes: a constrained tree over raw lags can split on individual time indices, so adjacency, time localization, and the floor are all within its reach, and it sits at the K ceiling, at seed scale from the O rung, and roughly $0.010$ short of the run rung on G. Monotone recurrence alone does not confer that reach. The fixed-decay accumulator bank---itself separable, so this arm tests the constraint mechanism rather than recurrence---stays below the repair-free two-state pool in every seed at both activations (its tanh gain on the floor regime is the Section~\ref{sec:architecture} curvature sensitivity: a saturating-count proxy of the floor). The weight-constrained Elman recurrence---the genuinely non-separable form, and the industrial patent's construction---is the informative arm: its interleaved saturation can encode local adjacency, and on the consecutivity regime it recovers a real share of the run signal ($0.796 \pm 0.005$, above both bank arms and just below the learned pool's $0.806$ with overlapping seed ranges), but on time localization it stays far below the repair-free pool, as a time-invariant recurrence must, and on the floor it does not find the proxy the tanh bank finds. No monotone recurrence tested crosses the boundary on any regime: among the monotone models, only the aggregation-free tree does. The boundary of E8 and E11 is therefore a property of separable aggregation---learned or engineered---not of monotone learning as such; the raw-lags escape is priced below.

\begin{table}[t]
\centering
\caption{E10 boundary regimes: test AUC (mean $\pm$ std over five seeds, shared splits) for every comparator class on the three E8 boundary regimes. The engineered comparators are the constrained-LightGBM arms of Section~\ref{sec:models}; the generic monotone recurrences are the fixed-decay accumulator banks (ReLU and tanh components) and the weight-constrained Elman form, at the default tier; the two-state rows restate, on the same splits, the E8 ladder's learned base and its frozen repair rung---run length for G, early-window EWMA for K, LWM for O (Table~\ref{tab:e8-necessity}). Each repair rung sits at the per-seed test-split Bayes ceiling, and among the monotone models only the aggregation-free raw-lags tree crosses the separable-aggregation boundary.}
\label{tab:e10-boundary}
\footnotesize
\setlength{\tabcolsep}{5pt}
\begin{tabular}{lccc}
\toprule
 & G (persistence) & K (early warning) & O (exposure floor) \\
\midrule
\multicolumn{4}{@{}l}{\emph{Engineered comparators}} \\
Practice grid (constrained GBM) & $0.800 \pm 0.005$ & $0.689 \pm 0.006$ & $0.682 \pm 0.010$ \\
Steelman grid (constrained GBM) & $0.800 \pm 0.005$ & $0.689 \pm 0.007$ & $0.666 \pm 0.010$ \\
Raw lags (constrained GBM) & $0.823 \pm 0.005$ & $0.760 \pm 0.008$ & $0.779 \pm 0.008$ \\
\midrule
\multicolumn{4}{@{}l}{\emph{Generic monotone recurrences}} \\
Fixed-decay bank (ReLU) & $0.764 \pm 0.003$ & $0.617 \pm 0.006$ & $0.640 \pm 0.009$ \\
Fixed-decay bank (tanh) & $0.769 \pm 0.004$ & $0.651 \pm 0.006$ & $0.740 \pm 0.009$ \\
Weight-constrained Elman (tanh) & $0.796 \pm 0.005$ & $0.656 \pm 0.019$ & $0.637 \pm 0.023$ \\
\midrule
\multicolumn{4}{@{}l}{\emph{Two-state architecture (E8)}} \\
Learned EWMA$+$HWM (M3) & $0.806 \pm 0.005$ & $0.689 \pm 0.007$ & $0.774 \pm 0.009$ \\
Explicit repair rung (run / early / LWM) & $0.833 \pm 0.005$ & $0.761 \pm 0.008$ & $0.783 \pm 0.008$ \\
\midrule
Bayes ceiling (per-seed test split) & $0.833 \pm 0.004$ & $0.761 \pm 0.008$ & $0.784 \pm 0.007$ \\
\bottomrule
\end{tabular}

\end{table}

The misspecification sweeps (Figure~\ref{fig:e10-misspec}) resolve the decay and threshold dimensions oppositely. On the decay sweep, each grid pays where its own coverage runs out. The steelman grid pays exactly the predicted cost: its paired gap to the learned channel reaches $0.019$ AUC at $\alpha^* = 0.1$, a distance of $0.4$ from its nearest fixed decay, and shrinks to $0.003$--$0.006$ at on-grid points. The cost is a property of the frozen grid, not of its consumer: at the off-grid extreme the steelman grid pays $+0.019$--$0.020$ identically under the monotone-MLP head, the calibrated lattice, and constrained LightGBM, so upgrading the head class buys back none of the price of hand-picking the decays. The practice grid, by contrast, pays essentially nothing at the same point ($+0.0004$): its last-value and three-step-window features absorb the fast-decay extreme that the steelman's decay set misses. Its uncovered region is the opposite end of the sweep: the windowed means lose ground monotonically as $\alpha^*$ grows, reaching $0.006$ at $\alpha^* = 0.99$, where no fixed window matches a near-flat discount.

The headline $0.019$ is therefore the price of committing to the \emph{wrong} grid, not of grid engineering as such: against the best of the seven grid and aggregate consumers at each swept point, the learned channel's paired advantage runs $+0.0001$ to $+0.0037$, growing monotonically toward slow decay. The standard engineered response to decay uncertainty is to decline to commit, and that response realizes the per-point envelope with no oracle: a constrained GBM over the deduplicated union of the two grids (nineteen columns---the pooled library), run through the same sweep on the same splits, holds the learned channel's paired advantage to $+0.0006 \pm 0.0005$ at $\alpha^* = 0.1$, rising to $+0.0037 \pm 0.0004$ at $\alpha^* = 0.99$, positive in every seed from $\alpha^* = 0.8$ upward. Pooling thus removes the committed-grid cost, and the learned channel's residual advantage over the strongest committed engineered consumer is the ${\le}0.004$ slow-decay band, not the $0.019$ headline. One comparator escapes the decay axis without any oracle: the raw-lags GBM, which commits to no aggregation at all, tracks the learned channel within $0.002$ across the entire sweep ($+0.0006$ at $\alpha^* = 0.1$, never beyond $+0.0018$). The escape is priced in the threshold and sample-size sweeps below, and it never wins a parity regime, so the trade the sweeps expose is symmetric: fixed aggregation pays where its coverage runs out, and refusing aggregation pays in efficiency wherever structure exists.

Two qualifications bound the finding: gaps of ${\sim}0.005$ against the committed grids persist at the slow-decay \emph{on-grid} points ($\alpha^* \in \{0.85, 0.95\}$), so the cost tracks the decay regime and not grid distance alone; and the threshold sweep is a null for the grid and aggregate consumers and the learned channel---the eight of them stay within ${\sim}\pm 0.002$ of one another across the entire $\tau^*$ range, including beyond every grid threshold, because the grids' windowed means, maxima, and last-value features carry the threshold signal without a matching breach count. The raw-lags GBM is the exception in the opposite direction: it trails the learned channel by $+0.0068$ at $\tau^* = 1.0$ ($+0.0044$ to $+0.0068$ for $\tau^* \le 7.5$, positive in every seed)---the estimation price of refusing aggregation, whose reverse side the boundary paragraph above already prices. The misspecification clause of C8 is therefore supported in the decay dimension---as the cost of a committed grid---and bounded in the threshold dimension for the grid consumers.

The sample-size sweep (Figure~\ref{fig:e10-auc-n}) compares all models at convergence: every torch cell trains under a raised epoch ceiling with early stopping deciding termination, and per-cell training-adequacy fields are recorded in the results (11 of 300 torch cells remain cap-binding, all at $n = 1000$ and none the exposure RNN). The protocol detail is load-bearing: under a budget-equalized 100-epoch cap every torch cell was cap-binding---the constrained models' slow convergence interacting with truncated optimization, the same confound C3 controls for---and the small-$n$ verdict inverted. At convergence the learned channel is the best model at every training size on regime A, including $n = 1000$ ($0.8310$ against the practice GBM's $0.8289$), with a paired per-seed delta of $+0.0064 \pm 0.0012$ against the practice GBM at $n = 2000$; on regime C it is best at every size except $n = 1000$, where the steelman grid with monotone head edges it by a within-noise $0.0012 \pm 0.0028$. The high-dimensional raw-lags comparator degrades exactly as predicted, its paired deficit growing from $+0.0041 \pm 0.0008$ at $n = 50{,}000$ to $+0.0184 \pm 0.0042$ at $n = 1000$ on regime C: twenty-four constrained columns cannot be estimated from a thousand sequences, while the compact learned channel and the low-dimensional grids both can. The nineteen-column union grid sits with the low-dimensional grids, not with raw lags: at $n = 1000$ its paired deficit to the learned channel is $+0.0019 \pm 0.0080$ on regime A and $+0.0026 \pm 0.0026$ on regime C---the same statistical tie as the practice grid---so pooling the two libraries does not, at this scale, trigger the feature-proliferation penalty the raw-lags comparator pays, and the small-sample axis remains an advantage over the high-dimensional raw-lags comparator specifically. The low-prevalence sensitivity confirms the contrast measures more than label scarcity: with regime C recalibrated to $5\%$ prevalence (Section~\ref{sec:experiments}), the orderings replicate---at full $n$ the learned channel keeps its lean over the practice GBM ($+0.0009 \pm 0.0004$, positive in every seed) and its margin over raw lags ($+0.0050$), and at $n = 1000$, now roughly fifty positive examples, the raw-lags deficit reappears at $+0.0187 \pm 0.0084$ (positive in every seed, beside the $+0.0184$ at the default prevalence) while the learned channel and the practice grid are a within-noise, sign-mixed tie ($-0.0026 \pm 0.0083$). Class rarity neither sharpens nor reverses the verdict: at both prevalences the small-sample advantage remains specific to the high-dimensional comparator.

Stress extrapolation (Figure~\ref{fig:e10-auc-shift}; Table~\ref{tab:e10-saturation}) exhibits the structural failure C8 predicts for tree ensembles, and one finding the design did not anticipate. Beyond the training support the grid GBMs' scores go exactly flat: at the $8\times$ scale shift their tie rates reach $0.59$ and $0.68$ with $24\%$ and $29\%$ of test scores saturated at the ensemble maximum, and their AUC collapses to $0.790$ and $0.784$; the learned channel's strictly increasing head never ties more than $0.1\%$ of scores and holds $0.815$. The raw-lags GBM escapes measurable degeneracy (its $24$ columns keep leaves distinct) and edges the learned channel at the most extreme severities ($0.821$ at $8\times$), at the price of the small-$n$ and parity deficits above. The unanticipated finding is that the best extrapolators are the fixed grids feeding \emph{neural monotone heads}: the monotone MLPs reach $0.848$ at $8\times$ with the calibrated lattices close behind ($0.845$ steelman, $0.832$ practice); the lattices stay degeneracy-free (tie rates ${\le}0.001$), while the grid-MLP heads tie up to a quarter of scores at $\times 8$ as their tanh heads saturate---without losing rank quality. The grids' aggregates are linear in the input scale, whereas the learned transforms' bounded activations compress inputs beyond the training range; that the advantage survives a change of head class localizes it in the linear-scale featurization, not in any one monotone head. These comparators are themselves instances of this paper's composition framework rather than deployed practice, so the result argues for the framework while marking a real boundary of the learned instantiation under extreme covariate shift.

C8 is supported as a division of labor. At full sample size and near grid points the learned channel is non-inferior on every core regime, strictly better on regime A against the grid GBMs (a lean against raw lags), and ahead with sub-$0.002$ leans on the peak and mixed regimes against the grid GBMs (clear wins there only against the raw-lags comparator). Its advantage appears where hand-crafted aggregation cannot adapt (off-grid decay, small samples against high-dimensional comparators) and where tree ensembles structurally fail (score degeneracy beyond support). Its two boundaries---threshold misspecification that the grids' non-breach features absorb, and extreme-severity extrapolation where fixed linear-scale features feeding a monotone head win---are properties of the evaluation worth stating, not deficits of the guarantee. One discipline governs the three advantage axes: each holds against a different incumbent. Raw lags and the pooled union grid escape the decay axis (within $0.002$ and $0.004$ of the learned channel, respectively); the small-sample gap is large only against the high-dimensional raw-lags comparator; the stress axis is lost to the grid-fed neural heads. No single engineered incumbent is dominated on all three axes, and no axis is won against every comparator: the case for the learned channel is adaptivity across failure modes an analyst cannot predict ex ante, not uniform dominance.

\begin{table}[t]
\centering
\caption{E10 parity ladder (AUC, mean $\pm$ std over five seeds) on regimes A--E with the parsimony account: trainable parameter count (for the GBMs, total leaf values of the fitted ensemble), engineered-feature count, and pooled monotonicity violations. Bold marks the best mean AUC per regime, with ties at the printed precision bolded jointly. The learned channel posts the best mean AUC on regimes A and C, where the recency channel carries signal, and on the peak regime B, with no engineered features and the smallest parameter-plus-feature pipeline of any sequence-aware comparator (the grid-MLP heads themselves are smaller, at $337$--$353$ parameters, but consume $11$--$12$ engineered columns); D and E are ties near the grid point, the expected outcome.}
\label{tab:e10-parity}
\scriptsize
\setlength{\tabcolsep}{2pt}
\begin{tabular}{lcccccrrc}
\toprule
Model & A & B & C & D & E & params & feat.\ & viol.\ \\
\midrule
Mono.\ exp.\ RNN & $\mathbf{0.8397 \pm 0.0043}$ & $\mathbf{0.7540 \pm 0.0087}$ & $\mathbf{0.8535 \pm 0.0045}$ & $\mathbf{0.7142 \pm 0.0109}$ & $0.8757 \pm 0.0033$ & 548 & -- & 0 \\
Mono.\ agg. & $0.8120 \pm 0.0056$ & $0.7538 \pm 0.0089$ & $0.8432 \pm 0.0050$ & $0.7129 \pm 0.0089$ & $0.8709 \pm 0.0045$ & 241 & 5 & 0 \\
Practice grid & $0.8345 \pm 0.0040$ & $0.7538 \pm 0.0088$ & $0.8510 \pm 0.0045$ & $0.7141 \pm 0.0109$ & $0.8758 \pm 0.0037$ & 353 & 12 & 0 \\
Steelman grid & $0.8339 \pm 0.0041$ & $0.7537 \pm 0.0088$ & $0.8515 \pm 0.0044$ & $0.7140 \pm 0.0109$ & $0.8760 \pm 0.0035$ & 337 & 11 & 0 \\
Practice lattice & $0.8360 \pm 0.0040$ & $0.7537 \pm 0.0086$ & $0.8519 \pm 0.0046$ & $0.7140 \pm 0.0110$ & $0.8752 \pm 0.0036$ & 4{,}384 & 12 & 0 \\
Steelman lattice & $0.8352 \pm 0.0046$ & $0.7536 \pm 0.0090$ & $0.8513 \pm 0.0046$ & $\mathbf{0.7142 \pm 0.0112}$ & $\mathbf{0.8761 \pm 0.0037}$ & 2{,}312 & 11 & 0 \\
Practice GBM & $0.8369 \pm 0.0037$ & $0.7528 \pm 0.0088$ & $0.8518 \pm 0.0046$ & $\mathbf{0.7142 \pm 0.0100}$ & $0.8754 \pm 0.0035$ & 2{,}299 & 12 & 0 \\
Steelman GBM & $0.8361 \pm 0.0044$ & $0.7530 \pm 0.0085$ & $0.8517 \pm 0.0045$ & $0.7141 \pm 0.0103$ & $0.8760 \pm 0.0037$ & 2{,}159 & 11 & 0 \\
Raw-lags GBM & $0.8384 \pm 0.0047$ & $0.7474 \pm 0.0081$ & $0.8495 \pm 0.0038$ & $0.7121 \pm 0.0088$ & $0.8721 \pm 0.0037$ & 4{,}659 & 24 & 0 \\
\bottomrule
\end{tabular}

\end{table}

\begin{table}[t]
\centering
\caption{E10 paired AUC deltas (exposure RNN minus comparator) against the three constrained-LightGBM incumbents on the shared per-seed test splits, with two-level bootstrap $95\%$ confidence intervals. The learned channel is ahead on regimes A, B, and C against all three comparators (intervals exclude zero and the per-seed delta is positive for every seed; the regime-B margins, the regime-C grid-GBM margins, and the regime-A margin against the raw-lags GBM are small enough to read as leans) and statistically tied on D and E against the grid GBMs (the regime-E raw-lags delta clears zero on its own but fails the joint-direction rule of Section~\ref{sec:eval-protocol}); it is nowhere statistically behind.}
\label{tab:e10-delta}
\footnotesize
\setlength{\tabcolsep}{4pt}
\begin{tabular}{lcccccc}
\toprule
 & \multicolumn{2}{c}{vs.\ Steelman GBM} & \multicolumn{2}{c}{vs.\ Raw-lags GBM} & \multicolumn{2}{c}{vs.\ Practice GBM} \\
\cmidrule(lr){2-3}\cmidrule(lr){4-5}\cmidrule(lr){6-7}
Regime & $\Delta$AUC & 95\% CI & $\Delta$AUC & 95\% CI & $\Delta$AUC & 95\% CI \\
\midrule
A & $+0.0037$ & $[+0.0028, +0.0045]$ & $+0.0013$ & $[+0.0002, +0.0025]$ & $+0.0029$ & $[+0.0019, +0.0038]$ \\
B & $+0.0011$ & $[+0.0005, +0.0017]$ & $+0.0065$ & $[+0.0040, +0.0092]$ & $+0.0010$ & $[+0.0004, +0.0017]$ \\
C & $+0.0018$ & $[+0.0011, +0.0023]$ & $+0.0040$ & $[+0.0029, +0.0051]$ & $+0.0017$ & $[+0.0010, +0.0025]$ \\
D & $-0.0001$ & $[-0.0022, +0.0020]$ & $+0.0019$ & $[-0.0019, +0.0059]$ & $+0.0003$ & $[-0.0021, +0.0026]$ \\
E & $-0.0003$ & $[-0.0009, +0.0003]$ & $+0.0036$ & $[+0.0022, +0.0050]$ & $+0.0002$ & $[-0.0004, +0.0009]$ \\
\bottomrule
\end{tabular}

\end{table}

\begin{figure}[t]
\centering
\includegraphics[width=\linewidth]{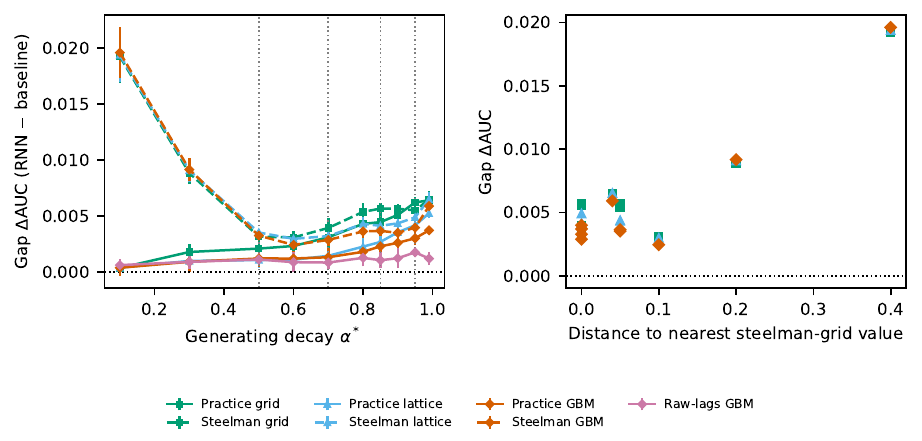}\\[2pt]
\includegraphics[width=\linewidth]{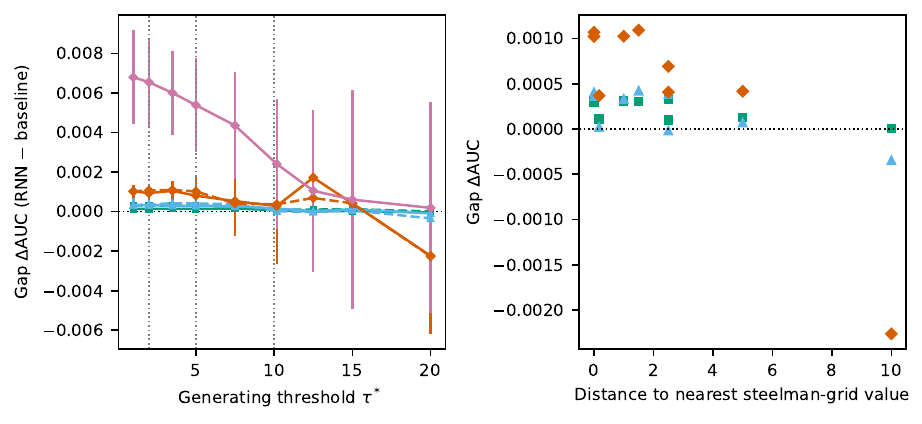}
\caption{E10 misspecification-cost curves: paired gap $\Delta$AUC (exposure RNN minus baseline, mean $\pm$ std over five seeds) under the recalibrated decay sweep on regime A (top) and threshold sweep on regime D (bottom); left panels plot the gap against the swept generating parameter (dotted verticals mark the grids' fixed values); right panels plot the steelman-grid consumers only---the models whose grid defines the distance---against the distance from the swept parameter to that grid; one legend serves both rows. The static monotone aggregate is omitted from the panels: its decay-sweep gap duplicates Figure~\ref{fig:e2-gap} at an order of magnitude above the grid consumers' typical gaps (peak $+0.044$) and would compress their curves, and on the threshold sweep it sits within the same ${\sim}\pm 0.002$ null band as the grid consumers. Reading key: each grid pays where its own coverage runs out, and the threshold sweep is a null for the grid consumers (see text).}
\label{fig:e10-misspec}
\end{figure}

\begin{figure}[t]
\centering
\includegraphics[width=\linewidth]{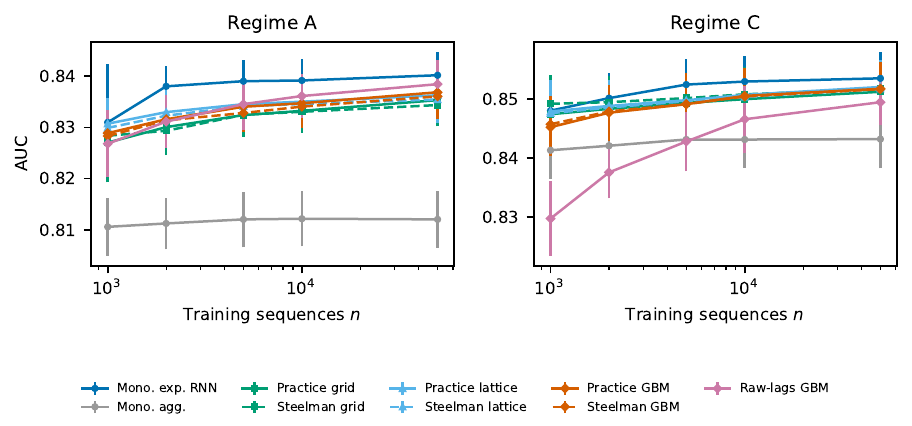}
\caption{E10 sample-size sweep on regimes A and C (AUC, mean $\pm$ std over five seeds; log-scale $n$), with all torch models trained to early-stopping convergence and validation/test splits fixed at full size across $n$. The learned channel is best at every $n$ on regime A and at every $n$ except $1000$ on regime C; the raw-lags GBM, the only comparator whose input dimension is not compressed by feature engineering or learned aggregation, degrades fastest as $n$ shrinks.}
\label{fig:e10-auc-n}
\end{figure}

\begin{figure}[t]
\centering
\includegraphics[width=\linewidth]{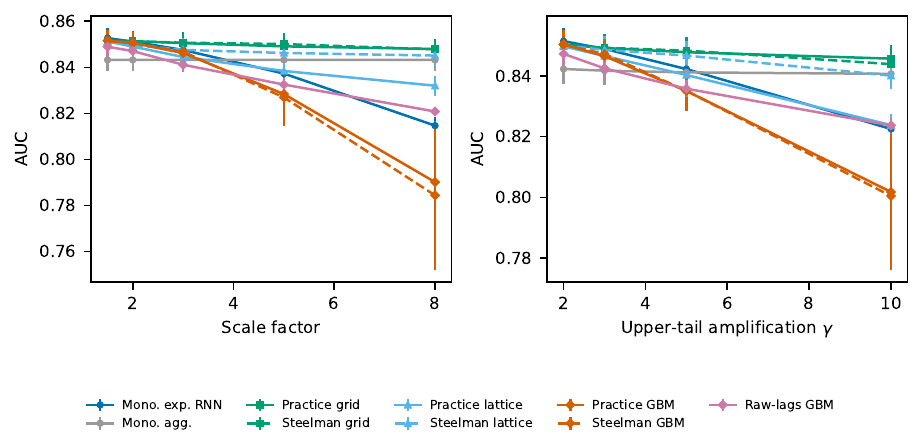}
\caption{E10 stress extrapolation on regime C: AUC versus severity under the scale (left) and upper-tail amplification (right) ladders for clean-trained models (mean $\pm$ std over five seeds). The grid GBMs collapse beyond the training support; the learned channel degrades gracefully; the fixed grids feeding neural monotone heads (monotone MLP or calibrated lattice), whose features are linear in the input scale, extrapolate best.}
\label{fig:e10-auc-shift}
\end{figure}

\begin{table}[tbp]
\centering
\caption{E10 score-degeneracy diagnostics per model and severity (mean over five seeds; $\times k$ denotes the scale multiplier and $\gamma k$ the upper-tail amplification at coverage $q = 0.5$): fraction of test scores saturated at the model's maximum attained score, and tie rate ($1 -$ fraction of distinct scores). The grid GBMs saturate and tie increasingly with severity---piecewise-constant ensembles over few engineered features go exactly flat beyond the training support---while the calibrated lattices and the learned channel never measurably tie; the grid-MLP heads tie moderately at extreme severity (up to $0.245$ at $\times 8$ and $0.254$ at $\gamma 10$, from tanh-head saturation) without the GBMs' rank collapse; the raw-lags GBM avoids degeneracy because its $24$ input columns keep leaf assignments distinct.}
\label{tab:e10-saturation}
\scriptsize
\setlength{\tabcolsep}{2pt}
\begin{tabular}{lcccccccccc}
\toprule
Model & base & $\times$1.5 & $\times$2.0 & $\times$3.0 & $\times$5.0 & $\times$8.0 & $\gamma$2 & $\gamma$3 & $\gamma$5 & $\gamma$10 \\
\midrule
\multicolumn{11}{l}{\emph{Saturation rate (fraction of test scores at the maximum)}} \\
Mono.\ exp.\ RNN & 0.000 & 0.000 & 0.000 & 0.000 & 0.000 & 0.000 & 0.000 & 0.000 & 0.000 & 0.000 \\
Mono.\ agg. & 0.000 & 0.000 & 0.000 & 0.000 & 0.000 & 0.000 & 0.000 & 0.000 & 0.000 & 0.000 \\
Practice grid & 0.000 & 0.001 & 0.001 & 0.003 & 0.014 & 0.049 & 0.001 & 0.003 & 0.012 & 0.067 \\
Steelman grid & 0.000 & 0.001 & 0.002 & 0.008 & 0.033 & 0.110 & 0.002 & 0.007 & 0.028 & 0.133 \\
Practice lattice & 0.000 & 0.000 & 0.000 & 0.000 & 0.000 & 0.000 & 0.000 & 0.000 & 0.000 & 0.000 \\
Steelman lattice & 0.000 & 0.000 & 0.000 & 0.000 & 0.000 & 0.000 & 0.000 & 0.000 & 0.000 & 0.000 \\
Practice GBM & 0.001 & 0.005 & 0.011 & 0.034 & 0.109 & 0.241 & 0.009 & 0.025 & 0.063 & 0.134 \\
Steelman GBM & 0.003 & 0.008 & 0.017 & 0.050 & 0.143 & 0.290 & 0.015 & 0.039 & 0.100 & 0.210 \\
Raw-lags GBM & 0.000 & 0.000 & 0.000 & 0.000 & 0.000 & 0.000 & 0.000 & 0.000 & 0.000 & 0.000 \\
\midrule
\multicolumn{11}{l}{\emph{Tie rate ($1 -$ fraction of distinct scores)}} \\
Mono.\ exp.\ RNN & 0.000 & 0.000 & 0.000 & 0.000 & 0.001 & 0.001 & 0.000 & 0.000 & 0.000 & 0.000 \\
Mono.\ agg. & 0.034 & 0.013 & 0.006 & 0.002 & 0.001 & 0.000 & 0.022 & 0.018 & 0.015 & 0.012 \\
Practice grid & 0.000 & 0.001 & 0.002 & 0.006 & 0.035 & 0.132 & 0.001 & 0.006 & 0.027 & 0.155 \\
Steelman grid & 0.001 & 0.001 & 0.004 & 0.016 & 0.084 & 0.245 & 0.004 & 0.014 & 0.062 & 0.254 \\
Practice lattice & 0.000 & 0.000 & 0.000 & 0.000 & 0.000 & 0.001 & 0.000 & 0.000 & 0.000 & 0.001 \\
Steelman lattice & 0.000 & 0.000 & 0.000 & 0.000 & 0.001 & 0.001 & 0.000 & 0.000 & 0.000 & 0.001 \\
Practice GBM & 0.076 & 0.045 & 0.059 & 0.141 & 0.351 & 0.585 & 0.062 & 0.111 & 0.242 & 0.486 \\
Steelman GBM & 0.270 & 0.181 & 0.167 & 0.244 & 0.461 & 0.676 & 0.190 & 0.234 & 0.377 & 0.633 \\
Raw-lags GBM & 0.000 & 0.000 & 0.000 & 0.000 & 0.000 & 0.001 & 0.000 & 0.000 & 0.000 & 0.000 \\
\bottomrule
\end{tabular}

\end{table}

\FloatBarrier
\subsection{E11: Functional emulation of hand-crafted statistics (C7/C8 corollary)}
\label{sec:e14}

The preceding experiments measure the learned channel's relation to hand-crafted statistics through prediction: E8 shows accuracy gaps closing without the target basis, and E10 shows parity at grid points with advantages off-grid. E11 is the study that converts those observations into the paper's headline claim. If the matched-admissibility question is what learning the aggregation is worth, the sharpest possible answer says which engineered statistics a learned channel \emph{contains} and which it does not---by regressing the channel directly onto each statistic, so that subsumption is demonstrated constructively rather than inferred from accuracy ties.

Table~\ref{tab:e14-spanning} reports the spanning table: held-out $R^2$ and Spearman $\rho$ for each learned channel regressed directly onto each target statistic (five seeds); Figure~\ref{fig:e14-boundary} plots the same boundary on one axis. The result makes the subsumption arguments of E8 and E10 constructive, and its three tiers match the placements of Proposition~\ref{prop:rank-closure}.

\begin{figure}[t]
\centering
\includegraphics[width=0.92\linewidth]{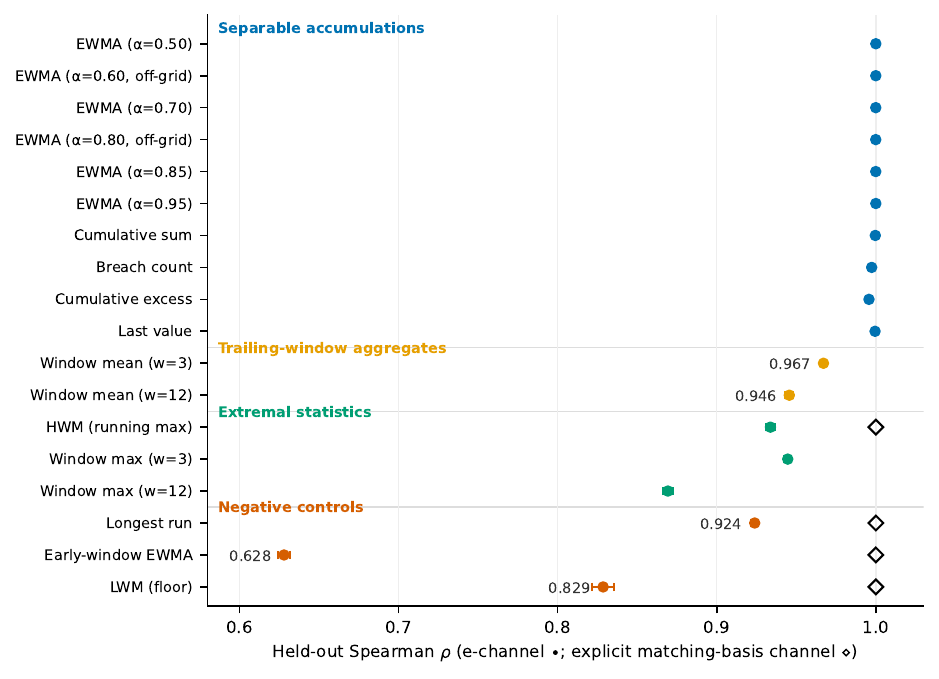}
\caption{The subsumption boundary in one panel: held-out Spearman $\rho$ of the learned e-channel regressed onto each of the eighteen target statistics (dots; whiskers are $\pm 1$ std over five seeds), ordered from separable accumulations to negative controls. The separable family sits at rank equivalence; trailing windows and extremes carry a measurable approximation ceiling; the negative controls fail as designed. Open diamonds mark the explicit matching-basis channels---the $\max$ channel for the HWM, the LWM channel for the floor, the early-window channel for the time-localization control, and the run-length channel for consecutivity---each of which restores its statistic exactly.}
\label{fig:e14-boundary}
\end{figure}

\emph{The separable family is reproduced, not approximated.} On every separable accumulation---EWMA at all six decays, including the two placed off every grid, the cumulative sum, breach count, cumulative excess, and last value---the e-channel reaches Spearman $\rho \ge 0.996$ (at $1.000 \pm 0.000$ for every EWMA decay and the cumulative sum) with $R^2 \ge 0.986$. Where the table reports $\rho = 1.000$, equality is to the three displayed decimals, not exact: residual rank discrepancies below that resolution remain possible. This is the constructive content of E10's misspecification result: a hand-crafted grid must commit to fixed decays and thresholds, while one channel \emph{class} adapts to whichever member the data requires---each table row is a single channel trained to a single target, so the claim is per-member adaptation, not simultaneous representation of the family by one fitted channel. Every statistic in the practice and steelman grids other than the windowed aggregates is functionally inside the learned channel pair---the EWMAs, hinges, counts, and last value in the e-channel, the HWM exactly in the m-channel.

\emph{Windows and extremes are approximated with a measurable ceiling.} Trailing-window means emulate at $\rho = 0.946$--$0.967$: no member of the geometric family realizes a step-shaped weight profile (Appendix~\ref{app:rank-closure}, which also notes that this is the one tier whose rank-closure status remains open), and the best trained geometric approximation is good but not exact---the functional counterpart of E10's finding that window grids and the learned channel trade narrow leads. Because this is the one tier with a measured ceiling and no designed adversary among regimes A--O, the regime-P diagnostic (Section~\ref{sec:experiments}) prices the ceiling predictively: the two-state channel lands $-0.0081 \pm 0.0021$ below the Bayes ceiling of $0.747$ (negative in all five seeds, and $-0.0065/{-}0.0061$ against the practice-grid and raw-lags GBMs, which hold the mechanism feature or its constituent lags and sit within $0.002$ of the ceiling)---a real but modest price, an order of magnitude below the structural failures' E8 gaps---and adding the recent-sum basis restores the ceiling exactly ($-0.0000 \pm 0.0001$), extending the one-basis repair pattern to the window tier. The HWM emulates at $\rho = 0.934$ from the e-channel (the rank-closure approximation) and exactly ($1.000$) from the m-channel, quantifying the learnability dividend of keeping the max explicit; windowed maxima are imperfect for both channels---$\rho = 0.869$--$0.945$ for the e-channel and far lower, $0.398$--$0.737$, for the m-channel---combining both restrictions.\footnote{The m-channel's profile has a one-line explanation: for non-decreasing $g_2$, $\max_t g_2(x_t) = g_2(\max_t x_t)$ (up to the initial state), so every m-channel instance is a monotone reparameterization of the global maximum---which is why it reproduces the HWM exactly and cannot represent windowed maxima or the floor.}

\emph{The negative controls fail as designed, and grade the E8 boundary.} The early-window EWMA---regime K's mechanism---is the hardest target at $\rho = 0.628$, confirming that time localization is the sharpest structural boundary; the explicit early-window channel (the same recurrence with increments accepted only for $t \le t_0$) reproduces it exactly ($\rho = 1.000$, $R^2 = 0.987$), the constructive counterpart of the M9f rung that returns regime K to its Bayes ceiling in E8. The longest run emulates at $\rho = 0.924$ under the e-channel, an imperfect proxy consistent with the half-open predictive gap on regime G; the explicit run-length channel (the current-run/max-run state above the same threshold) reproduces it exactly ($\rho = 1.000$, $R^2 = 1.000$), the constructive counterpart of E8's run rung. The LWM resolves the regime-O reading of Table~\ref{tab:e8-necessity}: the e-channel reaches only $\rho = 0.829$ (and the m-channel $0.263$), so the floor lies functionally beyond the reach of both \emph{trained} channels. The floor does sit in the separable family's idealized rank closure (Proposition~\ref{prop:rank-closure}), but that point is attainable only as steepness grows without bound (Section~\ref{sec:e8-results}), so the near-closed predictive residual on regime O reflects a distribution-specific proxy, not spanning---while the explicit LWM channel reproduces the floor exactly ($\rho = 1.000$, $R^2 = 0.998$). One channel---a parameter-free $\min$ recurrence under the same learned transform class---closes what the largest basis-level gap in the study leaves open. One concession frames all four explicit-channel rows: they are by-construction checks rather than discoveries---the run-length channel's state contains its target, and the LWM channel's rank exactness follows from $\min_t g(x_t) = g(\min_t x_t)$ for any monotone $g$, the mirror of the m-channel identity above---so their evidentiary role is confirming that each repair is admissible and that training attains it, not that the emulation is difficult.

One reading note on the tiers: the boundary is structural, drawn by basis family (inside the separable family, window-shaped, extremal, or outside), not by a threshold on $\rho$---and measured $\rho$ varies non-monotonically across the tiers (the longest run, a control, emulates at $0.924$, above the e-channel's $0.869$ on the 12-step window max, a windowed-extremal target). The families say \emph{why} a target is or is not reachable; the correlations say how well the best proxy does, and the two orderings need not coincide. The two readings are also differently portable: the structural placements are distribution-free---they are proved (Proposition~\ref{prop:rank-closure})---whereas each measured $\rho$ quantifies the best trained proxy on this exposure process and would shift under a differently persistent or heavier-tailed one.

The spanning table also bounds the claim it supports. Functional emulation shows the learned channel can \emph{represent} what hand-crafted grids compute; whether that representation wins predictively remains conditional, as E10 establishes---parity near grid points, advantages off-grid and at small samples, and boundaries under extreme extrapolation. What E11 rules out is the converse position: that hand-crafted monotone aggregates carry temporal content a learned monotone channel cannot express. Outside consecutivity, time localization, and the floor---each repaired by one explicit admissible basis, demonstrated in this table for all three ($\rho = 1.000$ for the run-length, early-window, and LWM channels) and predictively in E8 (Table~\ref{tab:e8-necessity})---engineered aggregates carry none.

\begin{table}[t]
\centering
\caption{E11 functional-emulation spanning table: held-out $R^2$ and Spearman $\rho$ (mean $\pm$ std over five seeds) for single-basis learned channels trained by MSE regression to reproduce each target statistic. The separable family is reproduced to rank equivalence, on- and off-grid; windows and extremes are approximated with a measurable ceiling; the negative controls fail as designed, and the explicit run-length, early-window, and LWM channels restore all three exactly.}
\label{tab:e14-spanning}
\footnotesize
\setlength{\tabcolsep}{4pt}
\begin{tabular}{llcc}
\toprule
Target statistic & Channel & $R^2$ & Spearman $\rho$ \\
\midrule
\multicolumn{4}{l}{\emph{Separable accumulations}} \\
\quad EWMA ($\alpha{=}0.50$) & e-channel & $0.992 \pm 0.004$ & $1.000 \pm 0.000$ \\
\quad EWMA ($\alpha{=}0.60$, off-grid) & e-channel & $0.993 \pm 0.003$ & $1.000 \pm 0.000$ \\
\quad EWMA ($\alpha{=}0.70$) & e-channel & $0.993 \pm 0.004$ & $1.000 \pm 0.000$ \\
\quad EWMA ($\alpha{=}0.80$, off-grid) & e-channel & $0.995 \pm 0.005$ & $1.000 \pm 0.000$ \\
\quad EWMA ($\alpha{=}0.85$) & e-channel & $0.995 \pm 0.006$ & $1.000 \pm 0.000$ \\
\quad EWMA ($\alpha{=}0.95$) & e-channel & $0.995 \pm 0.004$ & $1.000 \pm 0.000$ \\
\quad Cumulative sum & e-channel & $0.995 \pm 0.002$ & $1.000 \pm 0.000$ \\
\quad Breach count & e-channel & $1.000 \pm 0.000$ & $0.997 \pm 0.000$ \\
\quad Cumulative excess & e-channel & $0.993 \pm 0.004$ & $0.996 \pm 0.000$ \\
\quad Last value & e-channel & $0.986 \pm 0.019$ & $0.999 \pm 0.000$ \\
\midrule
\multicolumn{4}{l}{\emph{Trailing-window aggregates}} \\
\quad Window mean ($w{=}3$) & e-channel & $0.940 \pm 0.008$ & $0.967 \pm 0.001$ \\
\quad Window mean ($w{=}12$) & e-channel & $0.915 \pm 0.011$ & $0.946 \pm 0.002$ \\
\midrule
\multicolumn{4}{l}{\emph{Extremal statistics}} \\
\quad HWM (running max) & e-channel & $0.906 \pm 0.007$ & $0.934 \pm 0.003$ \\
\quad HWM (running max) & m-channel & $0.993 \pm 0.013$ & $1.000 \pm 0.000$ \\
\quad Window max ($w{=}3$) & e-channel & $0.880 \pm 0.003$ & $0.945 \pm 0.002$ \\
\quad Window max ($w{=}3$) & m-channel & $0.254 \pm 0.009$ & $0.398 \pm 0.012$ \\
\quad Window max ($w{=}12$) & e-channel & $0.787 \pm 0.013$ & $0.869 \pm 0.003$ \\
\quad Window max ($w{=}12$) & m-channel & $0.656 \pm 0.025$ & $0.737 \pm 0.010$ \\
\midrule
\multicolumn{4}{l}{\emph{Negative controls}} \\
\quad Longest run & e-channel & $0.811 \pm 0.004$ & $0.924 \pm 0.002$ \\
\quad Longest run & run-length channel & $1.000 \pm 0.000$ & $1.000 \pm 0.000$ \\
\quad Early-window EWMA & e-channel & $0.397 \pm 0.018$ & $0.628 \pm 0.004$ \\
\quad Early-window EWMA & early-window channel & $0.987 \pm 0.022$ & $1.000 \pm 0.000$ \\
\quad LWM (floor) & e-channel & $0.819 \pm 0.032$ & $0.829 \pm 0.007$ \\
\quad LWM (floor) & m-channel & $0.086 \pm 0.006$ & $0.263 \pm 0.013$ \\
\quad LWM (floor) & LWM channel & $0.998 \pm 0.001$ & $1.000 \pm 0.000$ \\
\bottomrule
\end{tabular}

\end{table}

\FloatBarrier
\section{Discussion and Limitations}
\label{sec:discussion}

\paragraph{Scope of the synthetic evidence.}
The evaluation is deliberately synthetic: it provides ground-truth monotone mechanisms, exact knowledge of which basis drives each label, the true distance between a generating parameter and every grid value, and the ability to construct adversarial regimes targeting individual bases---none of which is available with real data. Within that scope, the findings are conditional in a specific way. Learned monotone recurrence pays when the mechanism is order-dependent (E2) and when hand-crafted grids are misspecified (E10); the guarantee itself is never the differentiator. Two design constants also bound the scope: every experiment runs at horizon $T = 24$ (the cited bureau products' window), and the rank-closure limits of Proposition~\ref{prop:rank-closure} interact with horizon, so the sharpness of the boundary at substantially longer horizons is unmeasured; class prevalence is calibrated to ${\sim}25\%$ throughout, with the $5\%$-prevalence sensitivity of E10 as the one lower-base-rate check. What the synthetic design cannot establish is prevalence of the mechanisms themselves: whether real portfolios carry structure strong enough, and far enough off any sensible grid, for the misspecification cost to be material. A companion applied study on production credit data \citep{tan2026governed} takes up that question, running the same matched-admissibility protocol on deployed-scale portfolios.

\paragraph{Predictive cost, its confounds, and partial identifiability.}
C3 is phrased as \emph{little or no} predictive cost deliberately, because an apparent constraint cost is easily misattributed. Two confounds matter in practice. The first is representational: with a convex activation, the non-negative-weight function class is convex-only (Section~\ref{sec:architecture}), so a concave generating mechanism produces a persistent deficit that is the price of the activation, not of the constraint---E6 removes it with a non-convex monotone activation, and E9's four-parameter tier exhibits it in its purest form, surviving an unlimited training budget (Table~\ref{tab:e9-capacity}). The second confound is optimization: constrained models can converge more slowly than unconstrained baselines, so aggressive early-stopping rules truncate them disproportionately; stopping patience and training budget are part of the comparison, not fixed background. With both controlled, parity is what we observe: the paired per-regime $95\%$ confidence intervals of Table~\ref{tab:e3-delta} exclude any deficit beyond $0.0017$ AUC against either unconstrained sequence baseline. The same function-class analysis makes interpretability (C5) asymmetric across the channels: $g_2$ feeds a $\max$, is observed at its extremes, and is recovered under every activation, while $g_1$ enters through a discounted sum---distinct $(g_1,\alpha)$ pairs can induce similar distributions of $e_T$---so its recovery is activation-dependent and partial (Table~\ref{tab:e6-recovery}): the guarantee is parameterization-independent, but the interpretability of the learned shape is not.

\paragraph{What the two-state summary cannot express.}
The sharpened form of C7 (Section~\ref{sec:e8-results}; the boundary placed by proof in Proposition~\ref{prop:rank-closure} and quantified by regression in Table~\ref{tab:e14-spanning}) converts a vague expressiveness worry into a concrete audit question---does the risk mechanism plausibly involve consecutivity, time-localized windows, or a persistent floor? In consumer credit the first answer is routinely yes: delinquency states are defined by consecutive missed payments, and sustained-breach counts are shipped trended attributes, so the consecutivity boundary should be expected to bind in practice and the run-length channel is better read as a default member of the state than as a repair. The comparators expose a second route alongside the explicit repairs: a constrained tree over raw lags also crosses the boundary (Section~\ref{sec:e10-results}), at the sample-efficiency price measured there (a generic weight-constrained monotone recurrence does not cross it in any tested form), whereas the explicit channels restore the missing statistics inside the two-state architecture at no such price: the run-length, early-window, and LWM channels each reproduce their statistic exactly ($\rho = 1.000$) in E11 and sit at their regimes' Bayes ceilings in E8 (Table~\ref{tab:e8-necessity}); the trailing-window tier answers the same audit question at smaller stakes, with the recent-sum channel returning the window-price diagnostic (regime P) to its ceiling (Section~\ref{sec:e14}). Extending the state restores the guarantee by the same composition argument as Theorem~\ref{thm:score}.

\paragraph{Contiguous gaps stress the recency state.}
Missingness robustness (C4b) has one structural caveat. The EWMA state is a discounted sum, $e_T=\sum_k \alpha^k g_1(x_{T-k})$, concentrating weight on the most recent steps. A \emph{contiguous} block of zero-imputed steps near the readout therefore forces sustained decay with no interleaved observation to refresh the state, attenuating $e_T$ far more than the same number of missing entries scattered across the sequence---whereas order-invariant aggregates (sum, mean, max) respond only to how many entries are missing, not where. The effect is a structural property of recency weighting, not a monotonicity failure: under scattered (MCAR) missingness the recurrent state retains its predictive advantage over the static aggregate ($+0.010$ AUC at every rate under forward-fill), and lengthening zero-imputed contiguous blocks erodes that advantage to parity at the longest outage we test ($+0.004$ at block length 6, $-0.001$ at block length 8; Figure~\ref{fig:e5-missingness}). This is exactly where the two non-inferiority failures in Table~\ref{tab:e5-noninf} occur. Forward-fill imputation, which carries the last observation through the gap, preserves the advantage at every block length ($+0.010$ at block length 8).

\paragraph{Relation to deployed practice.}
The framing of C8 is a division of labor: the synthetic design measures exactly the quantities that real data cannot---the misspecification-cost curve (knowable only when the generating parameter is known), the sample-efficiency gap against a known Bayes ceiling, and the structural saturation of piecewise-constant ensembles beyond the training support. E10's stress phase also marks a boundary of the \emph{learned} instantiation within that division: at the most extreme severities the best extrapolators are fixed grids feeding neural monotone heads (monotone MLP or calibrated lattice alike), so a deployment expecting routine evaluation far beyond the training support may prefer hand-crafted features inside the same monotone composition. One caveat keeps the stress finding in its lane: in production, scoring far beyond the development support is itself a model-risk event---inputs are commonly capped at development-range percentiles, and population-stability monitoring flags such inputs for review before scores are consumed---so the finding speaks to stress-testing and scenario analysis, not to routine scoring under extreme drift.

\paragraph{Governance properties.}
Two governance properties are independent of that answer. First, parsimony has an operating-cost reading: under SR 11-7-style validation, every admitted feature carries its own conceptual-soundness documentation, ongoing monitoring, and change control, so replacing a multi-column engineered library with a handful of provable states---each a monotone function with plottable transforms---can reduce a recurring per-feature validation workload, not just a parameter count; the offset is that the learned channel adds its own model-level validation and re-fits on retraining, so the net effect is deployment-specific. Second, the upstream-score deployment shape mirrors the sub-score stacking already standard in credit pipelines (behavioral scores and trended-data attributes feeding downstream models), so adopting the learned channel changes what the upstream component is, not where it sits.

\paragraph{How large are these effects?}
The synthetic magnitudes are easiest to read against one another. The recurrent state's advantage over the static aggregate baseline peaks at $+0.044$ AUC (E2; $+0.034$ against the tanh-headed aggregate, Section~\ref{sec:results}), and grid misspecification costs the densified steelman grid up to $0.019$ AUC against its gradient-boosted consumer at off-grid decays (E10; the cost of committing to that grid)---both an order of magnitude above the ${\le}0.0017$ AUC constraint cost bounded by E3. The basis-level gaps of E8 run larger still, up to $+0.143$ AUC on the exposure-floor regime, and are the gaps one explicit admissible basis closes. For an external anchor: since Gini $= 2\,\mathrm{AUC} - 1$, an AUC delta of $0.01$ is two Gini points, a magnitude often treated as material in scorecard champion--challenger comparisons (conventions vary; large books act on less). The E2 peak gap and the E8 basis gaps clear it comfortably---with the caveat that the E2 gap is measured against the deliberately naive static control that carries C2's scientific claim, not against a deployable champion; the deployment-relevant analogue at the same sweep point is E10's sub-$0.004$ best-consumer margin. The worst-case misspecification cost clears the bar too, but only against the single committed grid; the best-consumer envelope falls short of it, and the parity-ladder leans do not reach it.

\paragraph{Uncertainty quality versus cost.}
Corollary~\ref{cor:mcdropout} guarantees that MC dropout preserves monotone semantics, but the guarantee says nothing about the \emph{quality} of the resulting uncertainty, so we measured it in a supplementary ensemble--MC-dropout comparison, labeled S1c after its configuration in the released code (Table~\ref{tab:s1c} in Appendix~\ref{app:results-tables}; protocol in Appendix~\ref{app:repro}). On the mean prediction the two mechanisms are indistinguishable---paired AUC, log loss, Brier, and ECE deltas all within noise on regimes A and C---but the uncertainty itself is not: a five-member deep ensemble's per-example predictive standard deviation correlates with absolute error at Spearman $0.75$--$0.77$ against MC dropout's $0.20$--$0.21$ (paired delta $+0.55 \pm 0.08$ pooled over the ten regime--seed cells, positive in every cell), and the ensemble cuts the area under the risk--coverage curve from $0.24$--$0.26$ to $0.14$--$0.15$, at five times the training cost---exactly the member count. Both mechanisms are monotone per sample by construction (dropout by Corollary~\ref{cor:mcdropout}; ensembles member-wise by Theorem~\ref{thm:score}), so the choice is purely an alignment-versus-compute trade: MC dropout is cheap and semantically safe, while these results favor ensembles of monotone models for practitioners who act on uncertainty---selective review, referral cutoffs.

\paragraph{Future work.}
Four extensions are open. First, mask-aware decay: suspending $\alpha$-decay over unobserved steps would address the contiguous-gap caveat without weakening the monotonicity argument. Second, partial monotonicity: extending the guarantee to settings where only a subset of inputs carries the constraint. Third, the multivariate generalization---vector-valued exposure paths whose increments are learned jointly across channels---and the conditioning of the decay on exogenous context, both of which preserve the guarantee by the same composition arguments. Fourth, a default-inclusion policy for the demonstrated repair channels: the run-length, early-window, LWM, and recent-sum channels are built and proven here as audit-triggered repairs, and what remains open is when a deployment should carry them by default---the selection problem, not the channels.

\section{Conclusion}
\label{sec:conclusion}

This paper argued that the standard evaluation of monotone sequence models measures the wrong quantity for the setting that motivates them: deployed monotone pipelines---hand-crafted monotone aggregates feeding sign-constrained boosting---are already monotone end-to-end by composition, so constrained-versus-unconstrained comparisons price a guarantee the incumbent has for free. The comparison that matters runs at matched admissibility and measures what learning the temporal aggregation is worth.

Built and run, that evaluation returns a boundary rather than a scoreboard. By direct functional regression, a learned monotone recurrent channel spans, member by member, the geometrically weighted separable family of hand-crafted statistics to rank equivalence---learned and engineered monotone features are substitutes on separable structure---and fails at consecutivity and time localization structurally and at the exposure floor as a learnability boundary, each repaired by one explicit admissible basis; trailing-window and windowed-extremal statistics in between are approximated at a measurable ceiling. Predictively, the learned channel is never statistically worse in-distribution at full sample size on the core regimes at its specified capacity---the scope of every parity statement here, which excludes E9's deliberately under-parameterized ablation tiers---and earns its advantage where a committed grid cannot adapt: uncovered decay regions, scarce training data against the highest-dimensional comparator, and stress beyond the training support where grid-fed tree-ensemble scores go flat. Two symmetric boundaries are stated: at the most extreme severities, fixed grids feeding neural monotone heads extrapolate best, and on the boundary regimes, refusing aggregation altogether---raw lags under a constrained tree---is the strongest incumbent, at the sample-efficiency price the evaluation quantifies. No single comparator is dominated on all three advantage axes, so the case for learning is adaptivity across failure modes no committed pipeline covers at once, not uniform dominance.

The instrument that makes these measurements possible is a recurrent architecture whose state consists of classical risk statistics, monotone by construction in every input, under every admissible configuration, and per MC-dropout sample---a governance constraint treated as a property of the hypothesis class rather than a training outcome. What this paper offers is the evaluation standard, the boundary, and the instrument that makes both measurable; whether real portfolios reward the learning is the prevalence question the same protocol can answer on production data.

\section*{Broader Impact Statement}
This work is motivated by credit-risk scoring, a domain with direct consumer impact. In that setting, the monotonicity constraint studied here is generally protective: it prevents a score from penalizing a borrower for reducing exposure, keeps adverse-action reasoning coherent (the per-input recourse argument of Section~\ref{sec:setup}), and makes model behavior auditable under retraining and stochastic inference. Two caveats follow from the same analysis. First, a monotonicity guarantee is not a fairness guarantee: it constrains the direction of feature effects, not calibration or error rates across subpopulations, and a monotone score can still encode disparate impacts that require their own assessment. Second, every result here is synthetic by design; deploying a learned monotone channel in production remains subject to the applicable model-risk-management validation, which this paper's evaluation design is meant to inform, not replace. The experiments involve no human data.

\FloatBarrier

\appendix

\section{Monotonicity of the additional admissible bases}
\label{app:bases}

This appendix proves Proposition~\ref{prop:bases}. Throughout, fix $\tau\ge 0$ and let $x\preceq x'$ denote the coordinatewise order of Section~\ref{sec:setup}.

\begin{lemma}[Last exposure]
\label{lem:last}
$L_T(x)=x_T$ is monotone non-decreasing in each $x_t$.
\end{lemma}
\begin{proof}
$L_T$ is non-decreasing in $x_T$ and independent of $x_1,\dots,x_{T-1}$.
\end{proof}

\begin{lemma}[Cumulative sum]
\label{lem:cumsum}
$C_T(x)=\sum_{t=1}^T x_t$ is monotone non-decreasing in each $x_t$.
\end{lemma}
\begin{proof}
Increasing any coordinate increases the sum.
\end{proof}

\begin{lemma}[Breach count]
\label{lem:breach}
$B_T(x)=\sum_{t=1}^T \mathbf{1}[x_t>\tau]$ is monotone non-decreasing in each $x_t$.
\end{lemma}
\begin{proof}
Each indicator $\mathbf{1}[x_t>\tau]$ is non-decreasing in $x_t$ (it can only switch from 0 to 1 as $x_t$ increases), and sums of non-decreasing functions are non-decreasing.
\end{proof}

\begin{lemma}[Cumulative excess]
\label{lem:excess}
$E_T(x)=\sum_{t=1}^T (x_t-\tau)^+$ is monotone non-decreasing in each $x_t$.
\end{lemma}
\begin{proof}
$(x-\tau)^+=\max(0,x-\tau)$ is non-decreasing, and sums preserve monotonicity.
\end{proof}

\begin{lemma}[Longest run above threshold]
\label{lem:run}
The longest-run statistic $R_T(x)$, defined as the maximum length of a run of consecutive indices with $x_t>\tau$, is monotone non-decreasing in each $x_t$.
\end{lemma}
\begin{proof}
Increasing an element can only add indices to the above-threshold set $\{t: x_t>\tau\}$. Enlarging this set cannot shorten any existing run, so the maximum run length cannot decrease.
\end{proof}

\begin{lemma}[Saturating transforms]
\label{lem:saturating}
For any non-decreasing $\phi:\mathbb{R}_{\ge 0}\to\mathbb{R}$ (e.g., $\phi(x)=\log(1+x)$ or $\phi(x)=x/(x+c)$ for $c>0$), the aggregate $S_T(x)=\sum_{t=1}^T \phi(x_t)$ is monotone non-decreasing in each $x_t$.
\end{lemma}
\begin{proof}
$\phi$ is non-decreasing, and sums preserve monotonicity.
\end{proof}

\begin{lemma}[Low-water mark]
\label{lem:lwm}
For monotone non-decreasing $g$, the low-water mark $\ell_t = \min\big(\ell_{t-1},\, g(x_t)\big)$ with fixed $\ell_0$ is monotone non-decreasing in each $x_t$.
\end{lemma}
\begin{proof}
Identical to Lemma~\ref{lem:hwm} with $\min$ in place of $\max$: $\min(\cdot,\cdot)$ is non-decreasing in each argument. With $\ell_0$ chosen above the range of $g$ on the inputs (the implementation uses a large finite sentinel), the recurrence computes exactly $\ell_T = \min_t g(x_t)$, the statistic of Proposition~\ref{prop:bases}; for smaller $\ell_0$ it computes the capped variant $\min(\ell_0, \min_t g(x_t))$, which is monotone by the same argument.
\end{proof}

\begin{lemma}[Early-window EWMA]
\label{lem:earlyewma}
For monotone non-decreasing $g$, $\alpha \in (0,1)$, and fixed $t_0 < T$, the early-window EWMA $e_{t_0}(x) = \sum_{t \le t_0} \alpha^{t_0 - t} g(x_t)$ is monotone non-decreasing in each $x_t$.
\end{lemma}
\begin{proof}
It is a non-negatively weighted sum of non-decreasing functions of the inputs it observes, and constant in $x_{t_0+1:T}$.
\end{proof}

\section{The separable family and its rank closure}
\label{app:rank-closure}

This appendix formalizes the spanning claims that Sections~\ref{sec:models}, \ref{sec:results}, and \ref{sec:e14} rely on. Fix the horizon $T$ and let
\begin{equation*}
\mathcal{S} \;=\; \Big\{\, x_{1:T} \mapsto \textstyle\sum_{t=1}^T \alpha^{T-t}\, g(x_t) \;:\; \alpha \in (0,1),\ g:\mathbb{R}_{\ge 0}\to\mathbb{R} \text{ monotone non-decreasing} \,\Big\}
\end{equation*}
be the idealized family the e-channel parameterizes ($g$ ranges over all monotone non-decreasing transforms, not only those a bounded monotone MLP realizes). Because ranking metrics are invariant to strictly increasing transforms of the score, the operative notion is rank agreement: say a statistic $S$ lies in the \emph{rank closure} of $\mathcal{S}$ if for every finite set $X$ of sequences on which $S$ takes distinct values, some member of $\mathcal{S}$ orders $X$ exactly as $S$ does. Restricting to distinct-value sets leaves the emulating member free to break the target's ties; for the one heavily tied target below (the breach count) this is immaterial for evaluation, because when labels depend on the path only through the target, tied sequences are label-exchangeable and any tie-breaking by a function of the path leaves expected AUC unchanged---the same exchangeability argument the results sections use for regime L's cap.

\begin{proposition}[Membership and rank closure of the separable family]
\label{prop:rank-closure}
\begin{enumerate}
  \item[(i)] Every EWMA $\sum_t \alpha_0^{T-t} x_t$ with $\alpha_0 \in (0,1)$ is a member of $\mathcal{S}$, as is every transformed EWMA $\sum_t \alpha_0^{T-t} g(x_t)$.
  \item[(ii)] The cumulative sum, the cumulative excess, the last value, and---on sequences whose entries avoid the threshold $\tau$---the breach count lie in the rank closure of $\mathcal{S}$.
  \item[(iii)] The HWM $\max_t x_t$ lies in the rank closure of $\mathcal{S}$.
  \item[(iv)] The LWM $\min_t x_t$ lies in the rank closure of $\mathcal{S}$.
  \item[(v)] The early-window EWMA $e_{t_0}$ with $1 \le t_0 < T$ and non-constant $g$ lies outside the rank closure of $\mathcal{S}$; a fortiori, no member realizes it. (A constant $g$ makes $e_{t_0}$ constant, hence vacuously inside.)
  \item[(vi)] For $T \ge 3$, the longest run above a threshold lies outside the rank closure of $\mathcal{S}$.
\end{enumerate}
\end{proposition}

\begin{proof}
(i) Take $\alpha = \alpha_0$ and the target transform for $g$ (the identity for the plain EWMA); both are admissible.

(ii) \emph{Cumulative sum:} with $g$ the identity, $\sum_t \alpha^{T-t} x_t \to \sum_t x_t$ uniformly on any bounded set of sequences as $\alpha \to 1^-$. On a finite set $X$ where the cumulative sums are distinct, choose $\alpha$ close enough to $1$ that the uniform error is below half the minimum gap between distinct target values; ranks then agree exactly. \emph{Cumulative excess:} identical with $g$ the hinge $(x-\tau)^+$. \emph{Last value:} as $\alpha \to 0^+$ the statistic converges to $g(x_T) = x_T$ with $g$ the identity; the same finite-gap argument applies. \emph{Breach count:} take $g_k(x) = 1/(1 + e^{-k(x-\tau)})$, an admissible sigmoid. If every entry of every sequence in $X$ differs from $\tau$ by at least $\delta > 0$, then $\sum_t g_k(x_t) \to \sum_t \mathbf{1}[x_t > \tau]$ uniformly on $X$ as $k \to \infty$; combine with $\alpha \to 1^-$ as above.

(iii) Take $g(x) = x^p$ with $p \ge 1$ (admissible). For a sequence with maximum $M > 0$,
$\alpha^{T-1} M^p \le \sum_t \alpha^{T-t} x_t^p \le T M^p$.
Hence for two sequences with maxima $M < M'$, the ratio of their statistics is at most $\big(T/\alpha^{T-1}\big) (M/M')^p \to 0$ as $p \to \infty$, so for $p$ large enough the statistics order as the maxima (a pair containing the all-zero sequence, $M = 0$, is ordered trivially: its statistic is $0$ against a strictly positive one). A finite set $X$ with distinct maxima involves finitely many pairs, so one finite $p$ orders all of them.

(iv) Fix $c > 0$ and take $g_{p,c}(x) = c^{-p} - (x+c)^{-p}$, which is non-negative and monotone non-decreasing on $\mathbb{R}_{\ge 0}$, hence admissible. Then
$\sum_t \alpha^{T-t} g_{p,c}(x_t) = c^{-p} \sum_t \alpha^{T-t} - \sum_t \alpha^{T-t} (x_t + c)^{-p}$,
and the first term does not depend on the sequence, so the member ranks $X$ exactly as $-\sum_t \alpha^{T-t}(x_t+c)^{-p}$ does. Writing $m = \min_t x_t + c$, the same two-sided bound as in (iii) gives $\alpha^{T-1} m^{-p} \le \sum_t \alpha^{T-t} (x_t+c)^{-p} \le T\, m^{-p}$; for two sequences with floors $m < m'$, the sequence with the lower floor eventually has the strictly larger subtracted mass and hence the strictly lower score. For $p$ large enough the member therefore orders $X$ exactly as the floor does---the mirror of (iii).

(v) Since the target's $g$ is non-constant, pick $H > 0$ with $g(H) > g(0)$, and let $x^1 = (H, 0, \dots, 0)$ and $x^2 = (0, \dots, 0, H)$. The target takes distinct values on this pair: writing $\alpha_0$ for the target's decay, $e_{t_0}(x^1) - e_{t_0}(x^2) = \alpha_0^{t_0 - 1} \big(g(H) - g(0)\big) > 0$, using $t_0 \ge 1$ and $t_0 < T$ (so the last coordinate lies outside the window). For any member of $\mathcal{S}$---with its own transform $\tilde{g}$ and decay $\alpha$, not necessarily the target's---
$S(x^1) - S(x^2) = \big(\tilde{g}(H) - \tilde{g}(0)\big)\big(\alpha^{T-1} - 1\big) \le 0$,
since $\alpha \in (0,1)$ puts the largest geometric weight on the last coordinate and $\tilde{g}$ is non-decreasing. Every member therefore either ties the pair or orders it \emph{opposite} to the target, so no member orders $\{x^1, x^2\}$ as $e_{t_0}$ does---at any parameter values, hence also in the limit. The statistic is nonetheless available to the architecture constructively: the explicit early-window channel (the prefix-restricted recurrence, admissible by the composition argument) reproduces $e_{t_0}$ exactly. The repair \emph{must} extend the state, because no separable member or separable limit reaches the target.

(vi) Let $u = (0, \dots, 0, H, H, 0)$ and $v = (0, \dots, 0, H, 0, H)$ with $H > \tau$: the longest run above $\tau$ is $2$ for $u$ and $1$ for $v$. For any member of $\mathcal{S}$ with transform $\tilde{g}$ and decay $\alpha$,
$S(u) - S(v) = \big(\tilde{g}(H) - \tilde{g}(0)\big)(\alpha - 1) \le 0$,
so again no member orders the pair as the run does. Adjacency information is invisible to every geometrically weighted separable accumulation: the run's repair, like the early window's, must extend the state.
\end{proof}

A parallel non-membership remark places the window tier of E11. The trailing-window mean $W_w(x) = w^{-1} \sum_{t=T-w+1}^{T} x_t$ with $1 < w < T$ is not a member of $\mathcal{S}$: exact equality $\sum_t \alpha^{T-t} \tilde{g}(x_t) = W_w(x)$ as functions would, varying $x_T$ alone, force $\tilde{g}(x) = x/w + c_0$, and varying $x_{T-1}$ alone force $\alpha \tilde{g}(x) = x/w + c_1$---impossible for $\alpha \in (0,1)$ and non-constant $\tilde{g}$. Unlike the run and the early window, however, its \emph{rank-closure} status is open: every member weights later coordinates more heavily, as the window does, so no two-sequence witness of the (v)--(vi) kind can settle it. The window tier's measured ceiling ($\rho = 0.946$--$0.967$ in Table~\ref{tab:e14-spanning}) is therefore reported without a representability verdict: it bounds what training delivers, and family non-membership says the reproduction cannot be exact, but whether some member sequence attains rank agreement in the limit is not resolved here.

One exact-membership statement completes the boundary trichotomy: the floor is not a \emph{member} of $\mathcal{S}$ either, so (iv) is a genuine closure statement. Fix $0 < a \le M$ and consider the four sequences $(x_1, M, \dots, M, x_T)$ with $x_1, x_T \in \{0, a\}$. Every member of $\mathcal{S}$ is additively separable across time steps, so the alternating sum $S(a,a) - S(a,0) - S(0,a) + S(0,0)$ vanishes, while for $\min_t x_t$ the same alternating sum equals $a > 0$. Together, (iv)--(vi) prove the trichotomy the results sections use: the floor lies outside the family but inside its rank closure (a learnability boundary), while the early window and the run lie outside the rank closure itself (a representability boundary).

Two readings matter for the paper. First, (iii) and (iv) are why the architecture keeps both extremal channels explicit: each extreme is a limit point a bounded learned transform approximates only at finite steepness, which the explicit HWM and LWM recurrences supply exactly---the learnability dividend E8 and E11 measure. Second, the limits in (ii)--(iv) require parameters ($\alpha \to 1$, $k \to \infty$, $p \to \infty$) that grow with the evaluation set and its gaps, so rank-closure membership is a statement about the family's expressive reach, not a claim that a trained bounded network attains it; the E11 spanning table measures how much of that reach training actually delivers.

\section{Structured monotonicity probes (P1--P5)}
\label{app:probes}

We use structured probes to localize common failure modes in unconstrained sequence models:
\begin{itemize}
  \item \textbf{P1 (single-step increase):} increase a single time index $t$ by a fixed amount.
  \item \textbf{P2 (early/late block increase):} increase a contiguous block either early or late in the sequence.
  \item \textbf{P3 (peak-aligned increase):} increase the time index containing the current maximum (peak) exposure.
  \item \textbf{P4 (micro-increments):} apply a small increment to a single low-exposure index at a time.
  \item \textbf{P5 (stress transition):} increase exposures around large observed jumps, a detected proxy for latent stress-state transitions.
\end{itemize}

\paragraph{Paired construction for P4 (micro-increments).}
Given a base sequence $x_{1:T}$, the probe perturbs each member of the low-exposure index set
\begin{equation}
S \;=\; \{\, t : x_t \le Q_{0.5}(x_{1:T}) \,\}
\end{equation}
(the per-sequence median) one at a time: the comparison sequence $\tilde{x}^{(t)}$ sets $\tilde{x}^{(t)}_t = x_t + \delta$ with $\delta = 0.05$ and leaves every other entry unchanged, giving $|S|$ ordered pairs per base sequence. By construction $x \preceq \tilde{x}^{(t)}$, so any monotone model must satisfy $f(\tilde{x}^{(t)}) \ge f(x)$. The pairs verify \emph{fine-grained} monotonicity in the flat region of the input: the increment is more than an order of magnitude smaller than the block and peak increases of P1--P3, so a model can pass the large-perturbation probes yet fail under increments at the scale of its score-surface noise, the failure mode P4 targets.

\section{Protocol details}
\label{app:protocols}

\paragraph{Distribution-shift severity grids (E4).}
Four shift families are swept over severity; all are ordering-preserving pathwise ($x' \succeq x$ row by row), so the frozen pre-shift labels remain valid. \emph{Scale:} all exposures multiplied by a factor in $\{1.5, 2, 3, 5, 8\}$. \emph{Upper-tail amplification:} the excess above a single pooled median threshold $\tau$ (quantile $0.5$ of all test exposures) is amplified as $x' = x + (\gamma - 1)(x - \tau)^+$ with $\gamma \in \{2, 3, 5, 10\}$---one piecewise-linear monotone map (slope $1$ below $\tau$, $\gamma$ above) applied identically to every entry. \emph{Spike frequency:} extra log-normal shocks are added at rate $0.05\,(m - 1)$ per step for $m \in \{2, 5, 10\}$ (where $0.05$ is the top of the generative per-sequence spike-rate range), adding extreme events without displacing the bulk of the distribution---a designed null case for ranking robustness. \emph{Mixture contamination:} a seed-controlled random $(1-\pi)$ fraction of test rows is replaced by the $\times 3$-scaled version of \emph{those same rows}, the rest kept at base, for $\pi \in \{0.9, 0.7, 0.5\}$; because each row keeps its index and is displaced upward pathwise, the existing labels remain exactly valid---a fresh draw from a different generator would invalidate them.

\paragraph{Missingness grid (E5).}
Masks are applied after generation, and labels are always computed from the unmasked sequence. Mechanisms: MCAR at rates $\{5, 10, 15, 20, 30\}\%$, and block missingness with contiguous block lengths $\{2, 4, 6, 8\}$. Each masked dataset is completed with zero, forward-fill, or mean imputation before being passed to the models, yielding a full mechanism $\times$ imputation grid. Every (mechanism, level, imputation) cell is evaluated under three independent mask draws per seed, and reported statistics average over draws within seed.

\paragraph{Adversarial operator.}
The gradient-based adversarial operator searches for a monotone perturbation that decreases the score. For each base sequence $x$, a non-negative perturbation $\delta$ is initialized uniformly on $[0,1]^T$ and updated for $50$ projected-gradient steps of size $0.1$ to minimize $f(x + \delta)$, with $\delta$ clamped to $\delta \ge 0$ after every step so that $x \preceq x + \delta$ throughout. The search runs with $3$ random restarts on each of $50$ held-out base sequences per (model, regime, seed)---the $150$ adversarial pairs per seed of E1---and a violation is recorded when the final score falls below the base score by more than the float64 tolerance.

\paragraph{Bayes ceilings.}
The Bayes AUC quoted for a regime (regimes G, K, and O in E8) is the AUC of the regime's true label index---the generating logit mechanism evaluated on a fresh $20{,}000$-sequence reference pool from the same exposure process, with the index checked to reproduce the registry labels exactly---which is rank-equivalent to the Bayes-optimal score under the label model, so no model consuming the same sequences can exceed it except by chance. Because rung AUCs are computed on each seed's $10{,}000$-sequence test split while the fresh-pool ceiling is not, a rung can nominally exceed the quoted ceiling by split noise; the ``at the ceiling'' claims are therefore verified by the paired protocol reported with Table~\ref{tab:e8-necessity}: the Bayes index is evaluated on each seed's own test split and the per-seed paired delta (rung minus Bayes index) carries the verdict. The Phase II oracle of E8 is distinct: a model fit over the full statistic pool including the non-monotone diagnostics, as stated in Section~\ref{sec:experiments}.

\paragraph{E8 incremental-MI debiasing.}
The Phase I conditional mutual information $I(b; y \mid \{\text{EWMA, HWM}\})$ is estimated by a stratified histogram plug-in, which inflates the estimate in small conditioning strata---worst for discrete bases such as the breach count, whose mass concentrates in few bins. Uncorrected, the estimator assigns the breach count $0.04$--$0.16$ nats of incremental information even on regime A, whose label depends on the EWMA functional alone and where the true conditional information is ${\approx}0$. The headline values are therefore debiased against a permutation null ($20$ within-stratum permutations of the candidate basis, whose mean estimates the finite-sample inflation), with the raw estimate recorded alongside. The permutation null removes most but not all of the artifact (the regime-A breach-count value falls to $0.014$--$0.086$ across seeds; the residual is coarse-stratification leakage, and the AUC ablations confirm it carries no label information), so debiased incremental MI is used for \emph{ranking} candidate bases---genuine signals stand out clearly, e.g.\ the longest run retains ${\sim}0.12$ nats on regime G---and small absolute values are not interpreted as real conditional information.

\paragraph{E8 Phase II subset search.}
Greedy forward selection (one basis added per step, subsets up to size $8$) scores each candidate basis subset with an $\ell_2$-regularized logistic regression on the z-scored basis values, fitted on the training split and scored on the validation split; the ${\ge}95\%$-of-oracle decision is likewise taken on validation scores. Every reported metric for the decided subsets is then computed on the test split, which the selection never touches, so the winner's curse of subset search does not reach the reported numbers.

\paragraph{E2 decay recalibration.}
Regime A is swept over a ten-point grid of generating decays,
\begin{equation*}
\alpha^* \in \{0.10,\ 0.30,\ 0.50,\ 0.60,\ 0.70,\ 0.80,\ 0.85,\ 0.90,\ 0.95,\ 0.99\},
\end{equation*}
with the grid densest near the expected peak and right shoulder of the gap curve. For each $\alpha^*$, the label coefficients are recalibrated so that (i) class prevalence is held at approximately $25\%$ and (ii) the logit slope on the \emph{standardized} EWMA feature is held fixed; only the recency structure of the label therefore varies along the sweep. The recalibration reproduces regime A's default coefficients at $\alpha^* = 0.9$.

\paragraph{E10 threshold-sweep recalibration.}
Regime D's generating threshold is swept over
\begin{equation*}
\tau^* \in \{1.0,\ 2.0,\ 3.5,\ 5.0,\ 7.5,\ 10.17,\ 12.5,\ 15.0,\ 20.0\},
\end{equation*}
crossing and exceeding the fixed grid thresholds $\{2, 5, 10\}$ used by the practice and steelman featurizations; $10.17$ is the regime's published reference threshold. For each $\tau^*$, the stress feature is $(\mathrm{HWM}(x) - \tau^*)^+$ on a fixed reference pool; the slope is rescaled so the standardized-stress coefficient matches the reference regime's, and the intercept is bisection-solved for the target prevalence (approximately $25\%$), mirroring the decay recalibration above. A guard rejects thresholds so deep in the HWM tail that the stress feature degenerates on the reference pool.

\paragraph{E10 sample-size grid.}
The training split is truncated to $n \in \{1000,\ 2000,\ 5000,\ 10000,\ 50000\}$ sequences on regimes A and C, with validation and test fixed at full size.

\paragraph{E11 emulation protocol.}
Emulation uses sequences of length $24$ drawn from the same exposure process as the predictive experiments; labels play no role. Each seed generates $50{,}000$ training, $10{,}000$ validation, and $10{,}000$ test sequences. Every target statistic is standardized by its train-split mean and scale (a monotone affine map, so ranks are unaffected), each channel is trained by MSE regression onto the standardized target with the same optimizer settings as the predictive experiments, and held-out $R^2$ and Spearman $\rho$ are computed on the test split.

\paragraph{Data-generating-process constants.}
Tables~\ref{tab:dgp-exposure} and~\ref{tab:dgp-regimes} list every numeric constant of the synthetic suite, making the paper self-contained for replication. The exposure-process parameters are drawn \emph{per sequence} from the uniform ranges of Table~\ref{tab:dgp-exposure}, so the population is heterogeneous by construction. All experiments use $T = 24$ and, in full mode, $50{,}000$ / $10{,}000$ / $10{,}000$ train/validation/test sequences per seed with seeds $\{42, \dots, 46\}$. In Table~\ref{tab:dgp-regimes}, $E$ denotes $\mathrm{EWMA}_{0.9}$ of the raw exposures ($e_t = 0.9\, e_{t-1} + x_t$, unnormalized), $\tilde{E}$ the same recurrence on $\log(1+x_t)$, $M = \max_t x_t$, $L = \min_t x_t$, $\mathrm{br}_\tau = \sum_t \mathbf{1}[x_t > \tau]$, $\mathrm{ex}_\tau = \sum_t (x_t - \tau)^+$, $\mathrm{run}_\tau$ the longest run of $\mathbf{1}[x_t > \tau]$, and $\mathrm{RS}_k = \sum_{t=T-k+1}^{T} x_t$. Labels are drawn as $y \sim \mathrm{Bernoulli}(\sigma(\eta))$. Every threshold and cutoff is a fixed constant (never a per-sequence quantile), which is what preserves monotonicity; the one quantity calibrated on the generated batch is regime L's cap $C$, the $0.95$ quantile of $E$---an outer, input-independent constant by the time labels are drawn; regime N's $\tau_N$ is a population quantile frozen as the literal constant $18.34$ before generation.

\begin{table}[t]
\centering
\caption{Exposure-process constants: each parameter is drawn independently per sequence from the stated uniform range. The latent state $z_t \in \{0,1\}$ (normal/stress) follows a Markov chain; log-exposure follows the state-dependent AR(1) of Section~\ref{sec:data}; spikes add $\exp(\mathcal{N}(\mu_u, \sigma_u^2))$ with per-step probability $p_{\text{spike}}$.}
\label{tab:dgp-exposure}
\footnotesize
\begin{tabular}{llll}
\toprule
Parameter & Range & Parameter & Range \\
\midrule
$p_{01}$ (normal$\to$stress) & $U(0.05, 0.15)$ & $\sigma_0$ (normal log-std) & $U(0.15, 0.35)$ \\
$p_{10}$ (stress$\to$normal) & $U(0.10, 0.30)$ & $\sigma_1 - \sigma_0$ (std gap) & $U(0.15, 0.40)$ \\
$\Pr[z_1 = 1]$ (initial stress) & $U(0.10, 0.30)$ & $p_{\text{spike}}$ (per-step) & $U(0.01, 0.05)$ \\
$\rho_{\mathrm{AR}}$ (AR(1) persistence) & $U(0.50, 0.90)$ & $\mu_u$ (spike log-mean) & $U(1.0, 2.0)$ \\
$\mu_0$ (normal log-mean) & $U(0.20, 0.80)$ & $\sigma_u$ (spike log-std) & $U(0.30, 0.70)$ \\
$\mu_1 - \mu_0$ (mean gap) & $U(0.80, 1.60)$ & & \\
\bottomrule
\end{tabular}
\end{table}

\begin{table}[t]
\centering
\caption{Label-regime logits $\eta$ with their exact constants (notation defined in the accompanying paragraph). Intercepts are fixed constants pre-calibrated to a class prevalence of approximately $25\%$ at $T = 24$; regime I draws a per-sequence mixture indicator $z \sim \mathrm{Bernoulli}(0.5)$. Regime P is the window-price diagnostic of E11, not a member of the fifteen-regime suite: its slope follows the standardized-signal convention anchored to regime A (the E2/E10 recalibration recipe), its window $k = 6$ matches a practice-grid window, and its Bayes ceiling is $0.747$ on the shared test splits.}
\label{tab:dgp-regimes}
\footnotesize
\begin{tabular}{ll}
\toprule
Regime & Logit $\eta$ \\
\midrule
A (recency) & $0.376\, \tilde{E} - 4.82$ \\
B (peak) & $0.108\, M - 2.65$ \\
C (mixed) & $0.266\, \tilde{E} + 0.077\, M - 4.84$ \\
D (threshold stress) & $0.117\, (M - 10.17)^+ - 1.75$ \\
E (diminishing returns) & $1.438 \log(1{+}E) + 1.621 \log(1{+}M) - 10.08$ \\
F (cumulative burden) & $0.164 \sum_t \log(1{+}x_t) - 5.20$ \\
G (persistence) & $0.484\, \mathrm{run}_{3.0} - 3.24$ \\
H (severity$\times$duration) & $0.0097\, \mathrm{ex}_{3.0} \log(1{+}\mathrm{br}_{3.0}) - 1.93$ \\
I (substitution) & $z\,(0.040\, E - 3.38) + (1{-}z)(0.077\, M - 3.38)$ \\
J (breach count) & $0.353\, \mathrm{br}_{3.0} - 3.63$ \\
K (early warning) & $0.080\, \mathrm{EWMA}_{0.9}(x_{1:T/3}) - 2.37$ \\
L (saturation) & $0.085 \min(E, C) - 3.59$, \quad $C = Q_{0.95}(E)$ \\
M (multi-scale) & $0.033\, E + 0.063\, M + 0.052\, \mathrm{RS}_4 - 3.64$ \\
N (peak prominence) & $3.0\, (M - 18.34)^+ - 2.39$ \\
O (exposure floor) & $5.66\, L - 2.08$ \\
\midrule
P (window price; diagnostic) & $0.0587\, \mathrm{RS}_6 - 2.24$ \\
\bottomrule
\end{tabular}
\end{table}

\paragraph{PGL training details.}
The penalty-trained baselines minimize the objective of Equation~\eqref{eq:pgl} with monotone pairs sampled on the fly: for every training example $x$ in a minibatch, five comparison sequences $x' = x + \delta$ are drawn, where each $\delta$ applies a non-negative increment with magnitude drawn from $\mathrm{Uniform}(0.1, 1.0)$ to a random subset of roughly $20\%$ of the time indices, so that $x \preceq x'$ by construction. We use penalty weight $\lambda_{\mathrm{mono}} = 1$ and margin $\varepsilon = 10^{-6}$; in the implementation the hinge term is squared for smoother gradients, and all other optimization settings match the unconstrained baselines. The penalty reduces but does not eliminate violations, which is precisely the contrast measured by E1.

\section{Additional results tables}
\label{app:results-tables}

Table~\ref{tab:e3-full} complements the AUC comparison of Table~\ref{tab:e3-auc} with log loss, Brier score, and expected calibration error for the same models on regimes A--E. The orderings broadly mirror AUC---best-per-regime cells differ across metrics at within-noise margins---and the constrained model's calibration is indistinguishable from the unconstrained baselines.

Table~\ref{tab:s1c} reports the ensemble-operator comparison of Section~\ref{sec:operators}: a deep ensemble of five independently trained monotone exposure RNNs against the MC-dropout variant ($K = 20$ samples), on regimes A and C with paired seeds and shared test splits. Both mechanisms are monotone per sample by construction (dropout by Corollary~\ref{cor:mcdropout}; ensembles member-wise by Theorem~\ref{thm:score}), so the table isolates uncertainty \emph{quality}: predictive metrics of the mean prediction, the Spearman correlation between per-example predictive standard deviation and absolute error, and the area under the risk--coverage curve (AURC; most-uncertain examples abstained first). Members and samples are combined identically (probability-space mean and standard deviation), so the contrast is the uncertainty mechanism alone; the last row records the training-cost trade.

\begin{table}[H]
\centering
\caption{S1c: deep ensemble ($N = 5$) versus MC dropout ($K = 20$) for the monotone exposure RNN on regimes A and C (mean $\pm$ std over five seeds; $\Delta$ is the paired per-seed difference, ensemble minus MC dropout). Both arms are monotone per sample by construction; the comparison isolates uncertainty quality and its training cost.}
\label{tab:s1c}
\scriptsize
\setlength{\tabcolsep}{1.9pt}
\begin{tabular}{lcccccc}
\toprule
 & \multicolumn{3}{c}{Regime A} & \multicolumn{3}{c}{Regime C} \\
Metric & MC drop. & Ensemble & $\Delta$ (ens $-$ MC) & MC drop. & Ensemble & $\Delta$ (ens $-$ MC) \\
\midrule
AUC & $0.8396 \pm 0.0043$ & $0.8397 \pm 0.0043$ & $0.0001 \pm 0.0002$ & $0.8534 \pm 0.0045$ & $0.8535 \pm 0.0045$ & $0.0001 \pm 0.0002$ \\
Log loss (NLL) & $0.4057 \pm 0.0058$ & $0.4055 \pm 0.0060$ & $-0.0002 \pm 0.0002$ & $0.3845 \pm 0.0042$ & $0.3842 \pm 0.0043$ & $-0.0003 \pm 0.0002$ \\
Brier & $0.1280 \pm 0.0019$ & $0.1279 \pm 0.0019$ & $-0.0000 \pm 0.0001$ & $0.1199 \pm 0.0012$ & $0.1198 \pm 0.0012$ & $-0.0000 \pm 0.0000$ \\
ECE & $0.0106 \pm 0.0011$ & $0.0115 \pm 0.0016$ & $0.0009 \pm 0.0013$ & $0.0110 \pm 0.0021$ & $0.0114 \pm 0.0018$ & $0.0004 \pm 0.0012$ \\
Uncertainty--error Spearman & $0.201 \pm 0.015$ & $0.752 \pm 0.103$ & $0.551 \pm 0.102$ & $0.210 \pm 0.015$ & $0.766 \pm 0.074$ & $0.556 \pm 0.075$ \\
AURC (lower better) & $0.2556 \pm 0.0051$ & $0.1483 \pm 0.0178$ & $-0.1073 \pm 0.0169$ & $0.2389 \pm 0.0058$ & $0.1355 \pm 0.0147$ & $-0.1034 \pm 0.0144$ \\
Training cost (s) & $191 \pm 1$ & $947 \pm 75$ & $\times 5.0$ & $267 \pm 14$ & $1321 \pm 13$ & $\times 5.0$ \\
\bottomrule
\end{tabular}

\end{table}

\begin{table}[H]
\centering
\caption{E3 log loss, Brier score, and ECE (mean $\pm$ std over five seeds) for the core model set and the calibrated-lattice comparator on regimes A--E; best per regime and metric in bold (lower is better; ties at the printed precision bolded jointly). Column abbreviations as in Table~\ref{tab:e3-auc}.}
\label{tab:e3-full}
\scriptsize
\setlength{\tabcolsep}{1.2pt}
\begin{tabular}{lccccccc}
\toprule
Regime & Mono.\ exp.\ RNN & Gen.\ mono.\ RNN & Mono.\ agg. & Lattice agg. & Unc.\ agg. & Unc.\ RNN & Unc.\ LSTM \\
\midrule
\multicolumn{8}{l}{\emph{Log loss}} \\
A & $0.4055 \pm 0.0060$ & $0.4238 \pm 0.0053$ & $0.4445 \pm 0.0067$ & $0.4283 \pm 0.0062$ & $0.4264 \pm 0.0067$ & $0.4053 \pm 0.0056$ & $\mathbf{0.4048 \pm 0.0059}$ \\
B & $0.4605 \pm 0.0066$ & $0.4901 \pm 0.0061$ & $\mathbf{0.4603 \pm 0.0068}$ & $0.4605 \pm 0.0065$ & $0.4605 \pm 0.0067$ & $0.4647 \pm 0.0064$ & $0.4625 \pm 0.0069$ \\
C & $\mathbf{0.3843 \pm 0.0043}$ & $0.3971 \pm 0.0030$ & $0.3981 \pm 0.0039$ & $0.3944 \pm 0.0038$ & $0.3942 \pm 0.0046$ & $0.3871 \pm 0.0036$ & $0.3850 \pm 0.0047$ \\
D & $0.4743 \pm 0.0080$ & $0.4968 \pm 0.0073$ & $\mathbf{0.4733 \pm 0.0083}$ & $0.4738 \pm 0.0083$ & $\mathbf{0.4733 \pm 0.0082}$ & $0.4776 \pm 0.0086$ & $0.4748 \pm 0.0079$ \\
E & $0.3647 \pm 0.0053$ & $0.3926 \pm 0.0054$ & $0.3797 \pm 0.0071$ & $0.3706 \pm 0.0060$ & $0.3733 \pm 0.0065$ & $0.3661 \pm 0.0053$ & $\mathbf{0.3643 \pm 0.0054}$ \\
\midrule
\multicolumn{8}{l}{\emph{Brier}} \\
A & $0.1279 \pm 0.0019$ & $0.1335 \pm 0.0016$ & $0.1413 \pm 0.0022$ & $0.1366 \pm 0.0021$ & $0.1358 \pm 0.0021$ & $0.1279 \pm 0.0017$ & $\mathbf{0.1278 \pm 0.0018}$ \\
B & $\mathbf{0.1463 \pm 0.0025}$ & $0.1568 \pm 0.0030$ & $\mathbf{0.1463 \pm 0.0025}$ & $\mathbf{0.1463 \pm 0.0024}$ & $0.1464 \pm 0.0024$ & $0.1477 \pm 0.0025$ & $0.1470 \pm 0.0024$ \\
C & $\mathbf{0.1199 \pm 0.0012}$ & $0.1242 \pm 0.0009$ & $0.1249 \pm 0.0010$ & $0.1236 \pm 0.0011$ & $0.1236 \pm 0.0012$ & $0.1208 \pm 0.0009$ & $0.1201 \pm 0.0012$ \\
D & $0.1509 \pm 0.0032$ & $0.1591 \pm 0.0035$ & $\mathbf{0.1507 \pm 0.0033}$ & $0.1508 \pm 0.0033$ & $\mathbf{0.1507 \pm 0.0032}$ & $0.1521 \pm 0.0034$ & $0.1511 \pm 0.0031$ \\
E & $\mathbf{0.1152 \pm 0.0018}$ & $0.1220 \pm 0.0018$ & $0.1187 \pm 0.0022$ & $0.1173 \pm 0.0020$ & $0.1179 \pm 0.0020$ & $0.1157 \pm 0.0018$ & $\mathbf{0.1152 \pm 0.0018}$ \\
\midrule
\multicolumn{8}{l}{\emph{ECE}} \\
A & $0.0113 \pm 0.0021$ & $0.0298 \pm 0.0032$ & $0.0314 \pm 0.0042$ & $0.0120 \pm 0.0018$ & $0.0139 \pm 0.0070$ & $\mathbf{0.0101 \pm 0.0030}$ & $0.0115 \pm 0.0022$ \\
B & $0.0112 \pm 0.0026$ & $0.0416 \pm 0.0058$ & $0.0114 \pm 0.0027$ & $\mathbf{0.0106 \pm 0.0033}$ & $0.0119 \pm 0.0026$ & $0.0145 \pm 0.0061$ & $0.0113 \pm 0.0017$ \\
C & $0.0117 \pm 0.0015$ & $0.0212 \pm 0.0016$ & $0.0178 \pm 0.0021$ & $0.0088 \pm 0.0021$ & $0.0109 \pm 0.0041$ & $0.0118 \pm 0.0026$ & $\mathbf{0.0084 \pm 0.0021}$ \\
D & $0.0130 \pm 0.0054$ & $0.0420 \pm 0.0012$ & $0.0110 \pm 0.0036$ & $0.0120 \pm 0.0041$ & $\mathbf{0.0100 \pm 0.0036}$ & $0.0155 \pm 0.0060$ & $0.0104 \pm 0.0027$ \\
E & $0.0123 \pm 0.0025$ & $0.0386 \pm 0.0031$ & $0.0319 \pm 0.0018$ & $0.0105 \pm 0.0022$ & $0.0188 \pm 0.0077$ & $0.0099 \pm 0.0039$ & $\mathbf{0.0084 \pm 0.0034}$ \\
\bottomrule
\end{tabular}

\end{table}

\begin{table}[H]
\centering
\caption{Pooled monotonicity-violation totals for the unconstrained baselines quoted in the E5 and E9 prose: per capacity tier and pooled over the E9 sweep (both regimes, five seeds), and pooled over the E5 missingness grid (regime C). The static unconstrained aggregate has no capacity knob, so its identical per-tier counts restate one fitted model evaluated once per tier; it is not part of the E5 model set. Every monotone-by-construction model records zero violations in every corresponding cell (Table~\ref{tab:e9-capacity}; Section~\ref{sec:results}).}
\label{tab:baseline-violations}
\small
\begin{tabular}{lrrrrrr}
\toprule
 & \multicolumn{5}{c}{E9 capacity sweep (per tier and pooled)} & E5 grid \\
\cmidrule(lr){2-6}
Model & tiny & small & medium & large & total & total \\
\midrule
Unconstrained RNN & 324 & 926 & 1{,}713 & 4{,}172 & 7{,}135 & 2{,}373 \\
Unconstrained LSTM & 1{,}352 & 1{,}168 & 1{,}188 & 794 & 4{,}502 & 2{,}425 \\
Unconstrained aggregate & 1{,}958 & 1{,}958 & 1{,}958 & 1{,}958 & 7{,}832 & -- \\
\bottomrule
\end{tabular}

\end{table}

\section{Reproducibility details}
\label{app:repro}

\paragraph{Data and seeds.}
Every experiment uses $50{,}000$ training, $10{,}000$ validation, and $10{,}000$ test sequences per (regime, length) dataset, and five seeds ($42$--$46$) that govern both data generation and training; compared models always share identical splits within a (regime, seed) cell.

\paragraph{Optimization.}
All neural models train with Adam at learning rate $10^{-3}$ and batch size $256$ under a $100$-epoch ceiling, with early stopping on validation loss (patience $20$, minimum improvement $10^{-4}$; the small non-zero threshold lets slowly converging constrained models stop once improvement falls below noise level rather than running to the ceiling). Classification experiments use binary cross-entropy. The E10 sample-size sweep raises the epoch ceiling so that early stopping, not the budget, decides termination (Section~\ref{sec:experiments}); the E9 capacity study reports a training-adequacy audit and a cap-raising diagnostic for the one flagged cell. No per-model learning-rate or width search is performed: constrained and unconstrained models share the optimizer settings above, so every comparison runs at a matched training protocol rather than per-model tuning, and E9 sweeps the one remaining axis---capacity---to confirm that the parity verdicts are not an artifact of the shared defaults.

\paragraph{Gradient-boosting and lattice comparators.}
The constrained LightGBM comparators (E10) train with the binary objective, at most $500$ boosting rounds with early stopping on validation AUC (patience $50$ rounds), $31$ leaves per tree, learning rate $0.05$, an increasing monotone constraint on every feature with the \texttt{advanced} (penalization-free) constraint method, and deterministic single-threaded tree construction. The calibrated monotone lattices use one monotone calibrator per input feature---a convex combination of $8$ shifted sigmoids in log space, $c_i(x) = \sum_k w_{ik}\, \sigma(\beta_{ik}(\log(1+x_i) - t_{ik}))$ with softmax weights $w_{ik}$, non-negative slopes $\beta_{ik}$, and learnable knots $t_{ik}$ initialized across the input range, the smooth counterpart of a piecewise-linear calibrator---feeding a multilinear lattice on the unit hypercube whose M\"obius coefficients are constrained non-negative; lattice dimension equals the feature count of the grid it consumes. As with the neural models, these settings are shared library defaults across all regimes and both grids---no per-model or per-regime tuning is performed---so the sub-$0.01$ deltas of E10 are read against a fixed, disclosed configuration rather than a tuned incumbent; a per-seed validation-selected regularization grid over both grid GBMs (Section~\ref{sec:e10-results}) confirms the parity-ladder verdicts are insensitive to this choice. One stopping-criterion asymmetry is disclosed: the boosted comparators early-stop on validation AUC (the metric configured above) while the neural models early-stop on validation loss; the asymmetry mildly favors the GBMs on the headline metric, i.e., it is conservative for the learned channel's reported wins and anti-conservative only for the D/E tie verdicts.

\paragraph{Model tiers.}
The four capacity tiers of E9 are configuration presets: the tiny tier fixes identity transforms with a monotone linear head (four parameters); the small, medium, and large tiers use monotone MLPs for $g_1$, $g_2$, and $F$ with hidden layers $(8)$, $(16, 8)$, and $(64, 32)$ respectively (tanh activations, non-negative weights via softplus parameterization). The unconstrained RNN and LSTM baselines use hidden sizes $8$, $16$, $32$, and $128$ across the tiers ($1$ recurrent layer for tiny/small, $2$ for medium/large), and the generic monotone RNN the same hidden sizes; the medium tier ($2$-layer, hidden $32$ for the baselines; $(16, 8)$ monotone MLPs for the proposed model) is the default for E1--E8. One baseline asymmetry is disclosed: the static monotone aggregate's head and the generic monotone RNN use ReLU activations (the library defaults), which under non-negative weights are convex-only, whereas the proposed model's monotone components use tanh. The head-class share of the affected gaps is measured directly and carried in-line with the E2 and E3 results (at most $+0.0096$ on the E2 sweep; immaterial on regime C, the regime of E4 and E5, so the robustness verdicts are unaffected). The generic monotone RNN's sensitivities are likewise measured where they are load-bearing: the tanh-component arm and the weight-constrained Elman arm (hidden-to-hidden weights softplus-parameterized, tanh step activation, same cells and seeds) are reported per boundary regime in Table~\ref{tab:e10-boundary}; neither changes the boundary verdict. Per-tier parameter counts appear in Table~\ref{tab:e9-capacity}. The MC-dropout variant of E1 and the ensemble study S1c place dropout with rate $0.1$ inside the monotone MLP components. E1 verifies each of $K = 20$ pinned masks separately; S1c averages $K = 20$ stochastic passes.

\paragraph{Monotonicity verification.}
The random, adversarial, and coordinatewise checks run in float64 with a relative tolerance $\varepsilon = 10^{-6}$: a pair counts as a violation when $f(x') < f(x) - \varepsilon \max(|f(x)|, 1)$, so reported violations are genuine order inversions rather than floating-point noise. The structured probes P1--P5 run in the models' float32 inference precision with an absolute tolerance of $10^{-6}$, and the monotone models' zero-violation results are reported at those tolerances. The random operator draws perturbations $\delta \sim \mathrm{U}(0, 1)^T$ elementwise (the same scale as the adversarial initialization) and adds them to held-out base sequences under the per-experiment budgets. Tree-ensemble comparators are verified by the random, coordinatewise, and structured-probe operators; the gradient-based adversarial operator does not apply to them and is reported as not applicable.

\paragraph{Bootstrap and hardware.}
All two-level paired bootstrap intervals use $1{,}000$ resamples (seeds resampled at the outer level, test examples on shared indices at the inner level). All experiments run on CPU (a consumer Apple-silicon workstation; no GPU is used anywhere in the suite), one process per seed; the training-cost figures of Table~\ref{tab:s1c} are wall-clock seconds on this hardware.

\end{document}